\documentclass[11pt]{article}

\usepackage[dvips,letterpaper,margin=1.0in]{geometry}
\usepackage{sty}

\allowdisplaybreaks[4]

\title{Mild Over-Parameterization Benefits Asymmetric Tensor PCA}

\author{Shihong Ding
	\quad
	Weicheng Lin
	\quad Cong Fang$^{\textsuperscript{†}}$\\
	\\
	\small Peking University
	\\\\
}
\date{}

\begin{document}

\maketitle

\renewcommand{\thefootnote}{\dag}
\footnotetext{Corresponding author.}

\begin{abstract}
Asymmetric Tensor PCA (ATPCA) is a prototypical model for studying the trade-offs between sample complexity, computation, and memory. Existing algorithms for this problem typically require at least $d^{\left\lceil\overline{k}/2\right\rceil}$ state memory cost to recover the signal, where $d$ is the vector dimension and $\overline{k}$ is the tensor order. We focus on the setting where $\overline{k} \geq 4$ is even and consider (stochastic) gradient descent-based algorithms under a limited memory budget, which permits only mild over-parameterization of the model. We propose a matrix-parameterized method (in $d^{2}$ state memory cost) using a novel three-phase alternating-update algorithm to address the problem and demonstrate how mild over-parameterization facilitates learning in two key aspects: (i) it improves sample efficiency, allowing our method to achieve \emph{near-optimal} $d^{\overline{k}-2}$ sample complexity in our limited memory setting; and (ii) it enhances adaptivity to problem structure, a previously unrecognized phenomenon, where the required sample size naturally decreases as consecutive vectors become more aligned, and in the symmetric limit attains $d^{\overline{k}/2}$, matching the \emph{best} known polynomial-time complexity. To our knowledge, this is the \emph{first} tractable algorithm for ATPCA with $d^{\overline{k}}$-independent memory costs.
\iffalse
\fi
\end{abstract}
\section{Introduction}\label{intro}
Tensor PCA (TPCA) is a stylized testbed for studying the computation and statistical trade-off in estimating unknown parameters from higher-order moment tensors \citep{montanari2014statistical, perry2016statistical, dudeja2021statistical, dudeja2024statistical, kunisky2024tensor}. In this work, we consider the Signal Recovery Problem for Asymmetric TPCA (ATPCA): given $N$ independent and identically distributed (i.i.d.) tensor observations $\mathbf{T}^{(t)} = \lambda\bigotimes_{i=1}^{\overline{k}} v_{i}^{*} + \mathbf{E}^{(t)}$ for $t \in [N]$, where $v_{i}^{*}\in\mathbb{S}^{d-1}$ are unknown unit vectors, $\lambda\gtrsim 1$ is the signal-to-noise ratio (SNR), and $\mathbf{E}^{(t)}$ are i.i.d. zero-mean noise tensors, the goal is to recover $\{\widehat{v}_{i}\}_{i=1}^{\overline{k}}$ such that $\left\|\mathrm{Vec}\left(\bigotimes_{i=1}^{\overline{k}} \widehat{v}_{i} -\bigotimes_{i=1}^{\overline{k}} v_{i}^{*}\right)\right\|_{2} \leq o(1)$ with constant probability.

ATPCA is a special case of tensor CP decomposition for which no computationally tractable solution currently exists. For $\overline{k} \geq 3$, the model exhibits a pronounced statistical-to-computational gap \citep{montanari2014statistical}: the best polynomial-time algorithms require $\lambda^{2} N \gtrsim d^{\overline{k}/2}$ samples to recover the signal \citep{hopkins2015tensor, hopkins2016fast}, whereas $\Theta(d)$ samples suffice for maximum likelihood estimation (MLE) under exact vectorization parameterization \citep{montanari2014statistical}---which is, however, computationally intractable.

While the above results operate without memory constraints, practical applications often impose strict limitations, motivating the study of TPCA in this regime. The recent interesting work \citep{dudeja2024statistical} has shown that memory constraints fundamentally impact the sample complexity for TPCA, particularly in the asymmetric setting. In this regime, \cite[Theorem~2]{dudeja2024statistical} establishes a lower bound
\begin{equation*}
    S \cdot K \cdot \lambda^{2} N \gtrsim d^{\overline{k}},
    % \log_{d}(S) + \log_{d}(K) + \log_{d}(\lambda^{2} N) \geq \overline{k}+\mathcal{O}\left(\frac{1}{\log d}\right),
\end{equation*}
where $S$ is the state memory cost of the algorithm, $K$ is the total number of full data passes, and $N$ is number of samples (see formal definition in Section~\ref{sec-2.3}). Consequently, any algorithm with low memory cost must compensate with either more passes or more samples. However, existing algorithms---such as tensor unfolding \citep{montanari2014statistical}, spectral methods \citep{montanari2014statistical, zheng2015interpolating, hopkins2016fast, biroli2020iron} and the alternating least squares (ALS) algorithm with spectral initialization \citep{tang2025revisit}---incur memory costs of at least $S\gtrsim d^{\left\lceil\overline{k}/2\right\rceil}$, leaving low-memory algorithm design an \emph{open} problem.

In this paper, we investigate (stochastic) gradient-based methods applied to MLE under memory constraints, where the memory exponent in $d$ is independent of $\overline{k}$. These constraints naturally limit the degree of over-parameterization---a common strategy for addressing non-convexity---that can be employed. Heavy over-parameterization, where the parameter dimension scales as
$d^{\Omega\left(\overline{k}\right)}$, while simplifying optimization, significantly increases per-step computation and memory overhead. In contrast, very mild over-parameterization---as seen, for example, in single-neuron learning \citep{yehudai2020learning, damian2023smoothing, lee2024neural} and matrix sensing \citep{tu2016low, li2018algorithmic, xiong2023overparameterizationslowsgradientdescent, soltanolkotabi2025implicit}---can succeed in solving the target problem. We focus on the even-order case ($\overline{k}\geq 4$) and propose a matrix-parameterized streaming algorithm, demonstrating how mild over-parameterization facilitates learning in two key perspectives: (i) \textbf{improving sample efficiency}---our algorithm recovers the signal with $d^{\overline{k}-2}$ samples in a single pass in the worst case, matching the memory–runtime–sample complexity lower bound for ATPCA; (ii) \textbf{enhancing adaptivity to problem structure}---our sample complexity bound reveals previously unobserved algorithmic adaptivity, showing that required samples decrease as the problem simplifies. Specifically, for the symmetric case where $v_{1}^{*} = \cdots = v_{\overline{k}}^{*}$, our algorithm achieves the best-known sample complexity $d^{\overline{k}/2}$ for polynomial-time algorithms.

The main novelty of our algorithmic design is threefold: (i) \textbf{streaming learning with three-phase training}, where fresh data keeps the dynamics close to their population-level counterparts, and the three-phase schedule achieves fast learning rates while controlling noise; (ii) \textbf{matrix parameterization with tailored random initialization}, where we adopt matrix parameterization following \cite{tu2016low, li2018algorithmic, xiong2023overparameterizationslowsgradientdescent, soltanolkotabi2025implicit} in matrix sensing and Symmetric TPCA (STPCA), but now pair consecutive vectors into matrices---the initialization combines a random vector with an identity matrix, with the identity aiding escape from stationary points and the random component extracting signal structure, thereby enhancing adaptivity; (iii) \textbf{alternating updates with proper normalization}, where updates to each matrix parameter block are normalized and performed sequentially, balancing learning rates across components and over time, which stabilizes training and simplifies analysis. These proposed techniques offer new insights for algorithmic design.

The main contributions of this paper are summarized as follows:
\begin{itemize}
    \item \textbf{Algorithm.} We propose the \emph{first} tractable (and near-optimal) algorithm for ATPCA that achieves $d^{\overline{k}}$-independent memory costs. The key techniques---including matrix parameterization, tailored random initialization, sequential updates with normalization, and three-phase streaming learning---offer fresh perspectives to algorithmic design.
    \item \textbf{Theory.} Using ATPCA as a theoretical testbed, we demonstrate how mild over-parameterization enhances statistical efficiency---not only through improved sample complexity, but also via a novel form of model adaptivity that has \emph{rarely} been explored in prior work.
\end{itemize}

\paragraph{Notations.}
Vectors, matrices, and tensors (of order $k \geq 3$) are denoted by lowercase letters, uppercase letters, boldface uppercase letters, respectively. We use $\langle \cdot,\cdot\rangle$ and $\|\cdot\|_{2}$ for the standard inner product and Euclidean norm of vectors. For matrices, $\langle Q,W\rangle \coloneqq \mathrm{tr}(Q^\top W)$ and $\|\cdot\|_{\mathrm{F}}$ denotes the Frobenius norm. For tensors, inner product and Frobenius norm are defined via the standard vectorization $\mathrm{Vec}(\cdot)$, and we use $\bullet$ to denote tensor contraction over the data modes. We use standard asymptotic notation $\mathcal{O}(\cdot),\Omega(\cdot), \allowbreak \Theta(\cdot),\lesssim,\gtrsim,\asymp,\ll,\gg$, and those with tildes like $\widetilde{\mathcal{O}}(\cdot)$ to hide polylogarithmic factors. In addition, for vectors $u'\in\mathbb{R}^{d_{1}}$ and $v'\in\mathbb{R}^{d_{2}}$ with $d_{1}\leq d_{2}$, we define the correlation between $u'$ and $v'$ as $\mathrm{Cor}(u',v') = \langle u',\begin{pmatrix} I_{d_{1}} & 0_{d_{1}\times(d_{2}-d_{1})}\end{pmatrix}v'\rangle$.

\section{Problem Setup}\label{setup}
\subsection{Problem Formulation}
We focus on a rank-1 asymmetric tensor model. Let $\left\{\mathbf{T}^{(t)}\right\}_{t\in[N]}$ denote a collection of $N$ i.i.d. observations drawn from the distribution of $\mathbf{T}$, generated according to the following model:
\begin{align}\label{def-model}
	\mathbf{T} = \lambda \cdot \bigotimes_{n=1}^{\overline{k}} v_{n}^{*} + \mathbf{E},
\end{align}
where $\lambda \in \mathbb{R}^{+}$ represents the known SNR and $\mathbf{E}$ represents a zero-mean random noise tensor. Here we allow the unit vectors to have different dimensions: $v_{i}^{*} \in \mathbb{S}^{d_{i}-1}$ for $i \in \left[\overline{k}\right]$, with dimensions ordered as $d_{1} \leq d_{2} \leq \cdots \leq d_{\overline{k}}$. And we adopt a slightly stricter recovery criterion based on the unit vectors themselves rather than the tensor. Without loss of generality, we allow each $v_{n}^{*}$ to be recovered only up to sign, and define the component-wise estimation error accordingly as: $\mathcal{L}(v_{n}, v_{n}^{*}) \coloneqq \min\left\{\left\|v_{n}-v_{n}^{*}\right\|_{2}^{2}, \left\|v_{n}+v_{n}^{*}\right\|_{2}^{2}\right\}$. Then we formalize the Signal Recovery Problem.
\begin{problem}[Signal Recovery for ATPCA with Constant Probability]\label{PCA-def}
	Given $N$ i.i.d. observations $\Big\{\mathbf{T}^{(t)}=\lambda\cdot \bigotimes_{n=1}^{\overline{k}} v_{n}^{*}+\mathbf{E}^{(t)}\Big\}_{t\in[N]}$ and constant-level failure probability $\delta\in(0,1)$. The goal is to find unit vectors $\left\{\widehat{v}_{n}\right\}_{n=1}^{\overline{k}}$ such that $\max_{n\in\left[\overline{k}\right]}\mathcal{L}(\widehat{v}_{n}, v_{n}^{*})\leq o(1)$, with probability at least $1-\delta$.
\end{problem}
Note that if we solve Problem \ref{PCA-def}, then by selecting $\widehat{\mathbf{T}}$ as either $\bigotimes_{n=1}^{\overline{k}} \widehat{v}_{n}$ or $-\bigotimes_{n=1}^{\overline{k}} \widehat{v}_{n}$ (via a random coin flip), then with probability at least $1-2\delta$, we have $\left\|\widehat{\mathbf{T}} -  \bigotimes_{n=1}^{\overline{k}} v_{n}^{*} \right\|_{\mathrm{F}}\leq o(1)$.

\subsection{Model Assumptions}
We then make assumptions on the data generating process, which are either standard or weaker than those commonly used in ATPCA studies. We begin by imposing a condition on the SNR.
\begin{assumption}[SNR Magnitude]\label{ass-pro}
    The SNR $\lambda$ satisfies $\Omega(1) \leq \lambda \leq \mathcal{O}\left(\prod_{n=1}^{\overline{k}}d_{n}^{1/4}\right)$.
\end{assumption}
Assumption~\ref{ass-pro} focuses on a non-trivial intermediate SNR regime, previously considered in \citep{ding2025near}. The lower bound $\Omega(1)$ excludes regimes that are too noisy for existing algorithms to succeed under their established complexities and can be reduced to this regime via sample averaging. The upper bound excludes regimes where the noise no longer dominates the sample complexity, and other simpler efficient algorithms can achieve recovery with lower complexity \citep{montanari2014statistical, tang2025revisit}.
% Assumption~\ref{ass-pro} targets the non-trivial intermediate SNR regime. The lower bound $\Omega(1)$ is standard \citep{ding2025near}, below which the noise is too large for existing algorithms to succeed under their established complexities. The upper bound is also standard \citep{ding2025near}, beyond which the noise no longer dominates the sample complexity, and other simple efficient algorithms can achieve recovery with lower complexity \citep{montanari2014statistical, tang2025revisit}.

\iffalse
    \paragraph{Remark on SNR Regimes.} Assumption~\ref{ass-pro} targets the non-trivial intermediate SNR regime. 
    \begin{itemize}
    	\item The lower bound $\Omega(1)$ represents the constant SNR regime typically found in single-observation settings ($\lambda \asymp 1$). 
    	\item Significantly exceeding this level (i.e., $\lambda \gg \Omega(1)$) often involves averaging multiple independent observation tensors, which effectively amplifies the SNR to $\sqrt{N}\lambda$ \shihong{[Citation]}.
    	\item The upper bound is chosen because for SNR regimes exceeding $\mathcal{O}\left(\prod_{n=1}^{\overline{k}}d_{n}^{1/4}\right)$, efficient offline algorithms have already been established for recovering each signal vector $v_{n}^{*}$ \shihong{[Citation]}.
    \end{itemize}
    Consequently, our analysis focuses on the gap between these established bounds.
\fi

We then characterize the noise by imposing specific regularity conditions. Specifically, we adopt the general sub-Gaussian Assumption~that is weaker than those widely used in the literature.
\begin{assumption}[Noise Concentration, \cite{wainwright2019high}]\label{ass-base}
	The zero-mean random noise $\mathbf{E}$ satisfies the following sub-Gaussian tail bound:
	\begin{align}
		\mathbb{E}\left[\exp\left\{\left\langle u, \mathrm{Vec}(\mathbf{E})\right\rangle\right\}\right] \leq \exp\left\{ \frac{1}{2}\|u\|_{2}^{2} \right\}, \quad \forall u\in\mathbb{R}^{\prod_{m=1}^{\overline{k}}d_{m}}.\notag
	\end{align}
\end{assumption}
Assumption~\ref{ass-base} is a standard sub-Gaussian Assumption~applied to the vectorized noise $\mathrm{Vec}(\mathbf{E})$, ensuring that any projection of the random vector satisfies a one-dimensional sub-Gaussian property. While most TPCA literature typically restricts analysis to noise with i.i.d. standard Gaussian entries---a special case that satisfies Assumption~\ref{ass-base}---our framework naturally accommodates a broader class of noise models. For example, Assumption~\ref{ass-base} is satisfied if $\mathrm{Vec}(\mathbf{E})$ meets either of the following sufficient conditions:
\begin{itemize}
	\item The coordinates are mutually independent sub-Gaussian random variables with parameter $1$;
	\item The uniform bound $\|\mathrm{Vec}(\mathbf{E})\|_{2} \leq 1$ holds.
\end{itemize}

\subsection{Algorithm Constraints}\label{sec-2.3}
We focus on memory-bounded algorithms, which have been used to explain statistical–computational gap \citep{steinhardt2016memory}, and follow \cite{dudeja2024statistical} in using the template given in Algorithm~\ref{alg:mem-con-template}:

\begin{algorithm}[!t] \caption{Memory-Bounded Estimation Protocol with Parameters $(N,K,S)$} \label{alg:mem-con-template}
    \footnotesize
    \textbf{Input:} Dataset {\scriptsize $\left\{\mathbf{T}^{(i)} : i \in [N]\right\}$} of samples in $\mathbb{T}$.
    
    \textbf{Output:} Estimator $\widehat{\mathbf{V}} \in \widehat{\mathbb{V}}$.
    
    \textbf{Variables:} Memory state $M \in \{0,1\}^{S}$.
    
    \begin{algorithmic}[1]
        \STATE Initialize memory state $M \in \{0,1\}^{S}$.
        \FOR{$k = 1,2,\ldots,K$}
            \FOR{$i = 1,2,\ldots,N$}
                \STATE $M \leftarrow f_{k,i}(M, \mathbf{T}^{(i)})$.
            \ENDFOR
        \ENDFOR
        \STATE \textbf{return} $\widehat{\mathbf{V}} = g(M)$.
    \end{algorithmic}
\end{algorithm}

\begin{definition}[Memory-Bounded Algorithm with Resource Profile $(N, K, S)$, \cite{dudeja2024statistical}]\label{def-memory}
    Let $\mathbb{T}$ denote the space of input samples and  $\widehat{\mathbb{V}}$ denote the estimator space for the recovered signal. A memory-bounded estimation algorithm with $(N, K, S)$ computes an estimator by making $K$ passes through a dataset of $N$ samples using a memory state of $S$ bits. Such an algorithm is specified by the update functions $f_{k,i}:\{0,1\}^{S} \times \mathbb{T} \to \{0,1\}^{S}$ and an estimator function $g: \{0,1\}^{S} \to \mathbb{\widehat{V}}$, which are used as follows. In the $k$-th pass through the dataset, the algorithm considers each sample $\mathbf{T}^{(i)}$ for $i \in [N]$ in sequence, and it updates the memory state by applying the update function $f_{k,i}$ to the current memory state and the sample $\mathbf{T}^{(i)}$ under consideration. After all $K$ passes are complete, the estimator is computed by applying the estimator function $g$ to the final memory state.
\end{definition}

\paragraph{Memory Constraints.}
Definition~\ref{def-memory} provides a standard model for memory-bounded algorithms, capturing the information accessible to the algorithm through the memory state $M$ and the number of passes $K$. Under this model, we consider the regime where the tensor order is much smaller than the dimensions, i.e., $\overline{k} \ll d_{i}$, and focus on complexity exponents in $d$, treating $\mathrm{poly}\left(\overline{k}\right)$ and $4^{\overline{k}}$ as constants. Current algorithms typically require $S \gtrsim (\max_{n} {d_{n}})^{\overline{k}/2}$ to solve the problem. In contrast, we restrict the memory cost such that $S \lesssim (\max_{n} {d_{n}})^{c}$, where $c$ is independent of $\overline{k}$.

\paragraph{Gradient-based Algorithms.}
We focus on widely-used MLE-based methods. Let $\mathbf{T}(w): \mathbb{R}^{p} \to \mathbb{R}^{d_{1} \times \cdots \times d_{\overline{k}}}$ be the parameterized tensor, and consider the MLE objective $\min_{w} \frac{1}{T} \sum_{t=1}^{T} \left\|\mathbf{T}(w)-\mathbf{T}^{(t)}\right\|_{\mathrm{F}}^{2}$. Its single-sample gradient is $\nabla_{w} \left\|\mathbf{T}(w)-\mathbf{T}^{(t)}\right\| _{\mathrm{F}}^{2} = 2\left(\nabla_{w} \mathbf{T}(w)\right) \bullet \left(\mathbf{T}(w)-\mathbf{T}^{(t)}\right) \in \mathbb{R}^{p}$. We emphasize that this gradient depends on the sample $\mathbf{T}^{(t)}$ only through quantities of the form $\nabla_w \mathbf{T}(w) \bullet \mathbf{T}^{(t)}$. Accordingly, we consider algorithms that access the data only via such gradient information.

\section{Algorithm}\label{alg}
We design a multi-phase algorithm (see Algorithm~\ref{MSANSGA}) based on Sequential Normalized Stochastic Gradient Ascent (SGA) (see Algorithms~\ref{SGA-con-ss} and \ref{SGA-decay-ss}). We focus on the even case ($\overline{k}=2k$) in this section; extensions to the odd case ($\overline{k}=2k+1$) are provided in Section~\ref{alg-odd-case}.

\begin{algorithm}[!t] \caption{Sequential Normalized SGA with Constant Step-Size (SNSGA1)} \label{SGA-con-ss}
    \footnotesize
    \textbf{Input:} Initial weights {\scriptsize $\left\{W_{m}^{(0)}\right\}$}, step sizes $\left\{\eta_{m}\right\}$, iteration budgets $\left\{T_{m}\right\}$, reward oracle $\widehat{\mathcal{R}}$ with query access $\widehat{\mathcal{R}}_{m}^{(t)}(\cdot)$.
    \textbf{Output:} Normalized final weights {\scriptsize $\left\{\overline W_{m}^{(T_{m})}\right\}$}. 
    
    \textbf{Variables:} Memory states $M = \left\{M_{m}: m \in [k]\right\}$, where $M_{m}$ stores the iterates {\scriptsize $\left\{W_{m}^{(t)} : t \in \{0\} \cup [T_{m}]\right\}$} for block $m$.
    
    \begin{algorithmic}[1]
        \FOR{$m = k, k-1, \dots, 1$}
            \FOR{$t = 0, \dots, T_{m}-1$}
                \STATE Sample fresh data $\mathbf T^{(t+1)}$.
                \STATE Update iterate in memory $M_{m}$ via normalized stochastic gradient ascent:
                {\scriptsize
        		\begin{equation}\label{update-SGA}
        			W_{m}^{(t+1)}\leftarrow \left(1+\frac{\eta_{m} \widehat{\mathcal{R}}_{m}^{(t+1)}\left(W_{m}^{(t)}\right)}{\left\|W_{m}^{(t)}\right\|_{\mathrm{F}}}\right)W_{m}^{(t)}+\eta_{m}\left\|W_{m}^{(t)}\right\|_{\mathrm{F}}\nabla_{W_{m}}\widehat{\mathcal{R}}_{m}^{(t+1)}\left(W_{m}^{(t)}\right).
        		\end{equation}}%
            \ENDFOR
            \STATE Set final weight: $\overline W_{m}^{(T_{m})} \gets W_{m}^{(T_{m})}/\left\|W_{m}^{(T_{m})}\right\|_{\mathrm{F}}$.
        \ENDFOR
        \RETURN {\scriptsize $\left\{\overline{W}_{m}^{(T_{m})}\right\}$}.
    \end{algorithmic}
\end{algorithm}

\subsection{Sequential Normalized SGA}\label{sec-3.1}
\paragraph{Matrix Parameterization and Sequential Update.}
Following \cite{tu2016low, li2018algorithmic, xiong2023overparameterizationslowsgradientdescent, soltanolkotabi2025implicit}, we adopt matrix parameterization but pairing consecutive vectors into matrices. Specifically, we use $k$ matrix parameters $W_{m}$ to represent the rank-1 components $v_{2m-1}^{*} v_{2m}^{*\top}$ for $m \in [k]$. We introduce the inner loop update in Algorithm~\ref{SGA-con-ss} and its variant Algorithm~\ref{SGA-decay-ss} (see Section~\ref{sup-alg}), which sequentially update each matrix. This design balances the learning rates across components and simplifies theoretical analysis. More concretely, given initialization $\left\{W_{m}^{(0)}\right\}_{m\in[k]}$, each $W_{m}$ is updated using a stochastic gradient computed from a fresh sample, evaluated at $\left(\overline{W}_{1}, \dots, \overline{W}_{m-1}, W_{m}, \right.\allowbreak\left. \overline{W}_{m+1}, \dots, \overline{W}_{k}\right)$. In the $l$-th phase, for each $m$ and sample $\mathbf{T}^{(t)}$, we define the objective function
{\small
\begin{equation}\label{def-hatR}
    \widehat{\mathcal{R}}^{(t)}_{m}(W_{m})=\left\langle
    \bigotimes_{i=1}^{m-1}
        \overline{W}_{i}^{\left(\sum_{\tau=1}^{l-1}T_{i}^{(\tau)}\right)}
        \otimes W_{m}
    \otimes \bigotimes_{j=m+1}^{k}\overline{W}_{j}^{\left(\sum_{\tau=1}^{l}T_{j}^{(\tau)}\right)},
    \mathbf{T}^{(t)}
    \right\rangle,
\end{equation}}%
and updates $W_{m}$ via $\nabla_{W_{m}}\left(\|W_{m}\|_{\mathrm{F}}\widehat{\mathcal{R}}^{(t)}_{m}(W_{m})\right)$ by Eq.~\eqref{update-SGA}. It is worth noting that $\widehat{\mathcal{R}}^{(t)}_{m}(W_{m})$ is 1-homogeneous for each $W_{m}$, which implies {\small $\widehat{\mathcal{R}}^{(t)}_{m}(W_{m})=\left\langle \nabla_{W_{m}}\widehat{\mathcal{R}}_{m}^{(t+1)}\left(W_{m}\right), W_{m} \right\rangle$ so that the function value $\widehat{\mathcal{R}}^{(t)}_{m}(W_{m})$} can be computed via MLE gradient oracle. This resembles block coordinate methods such as alternating least squares (ALS). However, instead of solving each block subproblem to optimality, we perform only a fixed number of stochastic gradient iterations per block.

\paragraph{Normalization.}
To eliminate scale drift and stabilize optimization, each update is subject to a normalization mechanism. Specifically, the objective is scaled by $\|W_{m}\|_{\mathrm{F}}$, leading to the form $\|W_{m}\|_{\mathrm{F}} \left\langle \mathbf{T}(W_{m}), \mathbf{T}^{(t)} \right\rangle$, which is $2$-homogeneous in $W_{m}$ and implies that its gradient scales linearly with $\|W_{m}\|_{\mathrm{F}}$. Consequently, the evolution of the normalized iterate $W_{m} / \|W_{m}\|_{\mathrm{F}}$ is invariant to scale and depends only on the directional information of the stochastic gradients. Moreover, at the end of each update cycle for block $m$, the iterate is normalized as $\overline W_{m}^{(T_{m})} = W_{m}^{(T_{m})} / \left\|W_{m}^{(T_{m})}\right\|_{\mathrm{F}}$, which prevents scale drift and stabilizes the updates across blocks.

\begin{algorithm}[!t] \caption{Multi-Phase Sequential Normalized SGA (MPSNSGA)} \label{MSANSGA}
    \footnotesize
    \textbf{Input:} cycles $h^{*}$, initializations {\scriptsize $\left\{W_{m}^{(0,\tau)}\right\}_{\tau\in\left[4^{k}\right]}$}, phase schedules {\scriptsize $\left\{\left(\eta_{m,h}^{(1)},T_{m,h}^{(1)}\right)\right\}_{h\in[h^{*}]}$}
    and {\scriptsize $\left\{\left(\eta_{m}^{(i)},T_{m}^{(i)}\right)\right\}_{i\in\{2,3\}}$}, step size decay lengths {\scriptsize $\left\{\widehat T_{m}^{(3)}\right\}$}.
    
    \textbf{Output:} Estimator $V = \{v_n\}_{n=1}^{\overline{k}}$.
    
    \textbf{Variables:} Memory states $M = \{ M_{m,\tau}, \widehat{M}_{m,p} : m \in [k], \tau \in \left[4^{k}\right], p \in \{0,1\}\}$, where each $M_{m,\tau}$ stores iterates {\small $\left\{W_{m}^{(\tau,t)}: t \in \{0\} \cup \left[T_{m}^{(1)}+T_{m}^{(2)}\right]\right\}$}, $\widehat{M}_{m,p}$ stores iterates {\small $\left\{\widehat W_{m,p}^{(t)}: t \in \{0\} \cup \left[T_{m}^{(3)}\right]\right\}$}, and $T_{m}^{(1)}\coloneqq \sum_{h\in[h^{*}]}T_{m,h}^{(1)}$.
    
    \begin{algorithmic}[1]
        \STATE \textbf{Phase~I.}
        For each $\tau\in\left[4^{k}\right]$, run Algorithm~\ref{SGA-con-ss} for $h^{*}$ cycles with schedules $\left\{\left(\eta_{m,h}^{(1)}, T_{m,h}^{(1)}\right)\right\}_{h\in[h^{*}]}$ and reward $\widehat{\mathcal{R}}$, updating iterate $W_{m}^{(\tau,t)}$ in memory $M_{m,\tau}$ for each block $m\in[k]$.
        
        \STATE \textbf{Phase~II.}
        Continue each run with schedule $\left(\eta_{m}^{(2)},T_{m}^{(2)}\right)$ and reward $\widehat{\mathcal{R}}$ to update iterate $W_{m}^{(\tau,t)}$ in memory $M_{m,\tau}$ for each block $m\in[k]$, producing $\left\{\breve W_{m}^{(\tau)}\right\}$ at the final step.
        
        \STATE \textbf{Select Initialization.}
		For each $\tau\in\left[4^{k}\right]$, draw $N_{0}$ fresh samples $\mathbf{T}^{(\tau,i)}$ and compute the scores $S_{\tau}=\frac{1}{N_{0}}\sum_{i=1}^{N_{0}}\left\langle \bigotimes_{m=1}^{k} \breve W_{m}^{(\tau)},\mathbf{T}^{(\tau,i)}\right\rangle$, where the inner product can be evaluated using the MLE gradient oracle by treating $\breve{W}_{k}^{(\tau)}$ as the parameter (see Section~\ref{sec-3.1}). Let $\tau^{*}=\mathop{\arg\max}_{\tau} S_{\tau}$, and set $\widehat W_{m}^{(0)}=\breve W_{m}^{(\tau^{*})}$.
        
        \STATE \textbf{Phase~III.}
		For each block $m\in[k]$, run Algorithm~\ref{SGA-decay-ss} twice from {\scriptsize $\widehat W_{m}^{(0)}$} with schedule {\scriptsize $\left(\eta_{m}^{(3)},T_{m}^{(3)}\right)$} and step size decay lengths {\scriptsize $\widehat T_{m}^{(3)}$} to produce two sequences {\scriptsize $\left\{\widehat W_{m,1}^{(t)}\right\}$} and {\scriptsize $\left\{\widehat W_{m,2}^{(t)}\right\}$} in memory $M_{m,1}$ and $M_{m,2}$, corresponding to rewards $+\widehat{\mathcal{R}}$ and $-\widehat{\mathcal{R}}$, respectively, and set the final weight {\scriptsize $\widehat{W}_{m} \gets \mathop{\arg\min}_{W\in \left\{\overline{\widehat{W}}_{m,1}^{\left(T_{m}^{(3)}\right)},\overline{\widehat{W}}_{m,2}^{\left(T_{m}^{(3)}\right)}\right\}} \left\|W-\widehat W_{m}^{(0)}\right\|_{\mathrm{F}}$}.
        
        \STATE \textbf{Estimator.}
        For each $m\in[k]$, compute $\widehat v_{2m-1}=\mathrm{top\text{-}eig}\left(\widehat W_{m}^{}\widehat W_{m}^\top\right)$ and $\widehat v_{2m}=\mathrm{top\text{-}eig}\left(\widehat W_{m}^\top\widehat W_{m}^{}\right)$.
        
        \RETURN Estimator $V = \{v_n\}_{n=1}^{\overline{k}}$.
    \end{algorithmic}
\end{algorithm}

\subsection{Streaming Multi-Phase Learning with Tailored Initialization}\label{sec-3.2}
Then we introduce the full algorithm in Algorithm~\ref{MSANSGA}.
\paragraph{Initialization.}
The weight matrices in Algorithm~\ref{MSANSGA} are initialized as follows:
\begin{align}\label{init-setting}
    W_{m}^{(0)}:=\pm\underbrace{\frac{d_{2m-1}^{-1/2}}{2}\begin{pmatrix} I_{d_{2m-1}} & 0_{d_{2m-1}\times(d_{2m}-d_{2m-1})}\end{pmatrix}}_{\widehat{W}_{m}^{(0)}}\pm\underbrace{\frac{1}{2}\overline{v}_{2m-1}^{(0)}\left(\overline{v}_{2m}^{(0)}\right)^{\top}}_{\widetilde{W}_{m}^{(0)}},\qquad \forall m\in[k],
\end{align}
where the first term $\widehat{W}_{m}^{(0)}$ is deterministic, and the second term $\widetilde{W}_{m}^{(0)}$ is random with $v_{2m-1}^{(0)}\sim \mathcal{N}(0,I_{d_{2m-1}}), \allowbreak v_{2m}^{(0)}\sim \mathcal{N}(0,I_{d_{2m}})$ independently, and $\overline{v}=v/\|v\|$. We run the algorithm over all $4^{k}$ sign combinations for the $\pm$ choices in  $\widehat{W}_{m}^{(0)}$ and $\widetilde{W}_{m}^{(0)}$ across $m \in [k]$. Among these $4^{k}$ initializations $\left\{W_{m}^{(0,\tau)}\right\}_{\tau\in\left[4^{k}\right]}$, there exists at least one $\tau$ such that for every $m$, the signal correlations  $\left\langle\widehat{W}_{m}^{(0,\tau)},v_{2m-1}^{*}v_{2m}^{*\top}\right\rangle$ and $\left\langle\widetilde{W}_{m}^{(0,\tau)},v_{2m-1}^{*}v_{2m}^{*\top}\right\rangle$ are both positive---though the initial correlations may be very small. The algorithm then amplifies these correlations over time. In this initialization, the deterministic component leverages over-parameterization to reduce optimization difficulty, while the random component extracts the underlying signal structure, thereby enhancing adaptivity.

\paragraph{Multi-Phase Training.}
We employ a multi-phase training strategy. During the first phase (first $\sum_{h=1}^{h^{*}} T_{m,h}^{(1)}$ iterations for each $m\in[k]$), we perform Sequential Normalized SGA as described in Algorithm~\ref{SGA-con-ss}. Since the inner algorithm ensures controlled learning rates across different components, the step sizes are scheduled to increase as the signal magnitude grows. This achieves the fastest possible learning rate while maintaining controlled noise levels. In the second phase, we continue the updates with a constant step size schedule to further amplify the correlations. By the end of the second phase, the signal correlations reach a constant level, implying that Problem~\ref{PCA-def} is already solved in the sense that the estimation error is $o(1)$. We run the first two phases using $4^{k}$ different initializations, and select the best result based on performance on a streaming validation set (see Step~3 in Algorithm~\ref{MSANSGA}). In the final phase, we switch to SNSGA with exponentially decaying step sizes (Algorithm~\ref{SGA-decay-ss}) to guarantee non-asymptotic convergence of each correlation $\left|\left\langle \overline{W}_{m}, v_{2m-1}^{*} v_{2m}^{*\top} \right\rangle\right|$ to $1$, thereby refining the estimate to high accuracy. Finally, an extraction step recovers the individual vectors $v_{i}^{*}$ from the estimated matrices.

\paragraph{Resource Usage and Streaming Learning.}
Our algorithm requires storing $4^{k}$ sets of weight matrices, resulting in a total memory cost of $k \cdot 4^{k} \cdot d^{2}$, where the exponent of $d$ is $2$, independent of $\overline{k}$. The algorithm operates in a streaming fashion: at each iteration, it computes the gradient using only a single fresh data point, corresponding to a single pass over the data ($K = 1$). As shown in Theorem~\ref{main-theorem}, signal recovery is achieved with a sample complexity of $\widetilde{\mathcal{O}}\left(d^{\overline{k}-2}\right)$, matching the lower bound given by Theorem~2 in \cite{dudeja2024statistical} up to logarithmic factors.

\section{Main Results}\label{main_result}
In this section, we present the main theoretical guarantees of our algorithm focusing on the even-order case $\overline{k}=2k$. The odd-order case is discussed separately in Section~\ref{extension}.

\subsection{Effective Signal Strength Quantities}
We first introduce several quantities that capture the correlation structure between adjacent signal vectors and determine the effective signal strength in our analysis. For each $m\in[k]$, define
{\small
\begin{align}\label{def-gamma-m}
    \gamma_{m} \coloneqq \frac{\mathrm{Cor}\left(v_{2m-1}^{*},v_{2m}^{*}\right)}{\left(d_{2m-1}d_{2m}\right)^{1/4}}+\frac{\mathsf{c}_{0}}{\left(d_{2m-1}d_{2m}\right)^{1/2}}.
\end{align}}%
According to the initialization specified in Section~\ref{sec-4.2}, for a given failure probability $\delta \in (0,1/2)$, the hyper-parameter $\mathsf{c}_{0}$ in Eq.~\eqref{def-gamma-m} is defined as
{\small
\begin{align}\label{c0-definition}
	\mathsf{c}_{0}=\min_{n\in\left[\overline{k}\right]}\left\{t_{d_{n}-1,\left(1+\delta/\overline{k}\right)/2}^{2}\cdot\left(1+t_{d_{n}-1,\left(1+\delta/\overline{k}\right)/2}/\sqrt{d_{n}}\right)^{-2}\right\},
\end{align}}%
where $t_{d_{n}-1,\left(1+\delta/\overline{k}\right)/2}$ denotes the  $\left(1+\delta/\overline{k}\right)/2$-quantile of a $t$-distribution with $d_{n}-1$ degrees of freedom for each $n\in\left[\overline{k}\right]$. $\mathsf{c}_0$ is at constant level if we treat $\overline{k}$ and $\delta$ as constants. The quantity $\gamma_{m}$ characterizes the effective signal strength of the $m$-th pair of adjacent vectors. When the vectors are strongly correlated, the first term in \eqref{def-gamma-m} dominates and $\gamma_{m} \asymp (d_{2m-1}d_{2m})^{-1/4}$, while under weak correlation $\gamma_{m} \asymp \mathsf{c}_0(d_{2m-1}d_{2m})^{-1/2}$. Based on $\{\gamma_{m}\}_{m=1}^{k}$ defined above, we define the effective block SNRs for \emph{Phase~I} and \emph{Phase~II} in Eq.~\eqref{block-eff-SNR-phase-I} and Eq.~\eqref{block-eff-SNR-phase-II}, which satisfy the following inequalities:
{\small
\begin{equation}\label{block-eff-SNR}
	\begin{aligned}
		\lambda_{m,h}^{\mathrm{I}}\geq \lambda\prod_{i\in\left[\left\lceil\overline{k}/2\right\rceil\right]}^{i\neq m}\left(\frac{\zeta_{i}^{h-1}\gamma_{i}}{4}\right), \quad
		\lambda_{m}^{\mathrm{II}}\geq \lambda\left(1-\frac{1}{\left\lfloor\overline{k}/2\right\rfloor}\right)^{\left\lceil\overline{k}/2\right\rceil-m}\prod_{i=1}^{m-1}\gamma_{i}^{1-\frac{1}{\left\lfloor\overline{k}/2\right\rfloor-1}},
	\end{aligned}
\end{equation}}%
for any $m\in\left[\left\lceil\overline{k}/2\right\rceil\right]$ and $h\in[h^{*
}]$, where $\zeta_{m}=\exp\left\{\left(\log(4)+\frac{\log(\gamma_{m}^{-1})}{\left\lfloor\overline{k}/2\right\rfloor-1}\right)/h^{*}\right\}$.

\subsection{Signal Recovery Guarantees}\label{sec-4.2}
The following theorem establishes the guarantee of our algorithm for solving Problem~\ref{PCA-def}. It shows that Algorithm~\ref{MSANSGA} achieves accurate recovery of all signal vectors with the desired sample complexity.
\begin{theorem}[Main Theorem]\label{main-theorem}
	Consider the asymmetric tensor PCA problem \ref{PCA-def} of even order $\overline{k} = 2k \geq 4$ under Assumptions \ref{ass-pro} and \ref{ass-base}. 
	Given failure probability $\delta\in(0,1/2)$, let the total sample size be $N = 2^{\overline{k}}\left(N_{0} + \sum_{m=1}^{\overline{k}/2}\sum_{l=1}^{2}T_{m}^{(l)}\right)+2\sum_{m=1}^{\overline{k}/2}T_{m}^{(3)}$ where $N_{0}\gtrsim 1$, $T_{m}^{(1)}=\sum_{h=1}^{h^{*}} T_{m,h}^{(1)}$ and $h^{*}=\left\lceil\log(4) + \frac{2}{\overline{k}-2}\max_{m}\{\log(\gamma_{m}^{-1})\}\right\rceil$. Then there exist appropriate hyper-parameters setting such that Algorithm~\ref{MSANSGA} outputs estimators $\{\widehat v_{n}\}_{n=1}^{\overline{k}}$ satisfying, with probability at least $1-2\delta$,
	\begin{equation}\label{convergence-eq}
		\max_{n\in\left[\overline{k}\right]} \mathcal{L}\left(\widehat{v}_{n},v_{n}^{*}\right) \lesssim\frac{\log\left(T^{(3)}\right)\sqrt{\overline{k}}}{\lambda\sqrt{\delta T^{(3)}}}.
	\end{equation}
	Here, the required sample complexity terms are given by
	\begin{equation}\label{hyper-para-T-even}
		T_{m,h}^{(1)} \asymp \frac{\zeta_{m}^{h}\gamma_{m}}{\eta_{m,h}^{(1)}\lambda_{m,h}^{\mathrm{I}}}, \qquad T_{m}^{(2)} \asymp \frac{1}{\eta_{m}^{(2)}\lambda_{m}^{\mathrm{II}}}, \qquad T_{m}^{(3)}\equiv T^{(3)}\gtrsim\max\left\{\overline{k},\frac{\left(\overline{k}d_{\overline{k}-1}d_{\overline{k}}\right)^{2}}{\lambda^{2}}\right\}\cdot\log^{4}\left(\frac{\overline{k}d_{\overline{k}}T^{(3)}}{\delta}\right),
	\end{equation}
	and the corresponding step-sizes are chosen as
	{\small
    \begin{equation}\label{hyper-para-eta-even}
		\begin{aligned}
			\eta_{m,h}^{(1)} \lesssim \frac{1}{\mathsf{c}_{2}}\min\left\{\frac{\lambda_{m,h}^{\mathrm{I}}}{\zeta_{m}^{h}\gamma_{m}d_{2m-1}d_{2m}},\frac{\zeta_{m}^{h-1}\gamma_{m}}{\lambda_{m,h}^{\mathrm{I}}+1}\right\},\, \eta_{m}^{(2)} \lesssim \frac{1}{\mathsf{c}_{2}}\min\left\{\frac{\lambda_{m}^{\mathrm{II}}}{\overline{k}d_{2m-1}d_{2m}},\frac{1}{\lambda_{m}^{\mathrm{II}}+1}\right\},\, \eta_{m}^{(3)}\asymp\frac{\log^{2}\left(T^{(3)}\right)}{\lambda T^{(3)}},
		\end{aligned}
	\end{equation}}%
    for any $m\in\left[\overline{k}/2\right]$ and $h\in[h^{*}]$, where $\mathsf{c}_{2}\asymp \overline{k}\log\left(\overline{k}d_{\overline{k}}/\delta\right)$.
\end{theorem}
A proof sketch of Theorem~\ref{main-theorem} is provided later in Section~\ref{sketch}. We next discuss several implications of the theorem, beginning with a corollary for the equal-dimension case $d_{1}=\cdots=d_{\overline{k}}$.
\begin{corollary}\label{cor-equal-dim}
	Under the setting $d=d_{1}=\cdots=d_{\overline{k}}$, suppose that Assumptions \ref{ass-pro} and \ref{ass-base} hold. Given failure probability $\delta\in(0,1/2)$, the hyper-parameter $\mathsf{c}_{0}$ in Eq.~\eqref{def-gamma-m} is defined as
	\begin{equation}\label{c0-all-d}
    	\mathsf{c}_{0}=t_{d-1,(1+\delta/\overline{k})/2}^{2}\cdot\left(1+t_{d-1,(1+\delta/\overline{k})/2}/\sqrt{d}\right)^{-2}.
	\end{equation}
    Theorem \ref{main-theorem} implies that there exists a suitable step-size schedule under which Algorithm \ref{MSANSGA} achieves strong recovery of each $v_{n}^{*}$ provided the sample size $N$ satisfies
    %这里带k^{3}是因为一个k来自于\mathsf{c}_2, 一个来自eta^{(2)}的1/k, 一个k来自于k个block
	{\small
    \begin{equation}
		\begin{aligned}
            \lambda^2N \gtrsim 2^{\overline{k}}\overline{k}^{3}\cdot\min\left\{\log\left(\frac{\overline{k}d}{\delta}\right)\cdot d^2\max_{m\in[\overline{k}/2]}\left\{\gamma_{m}^{4}\right\}\cdot\prod_{m=1}^{\overline{k}/2}\gamma_{m}^{-2},\quad \log^3\left(\frac{\overline{k}d}{\delta}\right)\cdot\mathsf{c}_0^{4-\overline{k}}\cdot d^{\overline{k}-2}\right\},\notag
			% \lambda^{2}N \gtrsim \log\left(\frac{kd}{\delta}\right)4^{\overline{k}}\overline{k}^{2}\cdot\min\left\{\max_{m\in[\overline{k}/2]}\left(\frac{\left|\left\langle v_{2m-1}^{*},v_{2m}^{*}\right\rangle\right|d^{1/2}+\mathsf{c}_{0}}{d^{1/2}}\right)^{4}\prod_{m=1}^{\overline{k}/2}\left(\frac{d^{1/2}}{\left|\left\langle v_{2m-1}^{*},v_{2m}^{*}\right\rangle\right|+\mathsf{c}_{0}d^{-1/2}}\right)^{2}, \mathsf{c}_0^{4-\overline{k}}d^{\overline{k}-4}\right\}.\notag
		\end{aligned}
	\end{equation}}%
	with probability at least $1-2\delta$.
\end{corollary}

It is worth noting that the initialization parameter $\gamma_{m}$ in Theorem~\ref{main-theorem} depends on the unknown population quantities $v_{2m-1}^{*}$ and $v_{2m}^{*}$, making it infeasible for direct implementation. To address this issue, we propose an adaptive hyper-parameter configuration in Section~\ref{para-setting-empirical} that avoids such oracle dependence under the preliminary case of $\lambda\asymp 1$, with sample complexity given in Corollary~\ref{empirical-hyper-para-convergence}.
\begin{corollary}\label{empirical-hyper-para-convergence}
	Consider $\lambda\asymp 1$ in Problem \ref{PCA-def}. Under the setting $d=d_{1}=\cdots=d_{\overline{k}}$, suppose that Assumption \ref{ass-base} holds. Given failure probability $\delta\in(0,1/2)$, we run Algorithm \ref{search-c_3-even} to search a reference parameter $\mathsf{c}_{3}$ that approximates $\left|\prod_{m=1}^{\overline{k}/2}\left\langle v_{2m-1}^{*},v_{2m}^{*}\right\rangle\right|^{2/\overline{k}}$. If the search succeeds (i.e., \textbf{Case~I} in Section \ref{para-setting-empirical}), we adopt the hyper-parameter settings of Algorithm~\ref{MSANSGA} specified in Eq.~\eqref{empirical-sample-size-caseI} and Eq.~\eqref{empirical-set-size-caseI}. In this case, the sample complexity required to solve Problem~\ref{PCA-def} satisfies
    %这里多出一个k^{2}是因为步长设定时要乘1个1/k^{2} 
    $N \gtrsim{2^{\overline{k}}\overline{k}^{5}} \cdot \frac{d^{\overline{k}/2}}{\mathsf{c}_{3}^{\overline{k}}}\cdot \log^{3}\left(\frac{\overline{k}d}{\delta}\right)\left\lceil\log\left(\frac{1}{\mathsf{c}_3}\right)\right\rceil$, where $\mathsf{c}_{3}\in\left[d^{-1/2+2/\overline{k}},1\right]$. Moreover, an additional $\mathcal{O}\left(\left\lceil\log\left(\frac{1}{\mathsf{c}_3}\right)\right\rceil\right)$ pre-sampling steps are required to locate the suitable $\mathsf{c}_3$. Otherwise (i.e., \textbf{Case~II} in Section \ref{para-setting-empirical} holds), we adopt the hyper-parameter settings specified in Eq.~\eqref{empirical-sample-size-caseII} and Eq.~\eqref{empirical-set-size-caseII}. In this case, the required sample complexity satisfies $N \gtrsim \frac{2^{\overline{k}}\overline{k}^{3}}{\mathsf{c}_{0}^{\overline{k}-4}} \cdot d^{\overline{k}-2}\cdot\log^{3}\left(\frac{\overline{k}d}{\delta}\right)\left\lceil\log(d)\right\rceil$, where $\mathsf{c}_{0}$ defined in Eq.~\eqref{c0-all-d}.
\end{corollary}

\paragraph{Weak-Correlation Regime (Worst Case).}
In the worst-case scenario, consecutive signal vectors are weakly correlated, i.e., {\small $\left|\prod_{m=1}^{\overline{k}/2}\left\langle v_{2m-1}^{*},v_{2m}^{*}\right\rangle\right|\lesssim d^{-\overline{k}/4}$}. Applying Corollary~\ref{cor-equal-dim} and Corollary~\ref{empirical-hyper-para-convergence}, Algorithm \ref{MSANSGA} requires at most {\small $\widetilde{\mathcal{O}}\left(\frac{2^{\overline{k}}\overline{k}^{3}}{\mathsf{c}_{0}^{\overline{k}-4}}\cdot d^{\overline{k}-2}+\frac{\overline{k}^{2}}{\delta\epsilon^{2}}\right)$} critical threshold of sample complexity $\lambda^2N$ to ensure that {\small $\max_{n\in\left[\overline{k}\right]} \mathcal{L}\left(\widehat{v}_{n},v_{n}^{*}\right) \leq \epsilon$} with probability at least $1-2\delta$ and yields a memory–runtime–sample complexity {\small $\lambda^2NKS \leq \widetilde{\mathcal{O}}\left(d^{\overline{k}}\right)$} by treating $\overline{k}$, $\delta$, and $\epsilon$ as constants, which matches the memory–runtime–sample complexity lower bound established by \cite{dudeja2024statistical} up to logarithmic factors. In contrast, using exact vectorization parameterization, (stochastic) gradient descent requires $d^{\overline{k}-1}$ samples even for the simpler STPCA problem, implying the sample efficiency improvement achieved by mild over-parameterization.

\paragraph{Strong-Correlation Regime (Model Adaptivity).}
In the strong-correlation regime, where consecutive signal vectors are sufficiently correlated, i.e., {\small $\left|\prod_{m=1}^{\overline{k}/2}\left\langle v_{2m-1}^{*},v_{2m}^{*}\right\rangle\right|\gtrsim 1$}, our algorithm adapts effectively to the problem structure. Specifically, applying Corollary~\ref{cor-equal-dim} and Corollary~\ref{empirical-hyper-para-convergence}, Algorithm~\ref{MSANSGA} requires at most {\small $\widetilde{\mathcal{O}}\left(2^{\overline{k}}\overline{k}^{5}\cdot d^{\overline{k}/2}+\frac{\overline{k}^{2}}{\delta\epsilon^{2}}\right)$} critical threshold of sample complexity $\lambda^2N$ to ensure that {\small $\max_{n\in\left[\overline{k}\right]} \mathcal{L}\left(\widehat{v}_{n},v_{n}^{*}\right) \leq \epsilon$} with probability at least $1-2\delta$. This regime naturally holds for STPCA \citep{montanari2014statistical}, where all signal vectors are identical. In this case, by treating $\overline{k}$, $\delta$, and $\epsilon$ as constants, the complexity simplifies to $\widetilde{\mathcal{O}}\left(d^{\overline{k}/2}\right)$, matching the best known rate achieved by polynomial-time algorithms \citep{hopkins2015tensor, hopkins2016fast, ding2025near}, highlighting our algorithm's adaptivity.

\paragraph{Convergence Rate.}
In addition to the sample complexity guarantees discussed above, Theorem~\ref{main-theorem} also characterizes the convergence rate of the recovery error. Consider the regime where the target accuracy satisfies $\epsilon \ll 1/d$. Treating $\delta$ as a constant, our algorithm recovers $\{v_{n}^{*}\}_{n=1}^{\overline{k}}$ with $\epsilon$-accuracy and constant probability achieving {\small $\widetilde{\mathcal{O}}\left(\mathrm{Poly}_{\overline{k}-2}(d)+\overline{k}^{2}\epsilon^{-2}\right)$} sample complexity. In particular, the dependence on $\epsilon$ improves over the $\mathcal{O}\left(d\epsilon^{-2}\right)$ convergence rate achieved by the ALS algorithm \citep{tang2025revisit}. Moreover, suppose Assumption~\ref{ass-base} is replaced by the following noise tensor assumption: $\mathbf{E}$ has zero mean and i.i.d. Gaussian entries with unit variance. Under the hyperparameter setting of Theorem~\ref{main-theorem}, Corollary \ref{main-result-gaussian-ass-even} implies that Algorithm~\ref{MSANSGA} outputs estimators $\{\widehat{v}_{n}\}_{n=1}^{\overline{k}}$ satisfying {\small $\max_{n\in\left[\overline{k}\right]} \mathcal{L}\left(\widehat{v}_{n},v_{n}^{*}\right) \leq \widetilde{\mathcal{O}}\left(\frac{d^2\overline{k}^{3/2}}{\sqrt{\delta}T^{(3)}}\right)$} with probability at least $1-2\delta$. Treating $\delta$ as a constant, the $\epsilon$-dependent term in the sample complexity is further reduced to {\small $\widetilde{\mathcal{O}}\left(d^2\overline{k}^{5/2}\epsilon^{-1}\right)$}.

\section{Proof Sketch}\label{sketch}
We establish the convergence of Algorithm~\ref{MSANSGA} through a three-phase analysis via the alignment metric $\alpha_{m}^{(t)}$ (defined in Eqs.~\eqref{def-alpha} and \eqref{def-alphat} in Section~\ref{pre-and-not-in-appendix}), which quantifies the alignment between the matrix $W_{m}$ and the correlated signal direction. The phases correspond to (i) escaping the noise-dominated regime, (ii) reaching a constant alignment level, and (iii) achieving local linear convergence. For ease of exposition, throughout this section we state some rate expressions in the weak-correlation regime defined in Section~\ref{sec-4.2}.

\paragraph{Phase~I: Weak Alignment.}
Phase~I aims to lift the initialization $\alpha_{m}^{(0)}\asymp \gamma_{m}$ to the level $2\gamma_{m}^{1-\frac{1}{k-1}}$, providing Phase~II with a sufficiently large initial SNR, which is critical for achieving the optimal sample complexity in Phase~II. Phase~I is divided into $h^{*}$ geometrically expanding stages. Specifically, we partition the interval into $\left[\mathrm{LB}_{m,h}^{\mathrm{I}}, \mathrm{UB}_{m,h}^{\mathrm{I}}\right]$ with $\mathrm{LB}_{m,h}^{\mathrm{I}}\asymp \gamma_{m}\zeta_{m}^{h-1}$ and $\mathrm{UB}_{m,h}^{\mathrm{I}}\asymp \gamma_{m}\zeta_{m}^{h}$. For each stage, the step size {\small $\eta_{m,h}^{(1)} \asymp \frac{\lambda_{m,h}^{\mathrm{I}}}{\zeta_{m}^{h}\gamma_{m}d_{2m-1}d_{2m}}$} is chosen to satisfy Eq.~\eqref{eta-I-constraint}, which ensures the alignment recursion {\small $\alpha_{m,h}^{(t+1)} \geq \alpha_{m,h}^{(t)} + \frac{\eta_{m,h}^{(1)}\lambda_{m,h}^{\mathrm{I}}}{4} + \eta_{m,h}^{(1)} \xi_{m,h}^{(t+1)}$} (shown in Eqs.~\eqref{dynamic-hatalpha-lower-bound-proof} and \eqref{ascent-dynamic-alpha}), where $\xi_{m,h}^{(t+1)}$ is a zero-mean sub-Gaussian martingale difference. Detailed results can be found in Theorem~\ref{thm-phase-I-tensor-PCA}.

To control stochastic fluctuations, we decompose the failure of stage ascent into three events: (i) the iterate drops from $2\mathrm{LB}_{m,h}^{\mathrm{I}}$ below $\mathrm{LB}_{m,h}^{\mathrm{I}}$, (ii) the process fails to reach $2\mathrm{UB}_{m,h}^{\mathrm{I}}$ within the allotted time, and (iii) after reaching $2\mathrm{UB}_{m,h}^{\mathrm{I}}$, the iterate falls below $\mathrm{UB}_{m,h}^{\mathrm{I}}$. Events (i) and (iii) are controlled via martingale concentration and yield upper bounds on the stage duration, while event (ii) provides a matching lower bound ensuring sufficient ascent. Combining these bounds yields {\small $T_{m,h}^{(1)} \asymp \frac{\mathrm{UB}_{m,h}^{\mathrm{I}}-\mathrm{LB}_{m,h}^{\mathrm{I}}}{\eta_{m,h}^{(1)}\lambda_{m,h}^{\mathrm{I}}}\asymp \frac{\zeta_{m}^{h}\gamma_{m}}{\eta_{m,h}^{(1)}\lambda_{m,h}^{\mathrm{I}}} \asymp \frac{\left(\zeta_{m}^{h}\gamma_{m}\right)^{2}d_{2m-1}d_{2m}}{\left(\lambda_{m,h}^{\mathrm{I}}\right)^{2}} \asymp d^{2k-2}$} in the worst case.

It is worth noting that splitting Phase~I into $h^{*}$ stages is essential for achieving the optimal sample complexity. This design enables proper adjustment of step sizes and compensates for the stage-dependent growth of $\lambda_{m,h}^{\mathrm{I}}$, ensuring each stage maintains a similar duration. Summing over all stages yields {\small $T_{m}^{(1)}=\sum_{h=1}^{h^{*}} T_{m,h}^{(1)} \lesssim h^{*}T_{m,1}^{(1)}=\widetilde{\mathcal{O}}\left(d^{2k-2}\right)$}, matching the sample complexity in the worst case. Since {\small $h^{*}\gtrsim \log_{\zeta_{m}}\left(\gamma_{m}^{-1/(k-1)}\right)$}, the total duration of Phase~I scales as $T_{m,1}^{(1)}$, up to logarithmic factor with respect to $1/\gamma_{m}$. While without partitioning (i.e., $h^{*}=1$), the step size would be constrained by the smallest alignment level as {\small $\breve{\eta}_{m}^{(1)}\asymp \frac{\lambda_{m}^{\mathrm{I}}}{\gamma_{m}^{1-1/(k-1)}d_{2m-1}d_{2m}}$}, leading to a stage length of order {\small $\breve{T}_{m}^{(1)}\asymp\frac{\gamma_{m}^{2-2/(k-1)}d_{2m-1}d_{2m}}{\left(\lambda_{m}^{\mathrm{I}}\right)^{2}}\asymp\gamma_{m}^{-2/(k-1)}T_{m,1}^{(1)}$}, showing the necessity of introducing $h^{*}$.

\paragraph{Phase~II: Strong Alignment.}
Phase~II aims to elevate the alignment from {\small $\alpha_{m} \geq \gamma_{m}^{1-\frac{1}{k-1}}$} to $1-\widetilde{\epsilon}$, transitioning the iterates to a signal-dominated regime. The step size {\small $\eta_{m}^{(2)} \asymp \frac{\widetilde{\epsilon}\lambda_{m}^{\mathrm{II}}}{d_{2m-1}d_{2m}}$} is chosen to satisfy Eq.~\eqref{eta-II-constraint}, which ensures that the ascent recursion used in Phase~I remains valid while controlling higher-order terms. The analysis follows the same event-decomposition argument as in Phase~I. Events corresponding to dropping below $\mathrm{LB}_{m}^{\mathrm{II}}$ and failing to reach $1-\widetilde{\epsilon}/2$ are controlled using the same ascent dynamics. Once {\small $\alpha_{m}^{(t)}\geq 1-\widetilde{\epsilon}/2$}, the update admits the contraction recursion {\small $1-\alpha_{m}^{(t+1)} \leq \left(1-(1-\widetilde{\epsilon})\eta_{m}^{(2)}\lambda_{m}^{\mathrm{II}}\right) \left(1-\alpha_{m}^{(t)}\right) - \eta_{m}^{(2)}\xi_{m}^{(t+1)}$} as shown in Eq.~\eqref{martingale-concen-lower-bound}. Combining these bounds yields {\small $T_{m}^{(2)} \asymp \frac{1}{\eta_{m}^{(2)}\lambda_{m}^{\mathrm{II}}} \asymp \frac{1}{\left(\lambda_{m}^{\mathrm{II}}\right)^{2}d_{2m-1}d_{2m}} \asymp d^{2k-2} \lesssim T_{m}^{(1)}$}, after which the alignment satisfies $\alpha_{m} \geq 1-\widetilde{\epsilon}$ with high probability, establishing strong alignment. Detailed results can be found in Theorem~\ref{thm-phase-II-tensor-PCA}.

\paragraph{Phase~III: Estimation.}
Phase~III focuses on establishing the convergence rate of the alignment error, where $\alpha_{m}$ behaves similarly to a linear regression process. Once $\alpha_{m}$ enters the region $\alpha_{m}^{(t)} \in \left[1-\frac{3}{2}\widehat{\epsilon}, 1\right]$, it remains within this range throughout this phase with high probability, which is proven using similar techniques to those in Phase~II (see Lemma~\ref{phase-III-high-probability}). Given this, the dynamics in Phase~III follow a refined local recursion {\small $\mathbb{E} \left[\left(\widehat{\beta}_{m}^{(t+1)}\right)^{2}\right] \leq \left(1-\eta_{m}^{(3,t)}\lambda_{m}^{\mathrm{III}}\left(1-\frac{3}{2}\widehat{\epsilon}\right)\right)\mathbb{E}\left[\left(\widehat{\beta}_{m}^{(t)}\right)^{2}\right] + \left(\eta_{m}^{(3,t)}\right)^{2}g_{m}^{(t)}$} (see Lemma~\ref{phase-III-error-recursion-sub-Gaussian}), where the alignment error $\widehat{\beta}_{m}^{(t)}$ informally can be viewed as $1-\alpha_{m}^{(t)}$, and $g_{m}^{(t)}$ captures higher-order remainder terms and stochastic fluctuations arising from the local expansion of the dynamics. Using the decaying step-size schedule $\eta_{m}^{(3,t)}=\eta_{m}^{(3)}2^{-\left\lfloor t/\widehat{T}_{m}^{(3)}\right\rfloor}$, we analyze this recursion via a bias--variance decomposition. The resulting error bound consists of (i) a geometrically decaying term governed by the contraction factor $\eta_{m}^{(3)} \lambda_{m}^{\mathrm{III}}\left(1-\frac{3}{2}\widehat{\epsilon}\right)$, (ii) a higher-order bias induced by the step-size truncation, and (iii) a variance term capturing accumulated stochastic noise. Detailed results can be found in Theorem~\ref{thm-phase-III-tensor-PCA-sub-Gaussian}. Finally, by performing spectral analysis of the associated matrix, we obtain the estimation error bound Eq.~\eqref{convergence-eq}.

\clearpage
\appendix
\section{Organization}
After introducing notations in Section~\ref{intro}, we set up the problem in Section~\ref{setup}. Sections~\ref{alg} and \ref{main_result} present our algorithm and the main result, respectively. Section~\ref{sketch} provides a proof sketch of the main theorem. The remaining sections are organized as follows. Section~\ref{relate} reviews the related work. Section~\ref{extension} presents the algorithm and results for the odd-order case. Section~\ref{pre-and-not-in-appendix} establishes the analytical framework and introduces the alignment parameter that underlies the three-phase analysis. Sections~\ref{phase-I}, \ref{phase-II}, and \ref{phase-III} prove the main theorem via a three-phase analysis of the alignment parameter. Specifically, Section~\ref{phase-I} shows that, starting from a weakly correlated initialization, the alignment parameter reaches the weak-alignment regime with high probability; Section~\ref{phase-II} analyzes the signal-dominated regime and proves the transition from weak to strong alignment; and Section~\ref{phase-III} studies the local regime around the target component and derives the final error bound. Sections~\ref{proof-even} and \ref{proof-odd} provide the proofs of the main theorem in even and odd cases, respectively. Section~\ref{tech-lem} collects the technical estimates and concentration inequalities used throughout the analysis, and Section~\ref{aux-lem} provides additional auxiliary lemmas used in the proofs. Finally, Section~\ref{sup-alg} presents some algorithm variants not shown in Section~\ref{alg}.

\section{Related Work}\label{relate}
\paragraph{Symmetric TPCA.}
In the absence of memory constraints, a major research direction focuses on solving the problem using polynomial-time algorithms, with various evidence suggesting that the regime $d < \lambda^{2} N \ll d^{\overline{k}/2}$ is intrinsically hard for such algorithms. When $\lambda^{2} N \gtrsim d^{\overline{k}/2}$, efficient algorithms include Sum of Squares relaxations with enhanced trace methods \citep{hopkins2015tensor, hopkins2016fast, hopkins2017power, raghavendra2018high}, spectral methods \citep{montanari2014statistical, hopkins2015tensor, zheng2015interpolating, hopkins2016fast, hopkins2017power, biroli2020iron}, tensor power methods with global initialization \citep{anandkumar2017homotopy, biroli2020iron}, and higher-order generalizations of belief propagation \citep{zdeborova2016statistical, bandeira2018notes}. Another line of research investigates gradient-based algorithms, motivated by their widespread use and their connection to finding ground states in spin-glass-like landscapes—a fundamental problem in statistical physics. When the model is exactly parameterized via vectorization, (stochastic) gradient descent-based algorithms require $d^{\overline{k}-1}$ samples—a rate conjectured by \cite{arous2019landscape, arous2021online} to be unimprovable based on a local landscape analysis near the true signal. In contrast, \cite{ding2025near} propose a matrix parameterization method that achieves $d^{\left\lceil\overline{k}/2\right\rceil}$ sample complexity. Our paper builds on \cite{ding2025near} by adopting matrix parameterization, but working on the harder asymmetric setting: we pair consecutive vectors into matrices and introduce alternating updates, a new initialization strategy, and multi-phase training. As a result, our algorithm achieves adaptivity and near-optimal complexity under both symmetric and asymmetric criteria. Finally, \cite{dudeja2024statistical} establishes a $d^{\left\lceil\left(\overline{k}+1\right)/2\right\rceil}$ memory–runtime–sample complexity lower bound, suggesting that memory constraints are less critical in the symmetric setting.

\paragraph{Asymmetric TPCA.}
ATPCA allows the unit vectors to be non-identical, making it strictly harder than its symmetric counterpart. As has been mentioned, the memory–runtime–sample complexity lower bound is established by \cite{dudeja2024statistical}, which characterizes the fundamental trade-offs among memory usage, computational time, and sample complexity for this problem. On the algorithmic front, several provably efficient approaches have been studied for related tensor decomposition problems, including tensor unfolding and spectral methods based on matricizations of the observed tensor \citep{montanari2014statistical, zheng2015interpolating}. However, these approaches typically require storing large intermediate statistics whose memory cost scales on the order of $d^{\overline{k}/2}$, making them less suitable for memory-constrained settings. On the other hand, there are various empirical algorithms for ATPCA and general CP decomposition, including alternating least squares \citep{carroll1970analysis, harshman1970foundations, kolda2009tensor} and tensor power methods \citep{anandkumar2014tensor}. We hope that the techniques introduced in our algorithm provide theoretical justification and practical improvements for these heuristic methods, and may inspire new insights for related problems beyond tensor PCA.

\paragraph{Over-Parameterization and Gradient Descent.}
Lifting the parameterization space is a common technique to alleviate optimization difficulty—for example, reparameterizing vectors as matrices in matrix decomposition \citep{tu2016low, ge2017no} and symmetric TPCA \citep{ding2025near}, or lifting parameters to distributions in particle methods \citep{wei2019regularization}, SoS for polynomial optimization \citep{parrilo2000structured} or tensor decomposition \citep{ma2016polynomial}, and mean-field analysis \citep{chizat2018global, mei2018mean}. In models exhibiting a statistical-to-computational gap, over-parameterization often reduces sample complexity by easing optimization, as observed in matrix sensing and symmetric TPCA. However, its role in enhancing model adaptivity—a key theme in statistical learning—has rarely been explored. Gradient Descent and its stochastic variants are workhorses of modern machine learning. Traditional analyses typically rely on convexity and smoothness assumptions, but recent work provides refined, model-specific guarantees for settings such as linear regression \citep{dieuleveut2016nonparametric, jain2018accelerating, ge2019step, zou2021benign, wu2022last}, matrix decomposition \citep{sun2016guaranteed, tu2016low, ge2017no}, and neural networks \citep{jacot2018neural, zou2018stochastic, du2019gradient, allen2019convergence}. These analyses often yield improved rates by exploiting problem structure. For linear regression, for instance, standard convex optimization guarantees a rate of $d\sigma^{2}/\sqrt{\epsilon}$, while tail-averaging and exponential step-size decay schemes achieve the asymptotically optimal rate $d\sigma^{2}/\epsilon$. This $\sqrt{\epsilon}$ improvement comes from carefully tracking gradient noise across iterations and directions, effectively localizing estimation errors. In the final convergence phase of our algorithm, we similarly adapt the dynamics to a linear regression subproblem and follow \cite{ge2019step, zou2021benign, wu2022last} to attain faster convergence.

\section{Extension}\label{extension}
In this section, we extend the results to the odd-order case $\overline{k}=2k+1$. The overall approach is similar to the even-order setting, with minor modifications in the indexing of adjacent vector pairs.

To solve Problem~\ref{PCA-def} in the odd-order setting, we introduce Algorithm~\ref{MSANSGA-odd} in Section~\ref{alg-odd-case}. Analogous to the even-order case, we define the effective signal strength quantities using shifted pairs of vectors. For each $m\in[k+1]$, define
\begin{align}\label{def-gamma-m-odd}
	\gamma_{m} \coloneqq \frac{\mathrm{Cor}\left(v_{2m-2}^{*},v_{2m-1}^{*}\right)}{\left(d_{2m-2}d_{2m-1}\right)^{1/4}}+\frac{\mathsf{c}_{0}}{\left(d_{2m-2}d_{2m-1}\right)^{1/2}},
\end{align}
where we adopt the convention that $v_{0}^{*}=0\in\mathbb{R}$ and $d_{0}=\mathsf{c}_{0}^{1/2}$. Based on these quantities, the effective block SNRs for \emph{Phase~I} and \emph{Phase~II} are defined analogously to the even-order case. The detailed definitions are given in Eq.~\eqref{lambda-phase-I_odd} and Eq.~\eqref{lambda-phase_odd}. These SNRs also satisfy the inequalities in Eq.~\eqref{block-eff-SNR}, with the index range adjusted to $m\in[k+1]$.
Using these quantities, we now state the result for Algorithm~\ref{MSANSGA-odd}.

\begin{theorem}\label{main-theorem-odd}
	Consider the asymmetric tensor PCA problem \ref{PCA-def} of odd order $\overline{k} = 2k+1 \geq 5$ under Assumptions \ref{ass-pro} and \ref{ass-base}. 
	Given failure probability $\delta\in(0,1/2)$, let the total sample size be
	$N = \text{\small$2^{\overline{k}}\left(N_{0} + \sum_{m=1}^{\lceil\overline{k}/2\rceil}\sum_{l=1}^{2}T_{m}^{(l)}\right)+2\sum_{m=1}^{\lceil\overline{k}/2\rceil}T_{m}^{(3)}$}$
	where $N_0\gtrsim1$,  $T_{m}^{(1)}=\sum_{h=1}^{h^{*}} T_{m,h}^{(1)}$ and  $h^{*}=$\linebreak$\left\lceil\log(4)+\frac{1}{\lceil\overline{k}/2\rceil-2}\max_{m}\{\log(\gamma_{m}^{-1})\}\right\rceil$. Then there exist appropriate hyper-parameters setting such that Algorithm~\ref{MSANSGA-odd} outputs estimators $\{\widehat v_{n}\}_{n=1}^{\overline{k}}$ satisfying, with probability at least $1-2\delta$,
	\begin{equation}
		\max_{n\in\left[\overline{k}\right]} \mathcal{L}\left(\widehat{v}_{n},v_{n}^{*}\right) \lesssim \frac{\log\left(T^{(3)}\right)\sqrt{\overline{k}}}{\lambda\sqrt{\delta T^{(3)}}}.
	\end{equation}
	Here, the required sample complexity terms are given by
	\begin{equation*}
		\begin{aligned}
			T_{1,h}^{(1)}&\asymp\frac{\zeta_{1}^{h}\gamma_{1}}{\eta_{1,h}^{(1)}\lambda_{1,h}^{\mathrm{I}}},\quad T_{1}^{(2)}\asymp\frac{1}{\eta_{1}^{(2)}\lambda_{1}^{\mathrm{II}}}, \quad T_{1}^{(3)}=T^{(3)}\gtrsim \log^4\left(\frac{\overline{k}d_{\overline{k}}T^{(3)}}{\delta}\right)\cdot\max\left\{\overline{k},\frac{\left(\overline{k}d_{\overline{k}-1}d_{\overline{k}}\right)^{2}}{\lambda^{2}}\right\}.
		\end{aligned}
	\end{equation*}
	For any $m\in[2:\lceil\overline{k}/2\rceil]$, $\left\{T_{m,h}^{(1)}\right\}_{h\in[h^{*}]}$ and $\left\{T_{m}^{(l)}\right\}_{l\in\{2,3\}}$ have the same forms as that defined in Eq.~\eqref{hyper-para-T-even}. The corresponding step-sizes are chosen as 
	\begin{equation*}
		\begin{aligned}
			\eta_{1,h}^{(1)} \lesssim \frac{1}{\mathsf{c}_2}\min\left\{\frac{\lambda_{1,h}^{\mathrm{I}}}{\zeta_{1}^{h}\gamma_{1}d_{1}},\frac{\zeta_{1}^{h-1}\gamma_{1}}{\lambda_{1,h}^{\mathrm{I}}+1}\right\}, \quad 
			\eta_{1}^{(2)} \lesssim \frac{1}{\mathsf{c}_2}\min\left\{\frac{\lambda_{1}^{\mathrm{II}}}{\overline{k}d_{1}},\frac{1}{\lambda_{1}^{\mathrm{II}}+1}\right\}, \quad 
			\eta_{1}^{(3)} \asymp \frac{\log^{2}\left(T^{(3)}\right)}{\lambda T^{(3)}},
		\end{aligned}
	\end{equation*} 
	where $\mathsf{c}_{2}\asymp \overline{k}\log(\overline{k}d_{\overline{k}}/\delta)$. For any $m\in[2:\lceil\overline{k}/2\rceil]$, $\left\{\eta_{m,h}^{(1)}\right\}_{h\in[h^{*}]}$ and $\left\{\eta_{m}^{(l)}\right\}_{l\in\{2,3\}}$ have the same forms as that defined in Eq.~\eqref{hyper-para-eta-even} with $(d_{2m-1},d_{2m})$ replaced by $(d_{2m-2},d_{2m-1})$. 
\end{theorem}
We next discuss several implications of Theorem \ref{main-theorem-odd}. In particular, the following corollary specializes to the equal-dimension setting $d_{1}=\cdots=d_{\overline{k}}$. In the odd case, the sample size and memory cost required for updating the vector parameter $w_1$ differ from those for updating matrix parameters. We therefore naturally extend the definition of the memory–runtime–sample complexity to the odd case as follows:
$$
\lambda^2NKS=\lambda^2K\sum_{m=1}^{k+1}N_m S_m,
$$
where $N_m$ denotes the sample size required to update the parameter of the $m$-th block, satisfying $N_m\leq 2^{\overline{k}}\left(N_0+\sum_{l=1}^{2}T_m^{(l)}\right)+2T_m^{(3)}$, and $S_m$ represents the memory cost per update of the $m$-th block parameter for any $m\in[\lceil\overline{k}/2\rceil]$.
\begin{corollary}\label{cor-equal-dim-odd}
	Under the setting $d=d_{1}=\cdots=d_{\overline{k}}$, suppose that Assumptions \ref{ass-pro} and \ref{ass-base} hold. Given failure probability $\delta\in(0,1/2)$, the hyper-parameter $\mathsf{c}_{0}$ in Eq.~\eqref{def-gamma-m} is defined as
    {\small
	\begin{equation}
    	\mathsf{c}_{0}=t_{d-1,(1+\delta/\overline{k})/2}^{2}\cdot\left(1+t_{d-1,(1+\delta/\overline{k})/2}/\sqrt{d}\right)^{-2}.
	\end{equation}}%
    Theorem \ref{main-theorem} implies that there exists a suitable step-size schedule under which Algorithm \ref{MSANSGA} achieves strong recovery of each $v_{n}^{*}$ provided the memory–runtime–sample complexity satisfies
    %这里带k^{3}是因为一个k来自于\mathsf{c}_2, 一个来自eta^{(2)}的1/k, 一个k来自于k个block
	{\small
    \begin{equation}
		\begin{aligned}
            \lambda^2NKS \gtrsim 4^{\overline{k}}\overline{k}^{3}\cdot\min\left\{\log\left(\frac{\overline{k}d}{\delta}\right)\cdot \max\left\{d^4\max_{m\in[2:\lceil\overline{k}/2\rceil]}\left\{\gamma_{m}^{4}\right\},\,  d^2\gamma_1^4\right\}\cdot\prod_{m=1}^{\lceil\overline{k}/2\rceil}\gamma_{m}^{-2},\quad \log^3\left(\frac{\overline{k}d}{\delta}\right)\cdot\mathsf{c}_0^{2-\overline{k}}\cdot d^{\overline{k}}\right\},\notag
			% \lambda^{2}N \gtrsim \log\left(\frac{kd}{\delta}\right)4^{\overline{k}}\overline{k}^{2}\cdot\min\left\{\max_{m\in[\overline{k}/2]}\left(\frac{\left|\left\langle v_{2m-1}^{*},v_{2m}^{*}\right\rangle\right|d^{1/2}+\mathsf{c}_{0}}{d^{1/2}}\right)^{4}\prod_{m=1}^{\overline{k}/2}\left(\frac{d^{1/2}}{\left|\left\langle v_{2m-1}^{*},v_{2m}^{*}\right\rangle\right|+\mathsf{c}_{0}d^{-1/2}}\right)^{2}, \mathsf{c}_0^{4-\overline{k}}d^{\overline{k}-4}\right\}.\notag
		\end{aligned}
	\end{equation}}%
	with probability at least $1-2\delta$.
\end{corollary}
In Section \ref{para-setting-empirical-odd}, we present an easily implementable adaptive hyper-parameter configuration under the preliminary case $\lambda\asymp1$. Its convergence guarantee is established in Corollary \ref{empirical-hyper-para-convergence-odd}. 
\begin{corollary}\label{empirical-hyper-para-convergence-odd}
	Consider $\lambda\asymp 1$ in Problem \ref{PCA-def}. Under the setting $d=d_{1}=\cdots=d_{\overline{k}}$, suppose that Assumption \ref{ass-base} holds. Given failure probability $\delta\in(0,1/2)$, we run Algorithm \ref{search-c_3-odd} to search a reference parameter $\mathsf{c}_3$ that approximates $\left|\left\langle w^{(0)},v_{1}^{*}\right\rangle\cdot\prod_{m=2}^{\lceil\overline{k}/2\rceil}\left\langle v_{2m-2}^{*},v_{2m-1}^{*}\right\rangle\right|^{1/\lceil\overline{k}/2\rceil}$. 
    % Following a similar approach as in Corollary \ref{empirical-hyper-para-convergence}, we introduce a reference parameter $c_{3}$. And $\left|\left\langle w^{(0)},v_{1}^{*}\right\rangle\cdot\prod_{m=2}^{k+1}\left\langle v_{2m-2}^{*},v_{2m-1}^{*}\right\rangle\right|$ is approximated by $c_{3}^{k+1}$. 
    If the search succeeds (i.e., \textbf{Case~I} in Section \ref{para-setting-empirical-odd}), we adopt the hyper-parameter settings of Algorithm \ref{MSANSGA-odd} specified in Eq.~\eqref{empirical-sample-size-caseI-odd} and Eq.~\eqref{empirical-set-size-caseI-odd}. In this case, the memory–runtime–sample complexity required to solve Problem \ref{PCA-def} satisfies
    \begin{equation*}
        NKS \gtrsim{\log\left(\frac{kd}{\delta}\right)^34^{\overline{k}}\overline{k}^{5}} \cdot \frac{d^{\left\lceil\overline{k}/2\right\rceil+2}}{c_{3}^{\overline{k}+1}},
	\end{equation*}
	where $c_{3}\in \left[d^{-1/2+3/(\overline{k}+1)},1\right]$. Moreover, an additional $\calO\left(\left\lceil\log\left(\frac{1}{\mathsf{c}_3}\right)\right\rceil\right)$ pre-sampling steps are required to locate the suitable $\mathsf{c}_3$. Otherwise (i.e., \textbf{Case~II} in Section \ref{para-setting-empirical-odd}), we adopt the hyper-parameter setting specified in Eq.~\eqref{empirical-sample-size-caseII-odd} and Eq.~\eqref{empirical-set-size-caseII-odd}. In this case, the required memory–runtime–sample complexity satisfies
    \begin{equation*}
        NKS \gtrsim \log^3\left(\frac{\overline{k}d}{\delta}\right)\frac{4^{\overline{k}}\overline{k}^{3}}{\mathsf{c}_{0}^{\overline{k}-2}} \cdot d^{\overline{k}},
	\end{equation*}
	where $\mathsf{c}_{0}$ defined in Eq.~\eqref{c0-all-d}.
\end{corollary}

\section{Supplementary Setup and Technical Lemmas}\label{pre-and-not-in-appendix}
In this section, provide additional preliminaries that were not discussed in the main text, along with several key lemmas that will be used in the subsequent analysis. These results help to establish the necessary framework for understanding the optimization procedure and the behavior of the algorithm. Specifically, we present the definitions and technical details regarding the stochastic gradient updates, the alignment measures, and the noise properties, which are crucial for deriving the theoretical guarantees and convergence rates in the tensor recovery problem.

\subsection{Block-wise Update Structure}

Recall the setup from Eq.~\eqref{def-model} where we consider the rank-one tensor model. The optimization procedure is organized as a block-wise stochastic gradient algorithm executed according to a phase-wise schedule. The algorithm consists of three phases, \textbf{Phase~I}, \textbf{Phase~II} and \textbf{Phase~III}.

Within each phase, the algorithm cycles through all blocks $m \in [k]$. For a given phase and a given block $m$, the algorithm performs a contiguous sequence of stochastic gradient updates on $W_{m}$, while keeping all other blocks fixed. In the $l$-th ($l\in[3]$) phase, each block $m$ is updated for $T_{m}^{(l)}$ consecutive iterations.

For the purpose of analysis, we study the normalized stochastic gradient ascent (SGA) (Algorithm~\ref{SGA-con-ss}) procedure at the granularity of a single phase and a single block. Specifically, we fix a phase and a block index $m\in[k]$, and consider the sequence of updates applied to the block variable $W_{m}$ during its associated inner-loop execution within that phase.

The iteration index $t$ denotes the local inner-loop counter corresponding to the updates of block $m$ within the fixed phase. At each inner-loop iteration $t$, the algorithm draws a fresh stochastic sample $\mathbf{T}^{(t+1)}$ and updates $W_{m}$ while keeping all other blocks unchanged. Throughout the block-level analysis, all blocks other than $m$ are fixed at their normalized values.

The stochastic objective restricted to block $m$ depends on the phase through the choice of fixed blocks. In the $l$-th phase ($l\in\{1,2,3\}$ corresponds to \textbf{Phase~I}, \textbf{Phase~II} and \textbf{Phase~III}, respectively), each preceding block $i$ is fixed at the values obtained after $\sum_{\tau=1}^{l}T_{i}^{(\tau)}$ inner-loop updates in the previous phases, while each succeeding block $j$ is fixed at the values obtained after $\sum_{\tau=1}^{l-1}T_{j}^{(\tau)}$ inner-loop updates in the previous phases. The stochastic objective in the $l$-th phase is defined as Eq.~\eqref{def-hatR}. We also have the following equivalent expression for $\widehat{\mathcal{R}}_{m}^{(t+1)}(\cdot)$ at the $l$-th phase,
\begin{align}\label{def-reward-m}
	\widehat{\mathcal{R}}_{m}^{(t+1)}(W_{m}) &= \lambda_{m}^{\mathrm{Phase}}\left\langle W_{m}^{(t)}, v_{2m-1}^{*} v_{2m}^{*\top}\right\rangle+\left\langle W_{m}^{(t)}, E_{m}^{(t+1)}\right\rangle,
\end{align}
for any $W_{m}\in\mathbb{R}^{d_{2m-1}\times d_{2m}}$, where
\begin{equation}\label{lambda-phase_{1}}
	\begin{aligned}
		\lambda_{m}^{\mathrm{Phase}}:=\lambda\cdot\prod_{i=1}^{m-1}\left|\left\langle\overline{W}_{i}^{\left(\sum_{\tau=1}^{l}T_{i}^{(\tau)}\right)}, v_{2i-1}^{*}v_{2i}^{*\top}\right\rangle\right|\cdot\prod_{j=m+1}^{k}\left|\left\langle\overline{W}_{j}^{\left(\sum_{\tau=1}^{l-1}T_{j}^{(\tau)}\right)}, v_{2j-1}^{*}v_{2j}^{*\top}\right\rangle\right|,
	\end{aligned}
\end{equation}
and $E_{m}^{(t+1)}$ denotes the effective noise tensor projected onto block $m$, i.e.,
\begin{equation}\label{eq:errort}
	\begin{aligned}
		\left\langle E_{m}^{(t+1)},Q \right\rangle &:= \left\langle \left(\prod_{i=1}^{m-1}\otimes \overline{W}_{i}^{\left(\sum_{\tau=1}^{l}T_{i}^{(\tau)}\right)}\right) \otimes Q \otimes \left(\prod_{j=m+1}^{k}\otimes \overline{W}_{j}^{\left(\sum_{\tau=1}^{l-1}T_{j}^{(\tau)}\right)}\right), \mathbf{E}^{(t+1)}\right\rangle,
	\end{aligned}
\end{equation}
for any $Q\in\mathbb{R}^{d_{2m-1}\times d_{2m}}$.

\subsection{Alignment Measure}

To track the alignment between $W_{m}^{(t)}$ and the planted rank-one direction $v_{2m-1}^{*}v_{2m}^{*\top}$, we first define a normalized alignment functional for a generic matrix $W\in\mathbb{R}^{d_{1}\times d_{2}}$ and unit vectors $v\in\mathbb{R}^{d_{1}}, u\in\mathbb{R}^{d_{2}}$:
\begin{equation}\label{def-alpha}
    \alpha(v,u,W):= \frac{\langle W, vu^{\top}\rangle}{\|W\|_{\mathrm F}},
\end{equation}
then we measure the alignment by
\begin{equation}\label{def-alphat}
    \alpha_{m}^{(t)}:=\alpha\left(v_{2m-1}^{*}, v_{2m}^{*}, W_{m}^{(t)}\right)=\frac{\left\langle W_{m}^{(t)},v_{2m-1}^{*}v_{2m}^{*\top}\right\rangle}{\left\|W_{m}^{(t)}\right\|_{\mathrm{F}}}.
\end{equation}
By construction, $\alpha(v,u,W)\in[-1,1]$ and is invariant under rescaling of $W$, i.e., $\alpha(v,u,cW)=\alpha(v,u,W)$ for any $c>0$.

A direct calculation shows that the gradient of the scaled objective
$\|W_{m}\|_{\mathrm F}\widehat{\mathcal{R}}_{m}^{(t+1)}(W_{m})$
admits the decomposition
\begin{align}\label{gradient-decom}
      G_{m}^{(t)} &= \nabla_{W_{m}}\left(\left\|W_{m}^{(t)}\right\|_{\mathrm{F}}\widehat{\mathcal{R}}_{m}^{(t+1)}\left(W_{m}^{(t)}\right)\right)\notag \\
      &= \frac{W_{m}^{(t)}}{\left\|W_{m}^{(t)}\right\|_{\mathrm{F}}}\widehat{\mathcal{R}}_{m}^{(t+1)}\left(W_{m}^{(t)}\right)+\left\|W_{m}^{(t)}\right\|_{\mathrm{F}}\nabla_{W_{m}^{(t)}}\widehat{\mathcal{R}}_{m}^{(t+1)}\left(W_{m}^{(t)}\right)\notag \\
      &= \left(\lambda_{m}^{\mathrm{Phase}}\alpha_{m}^{(t)}+\frac{\left\langle W_{m}^{(t)},E_{m}^{(t+1)} \right\rangle}{\left\|W_{m}^{(t)}\right\|_{\mathrm{F}}}\right)W_{m}^{(t)}\notag \\
      &\phantom{\mathrel{=}}+ \left\|W_{m}^{(t)}\right\|_{\mathrm{F}}\left(\lambda_{m}^{\mathrm{Phase}}\cdot v_{2m-1}^{*}v_{2m}^{*\top}+E_{m}^{(t+1)}\right),
\end{align}
where $\mathrm{Phase}\in\{\mathrm{I},\mathrm{II},\mathrm{III}\}$ indexes the phase. In the subsequent proof, we retain the phase notation for $\alpha_{m}$ and $\lambda_{m}^{\mathrm{Phase}}$, as they vary across different phases. For $G_{m}^{(t)}$ and $E_{m}^{(t)}$, however, we omit this notation and simply write them as $G_{m}^{(t)}$ and $E_{m}^{(t)}$, since the ``Phase'' is implicitly determined by the iteration index $t$.

\subsection{Stochastic Gradient Updates and Alignment Dynamics}\label{sec-D.3}

This decomposition separates the signal-driven drift, which aligns $W_{m}^{(t)}$ with $v_{2m-1}^{*}v_{2m}^{*\top}$, from stochastic perturbations induced by the noise $\mathbf{E}$. Therefore, we obtain the iteration w.r.t $W_{m}$ as follows,
\begin{align}\label{eq-update-rule-Wm-detail}
	W_{m}^{(t+1)} &=W_{m}^{(t)}+\eta_{m} G_{m}^{(t)} \notag \\
    &= W_{m}^{(t)} + \eta_{m}\left\|W_{m}^{(t)}\right\|_{\mathrm{F}} \cdot \left(\lambda_{m}^{\mathrm{I}}\cdot v_{2m-1}^{*}v_{2m}^{*\top} + E_{m}^{(t+1)}\right)\notag \\
	&\phantom{\mathrel{=}}+ \eta_{m}\left(\lambda_{m}^{\mathrm{I}}\alpha_{m}^{(t)} + \frac{\left\langle W_{m}^{(t)},E_{m}^{(t+1)}\right\rangle}{\left\|W_{m}^{(t)}\right\|_{\mathrm{F}}}\right) \cdot W_{m}^{(t)}.
\end{align}
Substituting $u=v_{2m-1}^{*}, v=v_{2m}^{*}, W=W_{m}^{(t)}$, and $Q=G_{m}^{(t)}$ into the second-order Taylor expansion of $\alpha(u,v,W)$, and evaluating each term using the explicit decomposition of $G_{m}^{(t)}$, we obtain
\begin{align}\label{all-iteration-update-alpha}
      \alpha_{m}^{(t+1)}&= \alpha_{m}^{(t)} + \eta\left(\lambda_{m}\left(\left(\alpha_{m}^{(t)}\right)^{2} - \left(\alpha_{m}^{(t)}\right)^{2} + 1 - \left(\alpha_{m}^{(t)}\right)^{2}\right)\right)\notag \\
      &\phantom{\mathrel{=}} +\eta\left(\left\langle E_{m}^{(t+1)},v_{2m-1}^{*}v_{2m}^{*\top}\right\rangle - \alpha_{m}^{(t)}\frac{\left\langle E_{m}^{(t+1)},W_{m}^{(t)}\right\rangle}{\left\|W_{m}^{(t)}\right\|_{\mathrm{F}}}\right.\notag \\
      &\, \, \, \, \, \, \, \, \, \, \, \, \, \, \, \, \, \, \phantom{\mathrel{=}} +\left.\alpha_{m}^{(t)}\frac{\left\langle E_{m}^{(t+1)},W_{m}^{(t)}\right\rangle}{\left\|W_{m}^{(t)}\right\|_{\mathrm{F}}} - \alpha_{m}^{(t)}\frac{\left\langle E_{m}^{(t+1)},W_{m}^{(t)}\right\rangle}{\left\|W_{m}^{(t)}\right\|_{\mathrm{F}}}\right)\notag \\
      &\phantom{\mathrel{=}} +\frac{\eta^{2}}{2}\Psi_{1} \left(W_{m}^{(t)},G_{m}^{(t)},v_{2m-1}^{*},v_{2m}^{*},\overline{\eta}\right)\notag \\
      &= \alpha_{m}^{(t)} + \eta\left(\lambda_{m}\left(1-\left(\alpha_{m}^{(t)}\right)^{2}\right) + \left\langle E_{m}^{(t+1)},v_{2m-1}^{*}v_{2m}^{*\top}-\alpha_{m}^{(t)}\frac{W_{m}^{(t)}}{\left\|W_{m}^{(t)}\right\|_{\mathrm{F}}}\right\rangle\right)\notag \\
      &\phantom{\mathrel{=}} + \frac{\eta^{2}}{2}\Psi_{1} \left(W_{m}^{(t)},G_{m}^{(t)},v_{2m-1}^{*},v_{2m}^{*},\overline{\eta}\right),
\end{align}
where $\overline{\eta} \in [0, \eta]$ and the residual $\Psi_{1}: \mathbb{R}^{d_{1}\times d_{2}} \times \mathbb{R}^{d_{1}\times d_{2}} \times \mathbb{R}^{d_{1}}\times\mathbb{R}^{d_{2}}\times\mathbb{R} \to \mathbb{R}$ is defined as 
\begin{equation}\label{psi-1}
	\begin{aligned}
		\Psi_{1}(W, Q, v, u, \eta) &= - 2 \frac{(v^\top Q u) \cdot \left(\langle W, Q\rangle + \eta \|Q\|_{\mathrm{F}}^{2}\right)}{\|W + \eta Q\|_{\mathrm{F}}^{3}} - \frac{(v^\top W u + \eta v^\top Q u) \|Q\|_{\mathrm{F}}^{2}}{\|W + \eta Q\|_{\mathrm{F}}^{3}} \\
		&\qquad\qquad + 3 \frac{(v^\top W u + \eta v^\top Q u) \left(\langle W, Q\rangle + \eta \|Q\|_{\mathrm{F}}^{2}\right)^{2}}{\|W + \eta Q\|_{\mathrm{F}}^{5}}.
	\end{aligned}
\end{equation}

\subsection{Technical Lemmas}\label{tech-lemmma}
In the remaining part of this section, we will introduce several technical lemmas required for our proof.

To analyze the one-step evolution of the alignment for any block $m\in[k]$, we use the following technical lemma to further represent $\alpha\left(v,u,W+\eta Q\right)$ by $\alpha\left(v,u,W\right)$ and some negligible error.
\begin{lemma}\label{lemma-expansion}
Let $W$ and $Q$ be $d_{1}\times d_{2}$ matrices, $v$ and $u$ be $d_{1}$-dimensional and $d_{2}$-dimensional unit vectors, respectively, and $\eta>0$. We have
\begin{equation}
	\alpha\left(v,u,W+\eta Q\right)=\alpha(v,u,W)+\eta s(W,Q,v,u) + \frac{\eta^{2}}{2}\Psi_{1}(W,Q,v,u,\overline{\eta}),
\end{equation}
where the first-order term is given by
\begin{equation}
	s(W,Q,v,u)=\frac{v^{\top} Q u}{\left\|W\right\|_{\mathrm{F}}}-\frac{v^{\top} Wu \times\left\langle W, Q\right\rangle}{\left\|W\right\|_{\mathrm{F}}^{3}},
\end{equation}
$\overline{\eta}$ and $\Psi_{1}$ is defined in Section~\ref{sec-D.3}.
\end{lemma}
\begin{proof}
	Applying Lemma \ref{lemma-aux-frac} to function $\alpha\left(v,u,W+\eta Q\right)$ directly, we complete the proof.
\end{proof}

For any block $m\in[k]$ and $\mathrm{Phase}\in\{\mathrm{I}, \mathrm{II}, \mathrm{III}\}$,  we need a high-probability bound of the noise $E_{m}^{(t+1)}$ that is adopted throughout the proof. This is because, in our proof, we need to rule out certain low-probability events where random variables deviate beyond their expected bounds, commonly referred to as ``bad events''. 
% We use $T$ to denote the total sample size $\sum_{m=1}^{k}\sum_{l=1}^{3} T_{m}^{(l)}$.
We use a unified notation $T$ to denote the total sample size. In the analysis of the $l$-th phase, we may set $T=\sum_{m=1}^{k}T_{m}^{(l)}$.
\begin{lemma}\label{lemma:high-prob}
	Suppose that Assumption \ref{ass-base} holds. For any $\delta'>0$, the following event
	\begin{align*}
		\mathcal{A}_{m}^{(t+1)}(\delta') = &\bigcap_{\genfrac{}{}{0pt}{}{i\in [d_{2m-1}]}{ j\in[d_{2m}]}}\left\{\left| \left(E_{m}^{(t+1)}\right)_{i,j} \right| \leq \sqrt{\mathsf{c}_{1}} \right\} \\
		&\bigcap_{\smash{\phantom{\genfrac{}{}{0pt}{}{i\in [d_{2m-1}]}{ j\in[d_{2m}]}}}} \left\{\left|\left\langle E_{m}^{(t+1)},W_{m}^{(t)}\right\rangle\right|\leq\sqrt{\mathsf{c}_{1}}\left\|W_{m}^{(t)}\right\|_{\mathrm{F}}\right\} \\
		&\bigcap_{\smash{\phantom{\genfrac{}{}{0pt}{}{i\in [d_{2m-1}]}{ j\in[d_{2m}]}}}} \left\{\left|\left\langle E_{m}^{(t+1)},v_{2m-1}^{*}v_{2m}^{*\top}\right\rangle\right|\leq\sqrt{\mathsf{c}_{1}}\right\},
	\end{align*}
	satisfies $\mathbb{P}\left[\bigcap_{m=1}^{k}\bigcap_{t=0}^{T-1}\mathcal{A}_{m}^{(t+1)}(\delta')\right] \ge 1 - \delta'$, where 
	\begin{align}
		\mathsf{c}_{1}:=4\log\left(\frac{T\sum_{i=1}^{k}d_{2i-1}d_{2i}+2Tk}{\delta'}\right).
	\end{align}
\end{lemma}
Furthermore, given the matrix parameter $W_{m}^{(t)}$ at iteration $t$, we also need to ensure that, under the truncation of event $\mathcal{A}_{m}^{(t+1)}(\delta')$, the conditional expectation of the inner product
\begin{align*}
    \left\langle E_{m}^{(t+1)}\cdot\mathds{1}_{\mathcal{A}_{m}^{(t+1)}(\delta')},v_{2m-1}^{*}v_{2m}^{*\top}\right\rangle\quad \text{and} \quad\left\langle E_{m}^{(t+1)}\cdot\mathds{1}_{\mathcal{A}_{m}^{(t+1)}(\delta')},W_{m}^{(t)}\right\rangle,
\end{align*}
are vanishingly small. This condition is necessary for analyzing the dynamics of $\left\{\alpha_{m}\right\}_{m=1}^{k}$ with high probability using submartingales (or supermartingales).
\begin{lemma}\label{control-expectation}
	Suppose that Assumption \ref{ass-base} holds. Given constant $\tau\in\mathbb{R}_{+}$, letting $\delta' \lesssim {\tau^{2}}/\mathsf{c}_1$ and $T \geq 256(1+K)$ where $K$ denotes a fixed constant which satisfies
	\begin{align*}
        K\geq4\log\left(512\left(1+K\right)\tau^{-2}\right).
	\end{align*}
	Then we have
	\begin{gather}
		\left|\mathbb{E}_t\left[\left\langle E_{m}^{(t+1)}\cdot\mathds{1}_{\mathcal{A}_m^{(t+1)}(\delta')},v_{2m-1}^{*}v_{2m}^{*\top}\right\rangle\right]\right|\leq\tau,\label{tech-lemma-2-1} \\
		\left|\mathbb{E}_t\left[\left\langle E_{m}^{(t+1)}\cdot\mathds{1}_{\mathcal{A}_m^{(t+1)}(\delta')},W_{m}^{(t+1)}\right\rangle\right]\right|\leq\tau\left\|W_{m}^{(t+1)}\right\|_{\mathrm{F}},\label{tech-lemma-2-2}
	\end{gather}
	for any $t\in[0:T-1]$.
\end{lemma}

In the following analysis, we denote $E_m^{(t+1)}\mathds{1}_{\calA_m^{(t+1)}(\delta')}$ simply by $E_m^{(t+1)}$. We use the following lemma to estimate the second-order remainder of the Taylor expansion of $\alpha(v,u,W+\eta Q)$ with respect to $\eta$. Recall the definition of the correlation factor between $v_{2m-1}^{*}$ and $v_{2m}^{*}$, we have:
\begin{lemma}\label{estimation}
	Suppose that $d_{2m-1}, d_{2m}>k$ and Assumption \ref{ass-base} holds. For any $t\in \left[{T_{m}}\right]$, if we further assume that 
	\begin{align*}
		\alpha_{m}^{(t)} > 0 \qquad \text{and} \qquad \eta_{m}\leq\frac{\alpha_{m}^{(t)}}{32\left(\lambda_{m}^{\mathrm{Phase}}+\sqrt{\mathsf{c}_{1}}\right)},
	\end{align*} 
	then the following holds
	\begin{align}
		\left|\Psi_{1}\left(W_{m}^{(t)},G_{m}^{(t)},v_{2m-1}^{*},v_{2m}^{*},\overline{\eta}_{m}\right)\right|\leq&\frac{64\mathsf{c}_{1}\left[\left(\lambda_{m}^{\mathrm{Phase}}+2\right)^{2}+\alpha_{m}^{(t)}(d_{2m-1}d_{2m}+4)\right]}{\left[1-4\overline{\eta}_{m}\left(\lambda_{m}^{\mathrm{Phase}}\alpha_{m}^{(t)}+\sqrt{\mathsf{c}_{1}}\right)\right]^{3}}\label{esti-Psi1}
	\end{align}
    with probability at least $1-\delta'/k$, where $T_m=\sum_{l=1}^{3}T_m^{(l)}$.
	Moreover, we can obtain
	\begin{align}
		\left|\left\langle E_{m}^{(t+1)}, v_{2m-1}^{*} v_{2m}^{*\top}-\frac{\alpha_{m}^{(t)}}{\left\|W_{m}^{(t)}\right\|_{\mathrm{F}}}W_{m}^{(t)}\right\rangle\right| \leq \sqrt{\mathsf{c}_{1}}\left(1+\alpha_{m}^{(t)}\right),\label{uni-bound-1-order-random-term}
	\end{align}
	for any $t\in[T_m]$ with probability at least $1-\delta'/k$.
\end{lemma}

\subsection{Gaussian Assumption for Faster Convergence}
Although Assumption~\ref{ass-base} suffices to solve Problem~\ref{PCA-def}, we can derive sharper convergence rates under the standard Gaussian setting.
\begin{assumption}[Gaussian Noise for Faster Convergence]\label{assumption-Gaussian}
	The entries of the noise tensor $\mathbf{E}$ are i.i.d. random variables distributed as $\mathcal{N}(0,1)$.
\end{assumption}
Assumption~\ref{assumption-Gaussian} is the standard condition adopted in the tensor PCA literature \citep{montanari2014statistical, hopkins2015tensor, hopkins2016fast, feldman2023sharp}. It is worth noting that the faster convergence rate established in Theorem~\ref{thm-phase-III-tensor-PCA} does not strictly require the full Gaussian assumption. 
% It suffices to impose the following weaker second-moment symmetric condition, which is a direct consequence of Assumption \ref{assumption-Gaussian}:
% \begin{equation*}
%     \mathbb{E}\left[\left\langle u, \mathrm{Vec}(\mathbf{E})\right\rangle^{2}\right] = 1, \qquad\forall u\in\mathbb{S}^{\prod_{n=1}^{\overline{k}}d_{n}-1}. 
% \end{equation*}
If we adopt the assumption used by \cite{tang2025revisit}, namely that the random noise tensor $\mathbf{E}$ has i.i.d. sub-Gaussian entries with sub-Gaussian parameter $\mathcal{O}(1)$ and variance $1$, the accelerated convergence rate in Theorem~\ref{thm-phase-III-tensor-PCA} can still be obtained.

\section{Proof of Phase I}\label{phase-I}
\subsection{Proof Outline}
In this section, we analyze \textbf{Phase~I} of Algorithm~\ref{MSANSGA}, during which the iterates start from a weakly correlated initialization and progressively align with the ground-truth component. More precisely, for any $m\in[k]$, we show that the alignment parameter $\alpha_{m}$ increases from $2\mathrm{LB}_{m}^{\mathrm{I}}$ to $\frac{1}{2}\mathrm{UB}_{m}^{\mathrm{I}}$ within $T_{m}^{(1)}=\sum_{h=1}^{h^*}T_{m,h}^{(1)}$ iterations with high probability where
\begin{align*}
	\mathrm{LB}_{m}^{\mathrm{I}}\coloneqq\frac{\gamma_{m}}{4},\quad \mathrm{UB}_{m}^{\mathrm{I}}\coloneqq2\gamma_{m}^{1-\frac{1}{k-1}},
\end{align*}
and $h^*\coloneqq\lceil\log(4)+\frac{1}{k-1}\max_{m}\{\log(\gamma_{m}^{-1})\}\rceil$ thereby establishing a weak alignment. We now define the effective block SNR for each stage $h\in[h^*]$ in \textbf{Phase~I}:
\begin{equation}\label{block-eff-SNR-phase-I}
	\lambda_{m,h}^{\mathrm{I}}\coloneqq \lambda\prod_{i=1}^{m-1}\left|\left\langle\overline{W}_{i,h-1}^{\left(T_{i,h-1}^{(1)}\right)},v_{2i-1}^*v_{2i}^{*\top}\right\rangle\right|\cdot\prod_{j=m+1}^{k}\left|\left\langle\overline{W}_{j,h}^{\left(T_{j,h}^{(1)}\right)},v_{2j-1}^*v_{2j}^{*\top}\right\rangle\right|.
\end{equation}
And we set $\lambda_{m}^{\mathrm{I}}\coloneqq\lambda_{m,0}^{\mathrm{I}}$ for any $m\in[k]$. Then we state the main result for \textbf{Phase~I} as follows.
\begin{theorem}\label{thm-phase-I-tensor-PCA}
	Assume $d\geq\Omega(k)$ and $\lambda\leq\mathcal{O}\left(\prod_{i=1}^{2k}d_{i}^{1/4}\right)$.  Under Assumptions~\ref{ass-pro} and \ref{ass-base}, consider the dynamic generated via \textbf{Phase~I} of Algorithm~\ref{MSANSGA}. For any $0<\delta''<1/2$, $m\in[k]$ and $h\in[h^*]$, if we pick
	\begin{gather*}
		\eta_{m,h}^{(1)}\leq\frac{1}{8192}\cdot\min\left\{\frac{\lambda_{m,h}^{\mathrm{I}}}{\mathsf{c}_1\mathrm{UB}_{m,h}^{\mathrm{I}}d_{2m-1}d_{2m}}, \frac{\mathrm{LB}_{m,h}^{\mathrm{I}}}{\lambda_{m,h}^{\mathrm{I}}+\sqrt{\mathsf{c}_1}}\right\}, \\
		\frac{\left(\mathrm{LB}_{m,h}^{\mathrm{I}}\right)^2}{64\mathsf{c}_1\left(\eta_{m,h}^{(1)}\right)^{2}\left(\log\left(kh^*T_{m,h}^{(1)}\right)-\log\left(\delta''\right)\right)}\geq T_{m,h}^{(1)}\geq\frac{32\mathrm{UB}_{m,h}^{\mathrm{I}}}{\eta_{m}^{(1)}\lambda_{m,h}^{\mathrm{I}}},
	\end{gather*}
	where
    \begin{align*}
		\mathrm{LB}_{m,h}^{\mathrm{I}}\coloneqq\frac{\zeta_{m}^{h-1}\gamma_{m}}{4}, \quad
		\mathrm{UB}_{m,h}^{\mathrm{I}}\coloneqq\frac{\zeta_{m}^{h}\gamma_{m}}{2}, \quad \zeta_{m}=\exp\left\{\frac{\log(4)+\frac{1}{k-1}\max_{m}\{\log(\gamma_{m}^{-1})\}}{h^*}\right\},
	\end{align*}
	and $h^*=\lceil\log(4)+\frac{1}{k-1}\max_{m}\{\log(\gamma_{m}^{-1})\}\rceil$. Then $\alpha_{m}^{\left(T_{m}^{(1)}\right)} \geq \frac{1}{2}\mathrm{UB}_{m}^{\mathrm{I}}$ holds for every $m\in[k]$ with probability at least $1-\delta''/2$. 
\end{theorem}

It is worth noting that the following inequality holds for all $m \in [k]$:
\begin{equation*}
    \lambda_{m,h}^{\mathrm{I}} \geq \lambda \prod_{i\in[k]}^{i\neq m}\zeta_{i}^{h-1}\gamma_{i},
\end{equation*}
which quantifies the stage-wise growth of the effective SNR and will be used to control the length of each stage $T_{m,h}^{(1)}$. To analyze the stage-wise ascent, we decompose the failure of reaching the target level within stage $h$ into three ``bad'' events. For each $h \in [h^*]$ and fixed $m \in [k]$, define
\begin{equation*}
	\begin{gathered}
		\mathcal{E}_{1,h}=\left\{\exists t\in \left[T_{m,h}^{(1)}\right], \alpha_{m,h}^{(t)} < \mathrm{LB}_{m,h}^{\mathrm{I}} \right\}, \quad \mathcal{E}_{2,h}=\left\{\max_{t\in \left[T_{m,h}^{(1)}\right]} \alpha_{m,h}^{(t)} < \mathrm{UB}_{m,h}^{\mathrm{I}}\right\}, \\ 
		\mathcal{E}_{3,h}=\left\{\alpha_{m,h}^{\left({T_{m,h}^{(1)}}\right)} < \frac{1}{2}\mathrm{UB}_{m,h}^{\mathrm{I}} \right\}.
	\end{gathered}
\end{equation*}
These events capture different failure modes of the iterate within stage $h$ and are handled sequentially in the analysis. Specifically, $\mathcal{E}_{1,h}$ corresponds to the event that the iterate falls below the lower threshold $\mathrm{LB}_{m,h}^{\mathrm{I}}$, which is ruled out via a stopping-time argument and a concentration bound on the coupled process. Conditional on $\mathcal{E}_{1,h}^{c}$, the event $\mathcal{E}_{2,h}$ characterizes the failure to reach the target level $\mathrm{UB}_{m,h}^{\mathrm{I}}$, and is controlled by establishing a submartingale-type growth of the process. Finally, conditional on $\mathcal{E}_{1,h}^{c}\cap \mathcal{E}_{2,h}^{c}$, the event $\mathcal{E}_{3,h}$ captures the possibility that, after reaching the upper threshold, the iterate subsequently drops below $\frac{1}{2}\mathrm{UB}_{m,h}^{\mathrm{I}}$, which is excluded by combining the ascent property with a concentration argument over the post-hitting trajectory.

Accordingly, in the following three lemmas, we control these events in sequence by showing that for each $l\in\{1,2,3\}$,
\begin{align*}
    \mathbb{P}\left[\left(\bigcap_{1 \leq l'<l}\mathcal{E}_{l',h}^{c}\right)\bigcap \mathcal{E}_{l,h} \right] \leq \frac{\delta''}{6h^*}.
\end{align*}
In this phase, we set $\delta'\asymp\delta''/\left(kd_{2k}\max_m\left\{\sum_{h=1}^{h^*}T_{m,h}^{(1)}\right\}\right)^8$ wherever the technical lemmas from Section \ref{tech-lemmma} is applied.
\begin{lemma}\label{APCA-phase-I-lower-bound}
	Assume $d_{2m-1}, d_{2m}\geq\Omega(k)$ and $\lambda\leq\mathcal{O}\left(\prod_{i=1}^{\overline{k}}d_{i}^{1/4}\right)$, and suppose $\eta_{m,h}^{(1)}\leq f_{1,h}^{*}(m)$, and $T_{m,h}^{(1)}\left(\eta_{m,h}^{(1)}\right)^{2}\leq f_{2,h}^{*}(m)$, where
	\begin{equation}\label{para-lemma-upper-bound}
		\begin{gathered}
			f_{1,h}^{*}(m) \coloneqq \frac{1}{4096} \cdot \min\left\{\frac{\lambda_{m,h}^{\mathrm{I}}}{\mathsf{c}_{1}\left[\left(\lambda_{m,h}^{\mathrm{I}}\right)^{2}+\mathrm{LB}_{m,h}^{\mathrm{I}}d_{2m-1}d_{2m}\right]}, \frac{\mathrm{LB}_{m,h}^{\mathrm{I}}}{\lambda_{m,h}^{\mathrm{I}}+\sqrt{\mathsf{c}_{1}}}\right\}, \\
            f_{2,h}^{*}(m) \coloneqq \frac{\left(\mathrm{LB}_{m,h}^{\mathrm{I}}\right)^2}{64\mathsf{c}_{1} \log\left(3h^*\left(T_{m,h}^{(1)}\right)^{2}/\delta''\right)}.
		\end{gathered}
	\end{equation}
	Under the Assumption of Theorem~\ref{thm-phase-I-tensor-PCA}, the event $\mathcal{E}_{1,h}$ holds with probability at most $\frac{\delta''}{6h^*}$.
\end{lemma}
\begin{lemma}\label{APCA-phase-I-convergence-lower-bound}
	Assume $d_{2m-1}, d_{2m}\geq\Omega(k)$ and $\lambda\leq\mathcal{O}\left(\prod_{i=1}^{\overline{k}}d_{i}^{1/4}\right)$, and suppose $\eta_{m,h}^{(1)}\leq \min\{f_{1,h}^{*}(m),$ $ f_{3,h}^{*}(m)\}$ ($f_{1,h}^{*}(m)$ is defined in Eq.~\eqref{para-lemma-upper-bound}), and ${T_{m,h}^{(1)}}\geq f_{4,h}^{*}(m)$, and ${T_{m,h}^{(1)}}\eta_{m,h}^{(1)}\geq f_{5,h}^{*}(m)$, where
	\begin{gather*}
		f_{3,h}^{*}(m)\coloneqq\frac{\lambda_{m,h}^{\mathrm{I}}}{4096\mathsf{c}_{1}\left[\left(\lambda_{m,h}^{\mathrm{I}}\right)^{2}+\mathrm{UB}_{m,h}^{\mathrm{I}}d_{2m-1}d_{2m}\right]}, \\
		f_{4,h}^{*}(m)\coloneqq\frac{2048\mathsf{c}_{1}\log(6h^*/\delta'')}{\left(\lambda_{m,h}^{\mathrm{I}}\right)^{2}},\quad f_{5,h}^{*}(m)\coloneqq\frac{32\mathrm{UB}_{m,h}^{\mathrm{I}}}{\lambda_{m,h}^{\mathrm{I}}}.
	\end{gather*}
	Under the setting of Theorem~\ref{thm-phase-I-tensor-PCA}, the combined event $\mathcal{E}_{1,h}^{c}\bigcap\mathcal{E}_{2,h}$ holds with probability at most $\frac{\delta''}{6h^*}$.
\end{lemma}
\begin{lemma}\label{APCA-phase-I-convergence-lower-bound-last-iterate}
	Assume $d_{2m-1}, d_{2m}\geq\Omega(k)$ and $\lambda\leq\mathcal{O}\left(\prod_{i=1}^{\overline{k}}d_{i}^{1/4}\right)$, and suppose $\eta_{m,h}^{(1)}\leq\min\{ f_{1,h}^{*}(m),\allowbreak f_{3,h}^{*}(m),f_{6,h}^{*}(m)\}$,
	where
	\begin{align*}
	   f_{6,h}^{*}(m)=\frac{\lambda_{m,h}^{\mathrm{I}}\mathrm{UB}_{m,h}^{\mathrm{I}}}{256\mathsf{c}_{1}\log\left(3h^*\left({T_{m,h}^{(1)}}\right)^{2}/\delta''\right)},
	\end{align*}
	with $\mathrm{UB}_{m,h}^{\mathrm{I}}$ defined in Lemma~\ref{APCA-phase-I-convergence-lower-bound}. Under the setting of Theorem~\ref{thm-phase-I-tensor-PCA}, the combined event $\mathcal{E}_{1,h}^{c}\bigcap\mathcal{E}_{2,h}^{c} \allowbreak \bigcap\mathcal{E}_{3,h}$ holds with probability at most $\frac{\delta''}{6h^*}$.
\end{lemma}
Decomposing the event $\mathcal{E}_{3,h}$ according to the first violated condition and applying the union bound yields
\begin{align*}
	\mathbb{P}\left(\mathcal{E}_{3,h}\right) &\leq \mathbb{P}\left(\mathcal{E}_{1,h}\right) + \mathbb{P}\left(\mathcal{E}_{1,h}^{c} \bigcap \mathcal{E}_{2,h} \right) + \mathbb{P}\left(\mathcal{E}_{1,h}^{c} \bigcap \mathcal{E}_{2,h}^{c} \bigcap \mathcal{E}_{3,h} \right)
	\leq 3\times \frac{\delta''}{6h^*} \leq \frac{\delta''}{2h^*},
\end{align*} 
Consequently, with probability at least $1-\delta''/(2h^*)$, $\alpha_{m,h}$ reaches the level $\mathrm{UB}_{m,h}^{\mathrm{I}}$ by time $T_{m,h}^{(1)}$.

\subsection{Proof Details}
\begin{proof}[Proof of Lemma~\ref{APCA-phase-I-lower-bound}]
	We commence the proof by defining a constrained coupling process.
	\begin{definition}\label{def-hatW}
		Let $\left\{W_{m,h}^{(t)}\right\}_{t=0}^{{T_{m,h}^{(1)}}}$ be a Markov chain in $\mathbb{R}^{d_{2m-1}\times d_{2m}}$ adapted to filtration $\left\{\mathcal{F}_{m,h}^{(t)}\right\}_{t=0}^{{T_{m,h}^{(1)}}}$. Define the following event for a scalar
		\begin{equation}
			\mathcal{E}(\alpha)\coloneqq\left\{\alpha\geq\mathrm{LB}_{m,h}^{\mathrm{I}}\right\}.\notag
		\end{equation}
		The constrained coupling process $\left\{\widehat{W}_{m,h}^{(t)}\right\}_{t=0}^{{T_{m,h}^{(1)}}}$ with initialization $\widehat{W}_{m,h}^{(0)}=W_{m,h}^{(0)}$ evolves as
		\begin{enumerate}
			\item \emph{Updating stage: } If $\widehat{W}_{m,h}^{(t)}$ satisfies $\mathcal{E}\left(\widehat{\alpha}_{m,h}^{(t)}\coloneqq\frac{\left\langle \widehat{W}_{m,h}^{(t)}, v_{2m-1}^{*}v_{2m}^{*\top}\right\rangle}{\left\|\widehat{W}_{m,h}^{(t)}\right\|_{\mathrm{F}}}\right)$,
			let $\widehat{W}_{m,h}^{(t+1)}=W_{m,h}^{(t+1)}$.
			\item \emph{Absorbing state: } Otherwise, maintain $\widehat{W}_{m,h}^{(t+1)}=\widehat{W}_{m,h}^{(t)}$.
		\end{enumerate}
	\end{definition}
    We construct the constrained coupling process $\left\{\widehat W_{m,h}^{(t)}\right\}_{t=0}^{T_{m,h}^{(1)}}$ as in Definition~\ref{def-hatW}.
    Let
    \begin{align*}
        \overline{\tau} \coloneqq \inf_{t \geq 0}\left\{\widehat\alpha_{m,h}^{(t)}<\mathrm{LB}_{m,h}^{\mathrm{I}}\right\}
    \end{align*}
    be the stopping time. By construction, for all $t<\overline{\tau}$ we have $\widehat W_{m,h}^{(t)}=W_{m,h}^{(t)}$ and $\widehat\alpha_{m,h}^{(t)}\geq \mathrm{LB}_{m,h}^{\mathrm{I}}$. Moreover, according to the one-step dynamics in Eq.~\eqref{all-iteration-update-alpha} together with the boundedness control from Lemma~\ref{estimation}, the increment of $\widehat{\alpha}_{m,h}^{(t)}$ is uniformly bounded. Consequently, the trajectory cannot jump across the interval $\left[\mathrm{LB}_{m,h}^{\mathrm{I}},2\mathrm{LB}_{m,h}^{\mathrm{I}}\right]$ in a single iteration. Therefore, on the event $\mathcal{E}_{1,h}$, there must exist some $t_{1}<t_{2}$ such that the event
    \begin{align*}
        \mathcal{E}_{t_{1}}^{\overline{\tau}=t_{2}}\coloneqq\left\{\widehat{\alpha}_{m,h}^{(t_{1})}\geq\frac{3}{2}\mathrm{LB}_{m,h}^{\mathrm{I}}\bigcap\widehat{\alpha}_{m,h}^{(t_{1}:t_{2}-1)}\in\left[\mathrm{LB}_{m,h}^{\mathrm{I}},2\mathrm{LB}_{m,h}^{\mathrm{I}}\right]\bigcap\widehat{\alpha}_{m,h}^{(t_{2})}\leq \mathrm{LB}_{m,h}^{\mathrm{I}}\right\}
    \end{align*}
    occurs.
    For any $t\in[t_{1}:t_{2}-1]$, applying Eq.~\eqref{all-iteration-update-alpha} yields
    \begin{align}
		\widehat{\alpha}_{m,h}^{(t+1)} &\overset{\text{(a)}}{\geq} \widehat{\alpha}_{m,h}^{(t)} + \eta_{m,h}^{(1)}\lambda_{m,h}^{\mathrm{I}}\left(1-\left(\widehat{\alpha}_{m,h}^{(t)}\right)^{2}\right)+\eta_{m,h}^{(1)}\cdot\widehat{\xi}_{m,h}^{(t+1)}\notag \\
		&\phantom{\mathrel{=}}-\eta_{m,h}^{(1)}\left|\mathbb{E}_t\left[\left\langle E_{m,h}^{(t+1)},v_{2m-1}^{*}v_{2m}^{*\top}\right\rangle\right]\right|-\eta_{m,h}^{(1)}\frac{\widehat{\alpha}_{m,h}^{(t)}}{\left\|\widehat{W}_{m,h}^{(t)}\right\|_{\mathrm{F}}}\cdot\left|\mathbb{E}_t\left[\left\langle E_{m,h}^{(t+1)},\widehat{W}_{m,h}^{(t)}\right\rangle\right]\right|\notag \\
		&\phantom{\mathrel{=}}+\frac{\left(\eta_{m,h}^{(1)}\right)^{2}}{2}\Psi_{1} \left(\widehat{W}_{m,h}^{(t)},\widehat{G}_{m,h}^{(t)},v_{2m-1}^{*},v_{2m}^{*},\overline{\eta}_{m,h}\right)\notag \\
		&\overset{\text{(b)}}{\geq} \widehat{\alpha}_{m,h}^{(t)} + \frac{\eta_{m}^{(1)}\lambda_{m}^{\mathrm{I}}}{4} + \eta_{m,h}^{(1)}\cdot\widehat{\xi}_{m,h}^{(t+1)},\label{dynamic-hatalpha-lower-bound-proof}
    \end{align}
    where $\widehat{G}_{m,h}^{(t)}$ denotes the stochastic gradient of the risk function $\widehat{\mathcal{R}}_{m,\mathrm{I}}^{(t+1)}$ evaluated at the parameter matrix $\widehat{W}_{m,h}^{(t)}$ in stage $h$, and $\widehat{\xi}_{m,h}^{(t+1)}$ is a zero-mean random term which has the following form:
	\begin{align}\label{def-scalar-random-xi}
		\widehat{\xi}_{m,h}^{(t+1)}&\coloneqq\left\langle E_{m,h}^{(t+1)}, v_{2m-1}^{*}v_{2m}^{*\top}-\frac{\alpha_{m,h}^{(t)}}{\left\|W_{m,h}^{(t)}\right\|_{\mathrm{F}}}W_{m,h}^{(t)}\right\rangle\notag
		\\
		&\phantom{\mathrel{=}}-\mathbb{E}_t\left[\left\langle E_{m,h}^{(t+1)}, v_{2m-1}^{*}v_{2m}^{*\top}-\frac{\alpha_{m,h}^{(t)}}{\left\|W_{m,h}^{(t)}\right\|_{\mathrm{F}}}W_{m,h}^{(t)}\right\rangle\right],
	\end{align}
	(a) follows from Eq.~\eqref{all-iteration-update-alpha}, and (b) is derived from combining the construction of ${E}_{m,h}^{(t+1)}$, which satisfies
	\begin{align*}
		\frac{\lambda_{m,h}^{\mathrm{I}}\eta_{m,h}^{(1)}}{8} \geq \eta_{m,h}^{(1)}\left|\mathbb{E}_t\left[\left\langle E_{m,h}^{(t+1)},v_{2m-1}^{*}v_{2m}^{*\top}\right\rangle\right]\right|+\eta_{m,h}^{(1)}\frac{\widehat{\alpha}_{m,h}^{(t)}}{\left\|\widehat{W}_{m,h}^{(t)}\right\|_{\mathrm{F}}}\cdot\left|\mathbb{E}_t\left[\left\langle E_{m,h}^{(t+1)},\widehat{W}_{m,h}^{(t)}\right\rangle\right]\right|,
	\end{align*}
	and the result of Lemma~\ref{estimation} with the setting of $\eta_{m,h}^{(1)}$ which implicates that
	\begin{align}\label{eta-I-constraint}
		\frac{\lambda_{m,h}^{\mathrm{I}}\eta_{m,h}^{(1)}}{8} &\geq \left(\eta_{m,h}^{(1)}\right)^{2}\cdot\frac{32\mathsf{c}_{1}\left[\left(\lambda_{m,h}^{\mathrm{I}}+2\right)^{2}+2\mathrm{LB}_{m,h}^{\mathrm{I}}(d_{2m-1}d_{2m}+4)\right]}{\left[1-4\overline{\eta}_{m,h}\left(\lambda_{m,h}^{\mathrm{I}}+\sqrt{\mathsf{c}_{1}}\right)\right]^{3}} \notag \\
		&\geq \frac{\left(\eta_{m,h}^{(1)}\right)^{2}}{2}\left|\Psi_{1} \left(\widehat{W}_{m,h}^{(t)},\widehat{G}_{m,h}^{(t)},v_{2m-1}^{*},v_{2m}^{*},\overline{\eta}_{m,h}\right)\right|.
	\end{align}
	Since $\widehat{\xi}_{m,h}^{(t+1)}$
	is bounded, we demonstrate that ${E}_{m,h}^{(t+1)}$ satisfies the sub-Gaussian property for all $t\in[t_{1}:t_{2}]$. Thus we have
	\begin{equation}\label{hatalpha-supermartingale}
		\begin{split}
			\mathbb{E}_t\left[\exp\left\{\gamma\left(\widehat{\alpha}_{m,h}^{(t)}+\frac{\eta_{m,h}^{(1)}\lambda_{m,h}^{\mathrm{I}}}{4}-\widehat{\alpha}_{m,h}^{(t+1)}\right)\right\}\right]\leq \exp\left\{16\gamma^{2}\left(\eta_{m,h}^{(1)}\right)^{2}\mathsf{c}_{1}\right\},
		\end{split}
	\end{equation}
	for any $\gamma\in\mathbb{R}_{+}$. Applying Eq.~\eqref{dynamic-hatalpha-lower-bound-proof} and Eq.~\eqref{hatalpha-supermartingale} to Lemma~\ref{aux-martingale-concentration-subtraction-coro}, we can establish the  probability bound for event $\mathcal{E}_{t_{1}}^{\overline{\tau}=t_{2}}$ for any time pair $t_{1}<t_{2}\in\left[T_{m,h}^{(1)}\right]$ as
	\begin{align}\label{prob-calEc-component}
		\mathbb{P}\left(\mathcal{E}_{t_{1}}^{\overline{\tau}=t_{2}}\right) \leq \exp\left\{-\frac{\left(\mathrm{LB}_{m,h}^{\mathrm{I}}\right)^2}{64\mathsf{c}_{1}{T_{m,h}^{(1)}}\left(\eta_{m,h}^{(1)}\right)^{2}}\right\}.
	\end{align}
    Using Eq.~\eqref{prob-calEc-component} and the setting of hyper-parameters in Lemma~\ref{APCA-phase-I-lower-bound}, we obtain the following probability bound for event $\mathcal{E}_{1}$:
	\begin{align}
		\mathbb{P}\left(\mathcal{E}_{1,h}\right)\leq\sum_{1\leq t_{1}<t_{2}\leq {T_{m,h}^{(1)}}}\mathbb{P}\left(\mathcal{E}_{t_{1}}^{\overline{\tau}=t_{2}}\right)\leq\frac{\left({T_{m,h}^{(1)}}\right)^{2}}{2}\exp\left\{-\frac{\left(\mathrm{LB}_{m,h}^{\mathrm{I}}\right)^2}{64\mathsf{c}_{1}{T_{m,h}^{(1)}}\left(\eta_{m,h}^{(1)}\right)^{2}}\right\}\notag
		\leq\frac{\delta''}{6h^*}.\notag
	\end{align}
\end{proof}
\begin{proof}[Proof of Lemma~\ref{APCA-phase-I-convergence-lower-bound}]
	We begin the proof by introducing a coupling process.
	\begin{definition}\label{def-breveW-1}
		Let $\left\{W_{m,h}^{(t)}\right\}_{t=0}^{{T_{m,h}^{(1)}}}$ be a Markov chain in $\mathbb{R}^{d_{2m-1}\times d_{2m}}$ adapted to filtration $\left\{\mathcal{F}_{m,h}^{(t)}\right\}_{t=0}^{{T_{m,h}^{(1)}}}$. The coupling process $\left\{\breve{\alpha}_{m,h}^{(t)}\right\}_{t=0}^{{T_{m,h}^{(1)}}}$ with initialization $\breve{\alpha}_{m,h}^{(0)}=\alpha_{m,h}^{(0)}$ evolves as
		\begin{enumerate}
			\item Updating state: If event 
			\begin{align*}
                \breve{\mathcal{E}}\left(\breve{\alpha}_{m,h}^{(t)}\right)\coloneqq\left\{\mathcal{E}\left(\breve{\alpha}_{m,h}^{(t)}\right)\bigcap\breve{\alpha}_{m,h}^{(t)}<\mathrm{UB}_{m,h}^{\mathrm{I}}\right\},
	        \end{align*}
			holds, let $\breve{\alpha}_{m,h}^{(t+1)}=\alpha_{m,h}^{(t+1)}$,
			\item Growing state: Otherwise, let $\breve{\alpha}_{m,h}^{(t+1)}=\breve{\alpha}_{m,h}^{(t)}+\frac{\eta_{m,h}^{(1)}\lambda_{m,h}^{\mathrm{I}}}{4}$.
		\end{enumerate}
	\end{definition}
	We aim to demonstrate that $\breve{\alpha}_{m,h}^{(t)}-\frac{t\eta_{m,h}^{(1)}\lambda_{m,h}^{\mathrm{I}}}{4}$ is a submartingale. 
	If event $\breve{\mathcal{E}}\left(\breve{\alpha}_{m,h}^{(t)}\right)^{c}$ holds, we directly obtain $\breve{\alpha}_{m,h}^{(t+1)}\geq \frac{\eta_{m,h}^{(1)}\lambda_{m,h}^{\mathrm{I}}}{4}+\breve{\alpha}_{m,h}^{(t)}$. Otherwise, we have
	\begin{align}\label{ascent-dynamic-alpha}
		\breve{\alpha}_{m,h}^{(t+1)}=\alpha_{m,h}^{(t+1)}\overset{\text{(a)}}{\geq}&\alpha_{m,h}^{(t)} + \frac{\eta_{m,h}^{(1)}\lambda_{m,h}^{\mathrm{I}}}{4} + \eta_{m,h}^{(1)}\breve{\xi}_{m,h}^{(t+1)} \notag \\
		\overset{\text{(b)}}{\geq}&\breve{\alpha}_{m,h}^{(t)} + \frac{\eta_{m,h}^{(1)}\lambda_{m,h}^{\mathrm{I}}}{4} + \eta_{m,h}^{(1)}\breve{\xi}_{m,h}^{(t+1)},
	\end{align}
	where $\breve{\xi}_{m,h}^{(t+1)}$ is a zero-mean random variable defined analogously to Eq.~\eqref{def-scalar-random-xi}, with $\widehat{W}_{m,h}^{(t)}$ and $\widehat{\alpha}_{m,h}^{(t)}$
	substituted for $W_{m,h}^{(t)}$ and $\alpha_{m,h}^{(t)}$, respectively. Here, (a) is derived from Eq.~\eqref{dynamic-hatalpha-lower-bound-proof}, and (b) relies on the temporal exclusivity property that if event $\breve{\mathcal{E}}^{c}\left(\breve{\alpha}_{m,h}^{(t)}\right)$ occurs at time $t$, then $\breve{\mathcal{E}}\left(\breve{\alpha}_{m,h}^{(t')}\right)$ is permanently excluded for all subsequent times $t' > t$. Therefore, based on the supermartingale, we obtain
	\begin{align}\label{lower-case-I}
		\mathbb{P} \left(\mathcal{E}_{1,h}^{c}\bigcap\mathcal{E}_{2,h}\right) &\leq \mathbb{P}\left(\breve{\alpha}_{m,h}^{\left({T_{m,h}^{(1)}}\right)}<\mathrm{UB}_{m,h}^{\mathrm{I}}\bigcap\mathcal{E}\left(\breve{\alpha}_{m,h}^{\left({T_{m,h}^{(1)}}\right)}\right)\right) \notag \\
		&\overset{\text{(c)}}{\leq}\exp\left\{-\frac{\left(\frac{{T_{m,h}^{(1)}}\eta_{m,h}^{(1)}\lambda_{m,h}^{\mathrm{I}}}{4}+\mathrm{LB}_{m,h}^{\mathrm{I}}-\mathrm{UB}_{m,h}^{\mathrm{I}}\right)^{2}}{64\mathsf{c}_{1}{T_{m,h}^{(1)}}\left(\eta_{m,h}^{(1)}\right)^{2}}\right\}\notag \\
		&\overset{\text{(d)}}{\leq} \exp\left\{-\frac{{T_{m,h}^{(1)}}\left(\lambda_{m,h}^{\mathrm{I}}\right)^{2}}{2048\mathsf{c}_{1}}\right\}\overset{\text{(e)}}{\leq}\frac{\delta''}{6h^*},
	\end{align}
	where (c) is derived from applying the estimation of $\breve{\xi}_{m,h}^{(t+1)}$ below:
	\begin{align}
		\left|\breve{\xi}_{m,h}^{(t+1)}\right|\leq4\sqrt{\mathsf{c}_{1}},\notag
	\end{align}
	which implicates that $\breve{\xi}_{m,h}^{(t+1)}$ is sub-Gaussian with parameter $4\sqrt{\mathsf{c}_{1}}$, to Lemma~\ref{aux-martingale-concentration-subtraction-coro}. Moreover, since $T_{m,h}^{(1)}\eta_{m,h}^{(1)}$ $\lambda_{m,h}^{\mathrm{I}} \geq 32\mathrm{UB}_{m,h}^{\mathrm{I}}$ and ${T_{m,h}^{(1)}}\left(\lambda_{m,h}^{\mathrm{I}}\right)^{2} \geq 2048\mathsf{c}_{1}\log(6h^*(\delta'')^{-1})$, we obtain inequalities (d) and (e). 
\end{proof}
\begin{proof}[Proof of Lemma~\ref{APCA-phase-I-convergence-lower-bound-last-iterate}]
	If $\mathcal{E}_{1,h}^{c}\bigcap\mathcal{E}_{2,h}^{c}$ occurs, there exists $t \in \left[T_{m,h}^{(1)}\right]$ such that $\alpha_{m,h}^{(t)}\geq \mathrm{UB}_{m,h}^{\mathrm{I}}$. 
	For some $t_{0}\in\left[T_{m,h}^{(1)}\right]$, define $\widehat{\tau}_{1}(t_{0})$ as the stopping time satisfying $\alpha_{m,h}^{(\widehat{\tau}_{1}(t_{0}))}\geq \mathrm{UB}_{m,h}^{\mathrm{I}}$ as:
	\begin{equation}\nonumber
		\widehat{\tau}_{1}(t_{0})=\inf_{t\geq t_{0}}\left\{ t:\alpha_{m,h}^{(t)}\geq \mathrm{UB}_{m,h}^{\mathrm{I}}\right\}.
	\end{equation}
	We also define $\widehat{\tau}_{2}(t_{0})$ as the stopping time satisfying $\alpha_{m,h}^{(\widehat{\tau}_{2}(t_{0}))}<\frac{1}{2}\mathrm{UB}_{m,h}^{\mathrm{I}}$ after $\widehat{\tau}_{1}(t_{0})$ as:
	\begin{equation}\nonumber
		\widehat{\tau}_{2}(t_{0}) =\inf_{t>\widehat{\tau}_{1}(t_{0})}\left\{ t:\alpha_{m,h}^{(t)}<\frac{1}{2}\mathrm{UB}_{m,h}^{\mathrm{I}}\right\}.
	\end{equation}
    According to the one-step dynamic in Eq.~\eqref{dynamic-hatalpha-lower-bound-proof} together with the step-size condition on $\eta_{m,h}^{(1)}$, the increment of $\alpha_{m,h}^{(t)}$ is uniformly bounded. Consequently, the trajectory cannot jump across the interval $\left[\frac{1}{2}\mathrm{UB}_{m,h}^{\mathrm{I}},
    \mathrm{UB}_{m,h}^{\mathrm{I}}\right]$
    in a single iteration. Therefore, on the event $\mathcal{E}_{1,h}^{c}\bigcap\mathcal{E}_{2,h}^{c}\bigcap\mathcal{E}_{3,h}$, there must exist some $t_{1}<t_{2}\in\left[T_{m,h}^{(1)}\right]$ such that the event
    \begin{align*}
        \mathcal{E}_{\widehat{\tau}_{1}(t_{0})=t_{1}}^{\widehat{\tau}_{2}(t_{0})=t_{2}}\coloneqq\left\{\alpha_{m,h}^{(t_{1})}\geq\mathrm{UB}_{m,h}^{\mathrm{I}}\bigcap\alpha_{m,h}^{(t_{1}:t_{2})}\in\left[\frac{1}{2}\mathrm{UB}_{m,h}^{\mathrm{I}},2\mathrm{UB}_{m,h}^{\mathrm{I}}\right]\bigcap\alpha_{m,h}^{(t_{2})}<\frac{1}{2}\mathrm{UB}_{m,h}^{\mathrm{I}}\right\}
    \end{align*}
    occurs.
	As similar as Eq.~\eqref{ascent-dynamic-alpha}, under the setting of step size $\eta_{m,h}^{(1)}$, we have
	\begin{align*}
		\alpha_{m,h}^{(t+1)} \geq \alpha_{m,h}^{(t)}+\frac{\eta_{m,h}^{(1)}\lambda_{m,h}^{\mathrm{I}}}{4}+\eta_{m,h}^{(1)}\cdot\xi_{m,h}^{(t+1)},
	\end{align*}
	for any $t\in[t_{1}:t_{2}-1]$. Since ${\xi}_{m,h}^{(t+1)}$ is sub-Gaussian with parameter $4\sqrt{\mathsf{c}_{1}}$, we establish the probability bound for event $\mathcal{E}_{\widehat{\tau}_{1}(t_{0})=t_{1}}^{\widehat{\tau}_{2}(t_{0})=t_{2}}$ with any time pair $t_{1}<t_{2}\in\left[T_{m,h}^{(1)}\right]$ as
	\begin{align}\label{PCA-lower-bound}
		\mathbb{P}\left(\mathcal{E}_{\widehat{\tau}_{1}(t_{0})=t_{1}}^{\widehat{\tau}_{2}(t_{0})=t_{2}}\right)\leq\exp\left\{-\frac{\left(\frac{1}{2}\mathrm{UB}_{m,h}^{\mathrm{I}}+\frac{(t_{2}-t_{1})}{4}\eta_{m,h}^{(1)}\lambda_{m,h}^{\mathrm{I}}\right)^{2}}{32(t_{2}-t_{1})\left(\eta_{m,h}^{(1)}\right)^{2}\mathsf{c}_{1}}\right\},
	\end{align} 
	by combining Lemma~\ref{aux-martingale-concentration-subtraction-coro}. Therefore, we have
	\begin{equation*}
		\begin{aligned}
			\mathbb{P} \left(\mathcal{E}_{1,h}^{c}\bigcap\mathcal{E}_{2,h}^{c}\bigcap\mathcal{E}_{3,h}\right) &\leq \sum_{1\leq t_{1}<t_{2}\leq{T_{m,h}^{(1)}}} \mathbb{P}\left(\mathcal{E}_{\widehat{\tau}_{1}(t_{0})=t_{1}}^{\widehat{\tau}_{2}(t_{0})=t_{2}}\right) \\
			&\leq \frac{\left(T_{m,h}^{(1)}\right)^{2}}{2}\min\left\{\exp\left\{-\frac{\Delta_{{T_{m,h}^{(1)}}}\left(\lambda_{m,h}^{\mathrm{I}}\right)^{2}}{512\mathsf{c}_{1}}\right\},\exp\left\{-\frac{\left(\mathrm{UB}_{m,h}^{\mathrm{I}}\right)^{2}}{128\Delta_{{T_{m,h}^{(1)}}}\left(\eta_{m,h}^{(1)}\right)^{2}\mathsf{c}_{1}}\right\}\right\} \\
			&\leq \frac{\delta''}{6h^*},
		\end{aligned}
	\end{equation*}
	where $\Delta_{{T_{m,h}^{(1)}}}=\frac{1024\mathsf{c}_{1}\log\left(3h^*T_{m,h}^{(1)}/\delta''\right)}{\left(\lambda_{m,h}^{\mathrm{I}}\right)^{2}}$, the second inequality is derived from Eq.~\eqref{PCA-lower-bound} and the last inequality follows from the setting of $\eta_{m,h}^{(1)}$.
\end{proof}
\begin{proof}[Proof of Theorem~\ref{thm-phase-I-tensor-PCA}]
    Combining Lemmas~\ref{APCA-phase-I-lower-bound}, \ref{APCA-phase-I-convergence-lower-bound} and \ref{APCA-phase-I-convergence-lower-bound-last-iterate} with replacing $\delta''$ by $\delta''/k$, Theorem~\ref{thm-phase-I-tensor-PCA} follows.
\end{proof}

\section{Proof of Phase II}\label{phase-II}
\subsection{Proof Outline}
In this section, we analyze \textbf{Phase~II} of Algorithm~\ref{MSANSGA}, where the iterates lie in a signal-dominated regime and converge toward the contraction neighborhood of the ground truth. More precisely, for any $m\in[k]$, we show that the alignment parameter $\alpha_{m}$ increases from the output level of \textbf{Phase~I} to $1-\widetilde{\epsilon}$ within $T_{m}^{(2)}$ iterations with high probability. In this phase, the effective block SNR is defined as follows:
\begin{align}\label{block-eff-SNR-phase-II}
\lambda_{m}^{\rmII}:=\lambda\cdot\prod_{i=1}^{m-1}\left|\left\langle\overline{W}_{i}^{\left(T_{i}^{(1)}\right)}, v_{2i-1}^{*}v_{2i}^{*\top}\right\rangle\right|\cdot\prod_{j=m+1}^{k}\left|\left\langle\overline{W}_{j}^{\left(T_{j}^{(1)}+T_{j}^{(2)}\right)}, v_{2j-1}^{*}v_{2j}^{*\top}\right\rangle\right|,
\end{align}
for any $m\in[k]$. We now state the main result for \textbf{Phase~II}.
\begin{theorem}\label{thm-phase-II-tensor-PCA}
	Assume $d\geq\Omega(k)$ and $\lambda\leq\mathcal{O}\left(\prod_{i=1}^{2k}d_{i}^{1/4}\right)$.  Under Assumptions \ref{ass-pro} and \ref{ass-base}, consider the dynamic generated via \textbf{Phase~II} of Algorithm~\ref{MSANSGA}, initialized using the output from \textbf{Phase~I}. For any $0<\delta''<1$, and $0<\widetilde{\epsilon}<1$, and $m\in[k]$, if we pick 
	\begin{gather*}
		\eta_{m}^{(2)}\leq\min\left\{\frac{\lambda_{m}^{\mathrm{II}}\widetilde{\epsilon}}{8192\mathsf{c}_{1}d_{2m-1}d_{2m}}, \frac{1}{32\left(\lambda_{m}^{\mathrm{II}}+\sqrt{\mathsf{c}_{1}}\right)}\right\}, \\ 
		\frac{\left(\mathrm{LB}_{m}^{\mathrm{II}}\right)^2}{64\mathsf{c}_{1}\left(\eta_{m}^{(2)}\right)^{2}\left(\log\left(kT_{m}^{(2)}\right)-\log\left(\delta''\right)\right)}\geq T_{m}^{(2)}\geq\frac{32}{\eta_{m}^{(2)}\lambda_{m}^{\mathrm{II}}},
	\end{gather*}
	where $\mathrm{LB}_{m}^{\mathrm{II}}\coloneqq\gamma_m^{1-1/(k-1)}$.
	Then $\alpha_{m}^{\left(T_{m}^{(1)}+T_{m}^{(2)}\right)} \ge 1-\widetilde{\epsilon}$ holds for every $m\in[k]$ with probability at least $1-\delta''/2$. 
\end{theorem}

The proof follows the same high-probability decomposition framework as in \textbf{Phase~I}. For fixed $m\in[k]$, define
\begin{gather*}
	\mathcal{E}_{4}=\left\{\exists t\in \left[T_{m}^{(2)}\right], \alpha_{m}^{\left(T_{m}^{(1)}+t\right)} < \mathrm{LB}_{m}^{\mathrm{II}}\right\}, \quad\mathcal{E}_{5}=\left\{\max_{t\in \left[T_{m}^{(2)}\right]} \alpha_{m}^{\left(T_{m}^{(1)}+t\right)} <  1-\frac{\widetilde{\epsilon}}{2}\right\}, \\
    \mathcal{E}_{6}=\left\{\alpha_{m}^{\left(T_{m}^{(1)}+{T_{m}^{(2)}}\right)} < 1-\widetilde{\epsilon} \right\},
\end{gather*}
These events are direct analogues of $\mathcal{E}_{1,h}$--$\mathcal{E}_{3,h}$ in \textbf{Phase~I}. Decomposing $\mathcal{E}_{6}$ as in \textbf{Phase~I}, we have
\begin{align*}
	\mathbb{P}\left(\mathcal{E}_{6}\right) &\leq \mathbb{P}\left(\mathcal{E}_{4}\right) + \mathbb{P}\left(\mathcal{E}_{4}^{c} \bigcap \mathcal{E}_{5} \right) + \mathbb{P}\left(\mathcal{E}_{4}^{c} \bigcap \mathcal{E}_{5}^{c} \bigcap \mathcal{E}_{6} \right)
	\leq 3\times \frac{\delta''}{6} \leq \frac{\delta''}{2}.
\end{align*}
Therefore, it suffices to show that for each $l\in\{4,5,6\}$,
\begin{align*}
    \mathbb{P}\left[\left(\bigcap_{4\leq l'<l}\mathcal{E}_{l'}^{c}\right)\bigcap \mathcal{E}_l \right] \leq \frac{\delta''}{6},
\end{align*}
which implies $\mathbb{P}[\mathcal{E}_{6}]\leq \delta''/2$. Then we establish these bounds in the following three lemmas. In this phase, we set $\delta'\asymp\delta''/\left(kd_{2k}\max_m\left\{T_{m}^{(2)}\right\}\right)^8$ wherever the technical lemmas from Section \ref{tech-lemmma} is applied.

\begin{lemma}\label{APCA-phase-II-lower-bound}
	Assume $d_{2m-1}, d_{2m}\geq\Omega(k)$ and $\lambda\leq\mathcal{O}\left(\prod_{i=1}^{\overline{k}}d_{i}^{1/4}\right)$, and suppose $\eta_{m}^{(2)}\leq g_{1}^{*}(m)$, and ${T_{m}^{(2)}}\left(\eta_{m}^{(2)}\right)^{2}\leq g_{2}^{*}(m)$, where
    \begin{equation}\label{para-lemma-upper-bound-II}
        \begin{gathered}
            g_{1}^{*}(m)\coloneqq\frac{1}{4096}\cdot \min\left\{\frac{\lambda_{m}^{\mathrm{II}}}{\mathsf{c}_{1}\left[\left(\lambda_{m}^{\mathrm{II}}\right)^{2}+\mathrm{LB}_{m}^{\mathrm{II}}d_{2m-1}d_{2m}\right]},\frac{1}{\lambda_{m}^{\mathrm{II}}+\sqrt{\mathsf{c}_{1}}}\right\}, \\
        	g_{2}^{*}(m)\coloneqq\frac{\left(\mathrm{LB}_{m}^{\mathrm{II}}\right)^2}{64\mathsf{c}_{1}\log\left(3\left(T_{m}^{(2)}\right)^{2}/\delta''\right)}.
        \end{gathered}
	\end{equation}
	Under the Assumption of Theorem~\ref{thm-phase-II-tensor-PCA}, the event $\mathcal{E}_{4}$ holds with probability at most $\frac{\delta''}{6}$.
\end{lemma}
\begin{lemma}\label{APCA-phase-II-convergence-lower-bound}
	Assume $d_{2m-1}, d_{2m}\geq\Omega(k)$ and $\lambda\leq\mathcal{O}\left(\prod_{i=1}^{\overline{k}}d_{i}^{1/4}\right)$, and suppose $\eta_{m}^{(2)}\leq \min\{g_{1}^{*}(m),\allowbreak g_{3}^{*}(m)\}$ ($g_{1}^{*}(m)$ is defined in Eq.~\eqref{para-lemma-upper-bound-II}), and ${T_{m}^{(2)}}\geq g_{4}^{*}(m)$, and ${T_{m}^{(2)}}\eta_{m}^{(2)}\geq g_{5}^{*}(m)$, where
	\begin{gather*}
        g_{3}^{*}(m)\coloneqq\frac{\lambda_{m}^{\mathrm{II}}}{4096\mathsf{c}_{1}\left[\left(\lambda_{m}^{\mathrm{II}}\right)^{2}+\mathrm{UB}_{m}^{\mathrm{II}}d_{2m-1}d_{2m}\right]} \\
		g_{4}^{*}(m)\coloneqq\frac{2048\mathsf{c}_{1}\log(6/\delta'')}{\left(\lambda_{m}^{\mathrm{II}}\right)^{2}},\quad g_{5}^{*}(m)\coloneqq\frac{32\mathrm{UB}_{m}^{\mathrm{II}}}{\lambda_{m}^{\mathrm{II}}},\notag
	\end{gather*}
	with $\mathrm{UB}_{m}^{\mathrm{II}}\coloneqq1-\widetilde{\epsilon}/2$.
	Under the setting of Theorem~\ref{thm-phase-II-tensor-PCA}, the combined event $\mathcal{E}_{4}^{c}\bigcap\mathcal{E}_{5}$ holds with probability at most $\frac{\delta''}{6}$.
\end{lemma}
\begin{lemma}\label{APCA-phase-II-convergence-lower-bound-last-iterate}
	Assume $d_{2m-1}, d_{2m}\geq\Omega(k)$ and $\lambda\leq\mathcal{O}\left(\prod_{i=1}^{\overline{k}}d_{i}^{1/4}\right)$, and suppose $\eta_{m}^{(2)}\leq \min\{g_{1}^{*}(m), \allowbreak g_{3}^{*}(m),g_{6}^{*}(m)\},$
	where
	\begin{align*}
	   g_{6}^{*}(m)\coloneqq\frac{\widetilde{\epsilon}\lambda_{m}^{\mathrm{II}}}{4096}\cdot\min\left\{\frac{1}{d_{2m-1}d_{2m}},\frac{\widetilde{\epsilon}}{\mathsf{c}_{1}\log\left(3\left(T_{m}^{(2)}\right)^{2}/\delta''\right)}\right\}.
	\end{align*}
	Under the setting of Theorem~\ref{thm-phase-II-tensor-PCA}, the combined event $\mathcal{E}_{4}^{c}\bigcap\mathcal{E}_{5}^{c}\bigcap\mathcal{E}_{6}$ holds with probability at most $\frac{\delta''}{6}$.
\end{lemma}

\subsection{Proof Details}
\begin{proof}[Proof of Lemma~\ref{APCA-phase-II-lower-bound}]
    The proof is analogous to that of Lemma~\ref{APCA-phase-I-lower-bound}.
\end{proof}
\begin{proof}[Proof of Lemma~\ref{APCA-phase-II-convergence-lower-bound}]
    The proof is analogous to that of Lemma~\ref{APCA-phase-I-convergence-lower-bound}.
\end{proof}
\begin{proof}[Proof of Lemma~\ref{APCA-phase-II-convergence-lower-bound-last-iterate}]
	If $\mathcal{E}_{4}^{c}\bigcap\mathcal{E}_{5}^{c}$ occurs, there exists $t \in \left[T_{m}^{(2)}\right]$ such that $\alpha_{m}^{\left(T_{m}^{(1)}+t\right)}\geq 1-\frac{{\widetilde{\epsilon}}}{2}$. 
	For some $t_{0}\in\left[T_{m}^{(2)}\right]$, define $\widehat{\tau}_{1}(t_{0})$ as the stopping time satisfying $\alpha_{m}^{(\widehat{\tau}_{1}(t_{0}))}\geq 1-\frac{{\widetilde{\epsilon}}}{2}$ as:
	\begin{equation}\nonumber
		\widehat{\tau}_{1}(t_{0})=\inf_{t\geq t_{0}}\left \{ t:\alpha^{\left(T_{m}^{(1)}+t\right)}\geq 1-\frac{{\widetilde{\epsilon}}}{2}\right\}.
	\end{equation}
	We also define $\widehat{\tau}_{2}(t_{0})$ as the stopping time satisfying $\alpha_{m}^{(\widehat{\tau}_{2}(t_{0}))}<1-{\widetilde{\epsilon}}$ after $\widehat{\tau}_{1}(t_{0})$ as:
	\begin{equation}\nonumber
		\widehat{\tau}_{2}(t_{0}) =\inf_{t>\widehat{\tau}_{1}(t_{0})}\left \{ t:\alpha_{m}^{\left(T_{m}^{(1)}+t\right)}<1-{\widetilde{\epsilon}}\right\}.
	\end{equation}

    According to the one-step dynamic in Eq.~\eqref{dynamic-hatalpha-lower-bound-proof} together with the step-size condition on $\eta_{m}^{(2)}$, the increment of $\alpha_{m}^{(t)}$ is uniformly bounded. Consequently, the trajectory cannot jump across the interval $\left[1-{\widetilde{\epsilon}},1-\frac{{\widetilde{\epsilon}}}{4}\right]$
    in a single iteration. Therefore, on the event $\mathcal{E}_{4}^{c}\bigcap\mathcal{E}_{5}^{c}\bigcap\mathcal{E}_{6}$, there must exist some $t_{1}<t_{2}\in\left[T_{m}^{(2)}\right]$ such that the event
	\begin{equation}\nonumber
		\mathcal{E} _{\widehat{\tau}_{1}(t_{0})=t_{1}}^{\widehat{\tau}_{2}(t_{0})=t_{2}}\coloneqq\left\{\alpha_{m}^{\left(T_{m}^{(1)}+t_{1}\right)}\geq1-\frac{{\widetilde{\epsilon}}}{2}\bigcap\alpha_{m}^{\left(T_{m}^{(1)}+t_{1}:T_{m}^{(1)}+t_{2}\right)}\in\left[1-{\widetilde{\epsilon}},1-\frac{{\widetilde{\epsilon}}}{4}\right]\bigcap\alpha_{m}^{\left(T_{m}^{(1)}+t_{2}\right)}<1-{\widetilde{\epsilon}}\right\}
	\end{equation}
    occurs.
	For any $t\in\left[T_{m}^{(1)}+t_{1}:T_{m}^{(1)}+t_{2}-1\right]$, we have 
	\begin{align}\label{martingale-concen-lower-bound}
		1-\alpha_{m}^{(t+1)} &= 1-\alpha_{m}^{\left(t\right)}-\eta_{m}^{(2)}\lambda_{m}^{\mathrm{II}}\cdot\left(1+\alpha_{m}^{\left(t\right)}\right)\cdot\left(1-\alpha_{m}^{\left(t\right)}\right) \notag \\
		&\phantom{\mathrel{=}}-\eta_{m}^{(2)}\left[\left\langle{E}_{m}^{(t+1)}, v_{2m-1}^{*} v_{2m}^{*{\top}}-\alpha_{m}^{\left(t\right)}\frac{W_{m}^{\left(t\right)}}{\left\|W_{m}^{\left(t\right)}\right\|_{\mathrm{F}}}\right\rangle\right] \notag \\
		&\phantom{\mathrel{=}}-\frac{\left(\eta_{m}^{(2)}\right)^{2}}{2}\Psi_{1}\left(W_{m}^{\left(t\right)},G_{m}^{\left(t\right)},v_{2m-1}^{*},v_{2m}^{*},\bar{\eta}_{m}^{(2)}\right) \notag \\
		&\overset{\text{(a)}}{\leq} \left(1-(1-\widetilde{\epsilon})\eta_{m}^{(2)}\lambda_{m}^{\mathrm{II}}\right)\left(1-\alpha_{m}^{\left(t\right)}\right)-\eta_{m}^{(2)}\xi_{m}^{(t+1)},
	\end{align}
	where (a) is derived from combining Eq.~\eqref{esti-Psi1} with the setting of $\eta_{m}^{(2)}$ which implicates that 
	\begin{align*}
		\frac{\widetilde{\epsilon}\lambda_{m}^{\mathrm{II}}\eta_{m}^{(2)}}{8}\geq\eta_{m}^{(2)}\left|\mathbb{E}_t\left[\left\langle E_{m}^{(t+1)},v_{2m-1}^{*}v_{2m}^{*\top}\right\rangle\right]\right|+\eta_{m}^{(2)}\frac{1}{\left\|W_{m}^{\left(T_{m}^{(1)}+t\right)}\right\|_{\mathrm{F}}}\cdot\left|\mathbb{E}_t\left[\left\langle E_{m}^{(t+1)},W_{m}^{\left(T_{m}^{(1)}+t\right)}\right\rangle\right]\right|,
	\end{align*}
	and the result of Lemma~\ref{estimation} with the setting of $\eta_{m}$ which implicates that
	\begin{align}\label{eta-II-constraint}
		\frac{\widetilde{\epsilon}\lambda_{m}^{\mathrm{II}}\eta_{m}^{(2)}}{8} &\geq \left(\eta_{m}^{(2)}\right)^{2}\cdot\frac{32\mathsf{c}_{1}\left[\left(\lambda_{m}^{\mathrm{II}}+2\right)^{2}+2(d_{2m-1}d_{2m}+4)\right]}{\left[1-4\bar{\eta}_{m}^{(2)}\left(\lambda_{m}^{\mathrm{II}}+\sqrt{\mathsf{c}_{1}}\right)\right]^{3}} \notag \\
        &\geq \frac{\left(\eta_{m}^{(2)}\right)^{2}}{2}\left|\Psi_{1} \left(W_{m}^{\left(T_{m}^{(1)}+t\right)},G_{m}^{\left(T_{m}^{(1)}+t\right)},v_{2m-1}^{*},v_{2m}^{*},\overline{\eta}_{m}^{(2)}\right)\right|.
	\end{align}
	Since $\xi_{m}^{(t+1)}$ is sub-Gaussian with parameter $4\sqrt{\mathsf{c}_{1}}$, we establish the probability bound for event $\mathcal{E}_{\widehat{\tau}_{1}(t_{0})=t_{1}}^{\widehat{\tau}_{2}(t_{0})=t_{2}}$ with any time pair $t_{1}<t_{2}\in\left[T_{m}^{(2)}\right]$ as
	\begin{align}\label{PCA-lower-bound-II}
		\mathbb{P}\left(\mathcal{E}_{\widehat{\tau}_{1}(t_{0})=t_{1}}^{\widehat{\tau}_{2}(t_{0})=t_{2}}\right)\leq\exp\left\{-\frac{\lambda_{m}^{\mathrm{II}} (1-{\widetilde{\epsilon}}){\widetilde{\epsilon}}^{2}}{256\eta_{m}^{(2)}\mathsf{c}_{1}}\right\},
	\end{align} 
	by combining Lemma~\ref{aux-martingale-concentration}. Therefore, we have
	\begin{align*}\label{E-II-constraint}
        \mathbb{P}\left(\mathcal{E}_{4}^{c} \bigcap \mathcal{E}_{5}^{c} \bigcap \mathcal{E}_{6}\right) \leq \sum_{1\leq t_{1}<t_{2}\leq T_{m}^{(2)}}\mathbb{P}\left(\mathcal{E}_{\widehat{\tau}_{1}(t_{0})=t_{1}}^{\widehat{\tau}_{2}(t_{0})=t_{2}}\right) \leq \frac{\left(T_{m}^{(2)}\right)^{2}}{2}\exp\left\{-\frac{\lambda_{m}^{\mathrm{II}} (1-{\widetilde{\epsilon}}){\widetilde{\epsilon}}^{2}}{256\eta_{m}^{(2)}\mathsf{c}_{1}}\right\} \leq \frac{\delta''}{6},
	\end{align*}
	where the second inequality is derived from Eq.~\eqref{PCA-lower-bound-II} and the last inequality follows from the setting of $\eta_{m}^{(2)}$.
\end{proof}
\begin{proof}[Proof of Theorem~\ref{thm-phase-II-tensor-PCA}]
    Combining Lemmas~\ref{APCA-phase-II-lower-bound}, \ref{APCA-phase-II-convergence-lower-bound} and \ref{APCA-phase-II-convergence-lower-bound-last-iterate} with replacing $\delta''$ by $\delta''/k$, Theorem~\ref{thm-phase-II-tensor-PCA} follows.
\end{proof}

\section{Proof of Phase III}\label{phase-III}
\subsection{Proof Outline}
In this section, we analyze \textbf{Phase~III} of Algorithm~\ref{MSANSGA}, where the iterates are already well aligned and evolve within a local regime around the ground truth. We begin by stating the two main results for Phase~III under sub-Gaussian and Gaussian noise (Theorems~\ref{thm-phase-III-tensor-PCA-sub-Gaussian} and \ref{thm-phase-III-tensor-PCA}). Both results are established under a unified analytical framework, and differ only in the concentration properties of the noise terms.
\begin{theorem}\label{thm-phase-III-tensor-PCA}
	Assume $d\geq\Omega(k)$ and $\lambda\leq\mathcal{O}\left(\prod_{i=1}^{2k}d_{i}^{1/4}\right)$. Under Assumptions \ref{ass-pro} and \ref{assumption-Gaussian}, consider the dynamic generated via \textbf{Phase~III} of Algorithm~\ref{MSANSGA}, initialized using the output from \textbf{Phase~II}. For any $0<\delta''<1$, and $0<\widehat{\epsilon}<1$, and $m\in[k]$, if we pick the hyper-parameters setting as same as that of Theorem~\ref{thm-phase-III-tensor-PCA-sub-Gaussian}, then $\alpha_{m}^{(T_{m})}$ satisfies the following bound
	\begin{align}
		\left(1-\alpha_{m}^{(T_{m})}\right)^{2} &\lesssim \left(1-\frac{\eta_{m}^{(3)}\lambda_{m}^{\mathrm{III}}\left(1-\frac{3}{2}\widehat{\epsilon}\right)}{2}\right)^{\left\lfloor T_{m}^{(3)}/\log \left(T_{m}^{(3)}\right)\right\rfloor}\frac{\widehat{\epsilon}^{2}}{\delta''} + \frac{\eta_{m}^{(3)}\left\lceil\log \left(T_{m}^{(3)}\right)\right\rceil}{\left(\lambda_{m}^{\mathrm{III}}\right)^{2}\left(1-\frac{3}{2}\widehat{\epsilon}\right)^{2}\delta'' \left(T_{m}^{(3)}\right)^{4}} \notag \\
        &\phantom{\mathrel{=}}+ \frac{\left(\left(\lambda_{m}^{\mathrm{III}}\right)^{4}+\mathsf{c}_{1}^{2}d_{2m-1}^{2}d_{2m}^{2}\right)\eta_{m}^{(3)}\left\lceil\log \left(T_{m}^{(3)}\right)\right\rceil}{\left(\lambda_{m}^\mathrm{III}\right)^3\left(1-\frac{3}{2}\widehat{\epsilon}\right)^{3}\delta'' T_{m}^{(3)}},\notag
	\end{align}
	with probability at least $1-2\delta''$, where $T_{m}=T_{m}^{(1)}+T_{m}^{(2)}+T_{m}^{(3)}$.
\end{theorem}
\begin{theorem}\label{thm-phase-III-tensor-PCA-sub-Gaussian}
	Assume $d\geq\Omega(k)$ and $\lambda\leq\mathcal{O}\left(\prod_{i=1}^{2k}d_{i}^{1/4}\right)$. Under Assumptions \ref{ass-pro} and \ref{ass-base}, consider the dynamic generated via \textbf{Phase~III} of Algorithm~\ref{MSANSGA}, initialized using the output from \textbf{Phase~II}. For any $0<\delta''<1$, and $0<\widehat{\epsilon}<1$, and $m\in[k]$, we pick
	\begin{gather*}
		\eta_{m}^{(3)} \leq \min\left\{\frac{\lambda_{m}^{\mathrm{III}}\widehat{\epsilon}}{8192\left(\left(\lambda_{m}^{\mathrm{III}}\right)^{2}+\mathsf{c}_{1}d_{2m-1}d_{2m}\right)}, \frac{\lambda_{m}^{\mathrm{III}}\left(1-\frac{3}{2}\widehat{\epsilon}\right)\widehat{\epsilon}^{2}}{16\mathsf{c}_{1}\log\left(k\left(T_{m}^{(3)}\right)^{2}/{\delta''}\right)}\right\}, \\
		T_{m}^{(3)} \gtrsim 1, \\
		\eta_{m}^{(3,t)} = \eta_{m}^{(3)} \cdot 2^{-\left\lfloor t/\widehat{T}_{m}^{(3)} \right\rfloor},
	\end{gather*}
	where
	\begin{align*}
		\widehat{T}_{m}^{(3)} = \left\lfloor T_{m}^{(3)}/\log \left(T_{m}^{(3)}\right)\right\rfloor.
	\end{align*}
	Then every $\alpha_{m}^{(T_{m})}$ satisfies the following bound
	\begin{align}
		\left(1-\alpha_{m}^{(T_{m})}\right)^{2} &\lesssim \left(1-\frac{\eta_{m}^{(3)}\lambda_{m}^{\mathrm{III}}\left(1-\frac{3}{2}\widehat{\epsilon}\right)}{2}\right)^{\left\lfloor T_{m}^{(3)}/\log \left(T_{m}^{(3)}\right)\right\rfloor}\frac{k\widehat{\epsilon}^{2}}{\delta''} + \frac{k\eta_{m}^{(3)}\left\lceil\log \left(T_{m}^{(3)}\right)\right\rceil}{\left(\lambda_{m}^{\mathrm{III}}\right)^{2}\left(1-\frac{3}{2}\widehat{\epsilon}\right)^{2}\delta''\left(T_{m}^{(3)}\right)^{4}} \notag \\
		&\phantom{\mathrel{=}}+ \frac{k\left(\left(\lambda_{m}^{\mathrm{III}}\right)^{4}+\mathsf{c}_{1}^{2}d_{2m-1}^{2}d_{2m}^{2}\right)\eta_{m}^{(3)}\left\lceil\log \left(T_{m}^{(3)}\right)\right\rceil}{\left(\lambda_{m}^\mathrm{III}\right)^3\left(1-\frac{3}{2}\widehat{\epsilon}\right)^{3}\delta'' T_{m}^{(3)}}+\frac{k\left\lceil\log\left(T_m^{(3)}\right)\right\rceil}{\delta''\left(\lambda_{m}^{\mathrm{III}}\right)^2(1-\frac{3}{2}\widehat{\epsilon})^2T_{m}^{(3)}},\notag
	\end{align}
	with probability at least $1-2\delta''$, where $T_{m}=T_{m}^{(1)}+T_{m}^{(2)}+T_{m}^{(3)}$.
\end{theorem}
The analysis of Phase~III follows a two-step strategy. We first establish a high-probability stability event under which the alignment parameter remains confined to a small neighborhood of one throughout this phase. This ensures that the iterates stay within a local regime where a linearized approximation of the dynamics is valid. Conditioned on this event, we then derive a refined one-step recursion for the alignment error and analyze its convergence via a bias--variance decomposition. In this section, we consider without loss of generality the case $\alpha_{m}^{\left(T_{m}^{(1)}+T_{m}^{(2)}\right)}>0$. Indeed, if $\alpha_{m}^{\left(T_{m}^{(1)}+T_{m}^{(2)}\right)}<0$, Algorithm \ref{SGA-decay-ss} exists a trajectory of $W_{m}^{(t)}$ under gradient ascent on $-\widehat{\calR}_{m}^{(t)}$ for $t\in\left[T_{m}^{(1)}+T_{m}^{(2)}:T_{m}\right]$, which is equivalent to the trajectory under gradient ascent on $\widehat{\calR}_{m}^{(t)}$ when $\alpha_{m}^{\left(T_{m}^{(1)}+T_{m}^{(2)}\right)}>0$. We begin with the high-probability stability guarantee. In this phase, we set $\delta'\asymp\delta''/\left(kd_{2k}\max_m\left\{T_{m}^{(3)}\right\}\right)^8$ wherever the technical lemmas from Section \ref{tech-lemmma} is applied.
\begin{lemma}\label{phase-III-high-probability}
	Define $\overline{\epsilon}=\frac{3}{2}\widehat{\epsilon}$. Under the setting of Theorem~\ref{thm-phase-III-tensor-PCA-sub-Gaussian}, we consider SGD iterates starting from step $T_{m}^{(1)} + T_{m}^{(2)}$ with initialization $\alpha_{m}^{\left(T_{m}^{(1)} + T_{m}^{(2)}\right)}>1-\overline{\epsilon}$. Then the joint event \begin{align*}
	    \bigcap_{t=T_{m}^{(1)}+T_{m}^{(2)}}^{T_{m}} \widetilde{\mathcal{E}}\left(\alpha_{m}^{(t)}\right)
	\end{align*} holds with probability at least $1-\delta''/2$, where 
	\begin{align*}
		\widetilde{\mathcal{E}}\left(\alpha\right)\coloneqq\left\{\alpha\in\left[1-\overline{\epsilon},1\right]\right\}.
	\end{align*}
\end{lemma}
Lemma~\ref{phase-III-high-probability} ensures that, with high probability,
the alignment remains close to one throughout Phase~III, more precisely,
$\alpha_{m}^{(t)}\in[1-\overline{\epsilon}, 1]$ for any $t\in\left[T_{m}^{(1)}+T_{m}^{(2)}:T_{m}\right]$, which implies that the iterates stay in a locally stable neighborhood of the ground truth
$ v_{2m-1}^{*}v_{2m}^{*\top}$. Within this regime, the dynamics admit a linear approximation, enabling a refined characterization of the convergence behavior. To exploit this high-probability stability and facilitate an unconditional analysis, we introduce a truncated sequence that coincides with the original iterates as long as the event $\widetilde{\mathcal{E}} \left(\alpha_{m}^{(t)}\right)$ continues to hold, and is absorbed at the optimum once the event is violated. 
%Initialize $O_{m}^{(0)}=W_{m}^{\left(T_{m}^{(1)}+T_{m}^{(2)}\right)}$. For $t\geq 0$, define $\left\{O_{m}^{(t)}\right\}_{t=0}^{T_{m}^{(3)}}$ recursively by
%\begin{align*}
%    O_{m}^{(t)}=
%    \begin{cases}
%        W_{m}^{\left(T_{m}^{(1)}+T_{m}^{(2)}+t\right)}, & \text{if }\widetilde{\mathcal{E}} \left(\alpha_{m}^{\left(T_{m}^{(1)}+T_{m}^{(2)}+s\right)}\right) \text{ holds for all }0\leq s\leq t, \\
%        v_{2m-1}^{*}v_{2m}^{*\top}, & \text{otherwise}.
%    \end{cases}
%\end{align*}
Initialize $P_{m}^{(0)}=W_{m}^{\left(T_{m}^{(1)}+T_{m}^{(2)}\right)}$. For $t\geq 0$, define
\begin{align}
	\begin{cases}
		P_{m}^{(t+1)}=W_{m}^{\left(T_{m}^{(1)}+T_{m}^{(2)}+t+1\right)}, & \text{ if }\widetilde{\calE}\left(\alpha_{m}^{\left(T_{m}^{(1)}+T_{m}^{(2)}+t\right)}\right)\text{ occurs},
		\\
		P_{m}^{(\tau+1)}= v_{2m-1}^{*,\perp} \left(v_{2m}^{*,\perp}\right)^{\top},\, \, \, \forall \tau\geq t, & \text{ otherwise}.
	\end{cases}\notag
\end{align}
Moreover, we define the auxiliary function $\psi_{m}:\bbR^{d_{2m-1}\times d_{2m}}\rightarrow\bbR^{d_{2m-1}\times d_{2m}}$ as:
\begin{align}
	\psi_{m}(P_{m})=\begin{cases}
		P_{m}, & \text{ if }\widetilde{\calE}\left(\frac{\llangle v_{2m-1}^*,P_{m} v_{2m}^*\rrangle}{\|P_{m}\|_{\tF}}\right)\text{ occurs},
		\\
		v_{2m-1}^*v_{2m}^{*\top}, & \text{ otherwise}.
	\end{cases}\notag
\end{align}
We construct the truncated sequence $\left\{ O^{(t)}=\psi_{m}\left(P_{m}^{(t)}\right)\right\}_{t=0}^{T_{m}^{(3)}}$.
Our analysis is based on the evolution of the alignment parameter $\beta_{m}^{(t)}\coloneqq \frac{\left\langle O_{m}^{(t)}, v_{2m-1}^{*}v_{2m}^{*\top}\right\rangle}{\left\|O_{m}^{(t)}\right\|_{\mathrm{F}}}$, which can be viewed as a truncated counterpart of $\alpha_{m}^{\left(T_{m}^{(1)}+T_{m}^{(2)}+t\right)}$. By definition, the sequence $\left\{O_{m}^{(t)}\right\}$ follows the original iterates $\left\{W_{m}^{\left(T_{m}^{(1)}+T_{m}^{(2)}+t\right)}\right\}$ up to the first violation of the event $\widetilde{\mathcal{E}} \left(\alpha_{m}^{\left(T_{m}^{(1)}+T_{m}^{(2)}+t\right)}\right)$, and is thereafter absorbed at the ground truth. Since $\bigcap_{t=T_{m}^{(1)}+T_{m}^{(2)}}^{T}\widetilde{\mathcal{E}} \left(\alpha_{m}^{(t)}\right)$ holds with high probability, the truncation is inactive throughout Phase~III, and thus $\beta_{m}^{(t)}$ matches the alignment of the original iterates throughout this phase.

We now derive the one-step evolution of the alignment parameter $\beta_{m}^{(t)}$. By construction, the truncated sequence $\left\{O_{m}^{(t)}\right\}$ follows the original iterate as long as the event $\widetilde{\mathcal{E}}\left(\alpha_{m}^{\left(T_{m}^{(1)}+T_{m}^{(2)}+t\right)}\right)$ remains valid, in which case $\beta_{m}^{(t)}$ admits the expansion
\begin{align}
	\beta_{m}^{(t+1)} &= \beta_{m}^{(t)}+\eta_{m}^{(3,t)}\lambda_{m}^{\mathrm{III}}\left(1-\left(\beta_{m}^{(t)}\right)^{2}\right) + \eta_{m}^{(3,t)}\left\langle{E}_{m}^{(t+1)}, v_{2m-1}^{*}v_{2m}^{*\top} - \beta_{m}^{(t)}\frac{O_{m}^{(t)}}{\left\|O_{m}^{(t)}\right\|_{\mathrm{F}}}\right\rangle \notag \\
	&\phantom{\mathrel{=}}+\frac{\left(\eta_{m}^{(3,t)}\right)^{2}}{2}\Psi_{1}\left(O_{m}^{(t)},G_{m}^{(t)},v_{2m-1}^*,v_{2m}^{*},\overline{\eta}_{m}^{(t)}\right),\notag
\end{align}
where $G_{m}^{(t)}$ denotes the stochastic gradient evaluated at $O_{m}^{(t)}$,
and $\overline{\eta}_{m}^{(t)}\in\left[0,\eta_{m}^{(3,t)}\right]$ arises from the second-order remainder. Otherwise, once $\widetilde{\mathcal{E}}\left(\alpha_{m}^{\left(T_{m}^{(1)}+T_{m}^{(2)}+t\right)}\right)$ is violated, the truncation becomes active, and the sequence is absorbed at the ground truth, so that $\beta_{m}^{(\tau)}\equiv 1$ for all subsequent $\tau\geq t$.

For convenience, we introduce the error variable $\widehat{\beta}_{m}^{(t)}\coloneqq 1-\beta_{m}^{(t)}$ and denote $f_{\beta_{m}^{(t)}}\coloneqq \lambda_{m}^{\mathrm{III}}\left(1+\beta_{m}^{(t)}\right)$. Whenever the truncation is inactive, $\widehat{\beta}_{m}^{(t)}$ admits the iteration
\begin{align}\label{update}
	\widehat{\beta}_{m}^{(t+1)}
	=\left(1-\eta_{m}^{(3,t)} f_{\beta_{m}^{(t)}}\right)\widehat{\beta}_{m}^{(t)}
	+\eta_{m}^{(3,t)}\overline{g}_{m}^{(t)}
	+\left(\eta_{m}^{(3,t)}\right)^{2}\widetilde{g}_{m}^{(t)},
\end{align}
where $\overline{g}_{m}^{(t)}$ and $\widetilde{g}_{m}^{(t)}$ are defined by
\begin{align}
	\begin{cases}
		\overline{g}_{m}^{(t)}\coloneqq\left\langle{E}_{m}^{(t+1)}, v_{2m-1}^{*}v_{2m}^{*\top} - \beta_{m}^{(t)}\frac{O_{m}^{(t)}}{\left\|O_{m}^{(t)}\right\|_{\mathrm{F}}}\right\rangle, \\\\
		\widetilde{g}_{m}^{(t)}\coloneqq-\frac{1}{2}\Psi_{1}\left(O_{m}^{(t)},G_{m}^{(t)},v_{2m-1}^{*},v_{2m}^{*},\overline{\eta}_{m}^{(t)}\right),
	\end{cases}\notag
\end{align}
Otherwise, once the truncation is activated, the sequence is absorbed at the ground truth, and therefore
\begin{align}\label{non-update}
	\widehat{\beta}_{m}^{(\tau+1)}=0, \qquad \forall \tau\geq t.
\end{align}
Combining Eq.~\eqref{update} with Eq.~\eqref{non-update}, the iterative update of $\left(\widehat{\beta}_{m}^{(t)}\right)^{2}$ can be expressed as:
\begin{align}\label{non-expect-recursion}
	\left(\widehat{\beta}_{m}^{(t+1)}\right)^{2} &\leq \left(\left(1-\eta_{m}^{(3,t)} f_{\beta_{m}^{(t)}}\right)\widehat{\beta}_{m}^{(t)}+\eta_{m}^{(3,t)}\overline{g}_{m}^{(t)}+\left(\eta_{m}^{(3,t)}\right)^{2}\widetilde{g}_{m}^{(t)}\right)^{2}\cdot\mathds{1}_{O_{m}^{(t)}=W^{\left(T_{m}^{(1)}+T_{m}^{(2)}+t\right)}}\notag \\
    &\leq \left(\left(1-\eta_{m}^{(3,t)} f_{\beta_{m}^{(t)}}\right)^{2}\left(\widehat{\beta}_{m}^{(t)}\right)^{2}+2\left(1-\eta_{m}^{(3,t)} f_{\beta_{m}^{(t)}}\right)\widehat{\beta}_{m}^{(t)}\left(\eta_{m}^{(3,t)}\overline{g}_{m}^{(t)}+\left(\eta_{m}^{(3,t)}\right)^{2}\widetilde{g}_{m}^{(t)}\right)\right.\notag \\
    &\phantom{\mathrel{=}} \left.+ 2\left(\left(\eta_{m}^{(3,t)}\right)^{2}\left(\overline{g}_{m}^{(t)}\right)^{2}+\left(\eta_{m}^{(3,t)}\right)^{4}\left(\widetilde{g}_{m}^{(t)}\right)^{2}\right)\right) \cdot \mathds{1}_{O_{m}^{(t)}=W^{\left(T_{m}^{(1)}+T_{m}^{(2)}+t\right)}}.
\end{align}
Taking expectation and bounding the stochastic terms yields the following recursion.
\begin{lemma}\label{phase-III-error-recursion}
    Under Assumption~\ref{assumption-Gaussian} and the setting of Lemma~\ref{phase-III-high-probability}, for all $t\in[0:T_{3}-1]$,
    \begin{align}\label{main-recursion}
    	\mathbb{E}\left[\left(\widehat{\beta}_{m}^{(t+1)}\right)^{2}\right]\leq \left(1-\eta_{m}^{(3,t)}\lambda_{m}^{\mathrm{III}}\left(1-\overline{\epsilon}\right)\right)\mathbb{E}\left[\left(\widehat{\beta}_{m}^{(t)}\right)^{2}\right]+\left(\eta_{m}^{(3,t)}\right)^{3}g_{m}^{(t)}, 
    \end{align}
    where $g_{m}^{(t)}=\frac{1}{\left(\eta_{m}^{(3,t)}\right)^{3}\left(T_{m}^{(3)}\right)^6} + \frac{33554432\left(\left(\lambda_{m}^{\mathrm{III}}\right)^{4}+\mathsf{c}_{1}^{2}d_{2m-1}^{2}d_{2m}^{2}\right)}{\lambda_{m}^{\mathrm{III}}(1-\overline{\epsilon})}$.
\end{lemma}
\begin{lemma}\label{phase-III-error-recursion-sub-Gaussian}
    Under the setting of Lemma~\ref{phase-III-high-probability}, for all $t\in[0:T_{3}-1]$,
    \begin{align}\label{main-recursion-sub-Gaussian}
    	\mathbb{E}\left[\left(\widehat{\beta}_{m}^{(t+1)}\right)^{2}\right]\leq \left(1-\eta_{m}^{(3,t)}\lambda_{m}^{\mathrm{III}}\left(1-\overline{\epsilon}\right)\right)\mathbb{E}\left[\left(\widehat{\beta}_{m}^{(t)}\right)^{2}\right]+\left(\eta_{m}^{(3,t)}\right)^{2}g_{m}^{(t)}+\left(\eta_{m}^{(3,t)}\right)^{3}\widetilde{g}_{m}^{(t)}, 
    \end{align}
    where $g_m^{(t)}=2\left(1+\beta_m^{(t)}\right)^2$ and  $\widetilde{g}_{m}^{(t)}=\frac{1}{\left(\eta_{m}^{(3,t)}\right)^{3}\left(T_{m}^{(3)}\right)^6} + \frac{33554432\left(\left(\lambda_{m}^{\mathrm{III}}\right)^{4}+\mathsf{c}_{1}^{2}d_{2m-1}^{2}d_{2m}^{2}\right)}{\lambda_{m}^{\mathrm{III}}(1-\overline{\epsilon})}$.
\end{lemma}
Lemmas~\ref{phase-III-error-recursion} and \ref{phase-III-error-recursion-sub-Gaussian} establish contraction-type recursions under Gaussian and sub-Gaussian noise, respectively, where the update consists of a contracting term together with higher-order and stochastic correction terms.

We now establish a contraction bound for the truncated alignment error
$\widehat{\beta}_{m}^{(t)}$, which serves as the main technical ingredient in Phase~III. Starting from the one-step recursion in Eq.~\eqref{main-recursion}, we decompose the evolution into a bias component, corresponding to the decay of the initialization,
and a variance component, capturing the accumulated stochastic fluctuations. Specifically, define two auxiliary sequences $\left\{B_{m}^{(t)}\right\}$ and $\left\{V_{m}^{(t)}\right\}$ by
\begin{align}
	V_{m}^{(t+1)} &= \left(1-\eta_{m}^{(3,t)}\lambda_{m}^{\mathrm{III}}(1-\overline{\epsilon})\right)V_{m}^{(t)}+\left(\eta_{m}^{(3,t)}\right)^{3}g_{m}^{(t)},\label{variance-square-hat-beta} \\
	B_{m}^{(t+1)} &= \left(1-\eta_{m}^{(3,t)}\lambda_{m}^{\mathrm{III}}(1-\overline{\epsilon})\right)B_{m}^{(t)}\label{bias-square-hat-beta}
\end{align}
with initialization $V_{m}^{(0)}=0$ and $B_{m}^{(0)}=\left(\widehat{\beta}_{m}^{(0)}\right)^{2}$, where $g_{m}^{(t)}$ follows the definition in Eq.~\eqref{main-recursion}. Therefore, we can obtain
\begin{align}\label{main-error}
	\mathbb{E}\left[ \left(\widehat{\beta}_{m}^ {\left(T_{m}^{(3)}\right)}\right)^{2} \right] \leq V_{m}^{\left(T_{m}^{(3)}\right)} + B_{m}^{\left(T_{m}^{(3)}\right)}.
\end{align}

We begin by controlling the bias term, which captures the deterministic contraction of the initial misalignment.
\begin{lemma}\label{esti-B}
	Under Assumption \ref{assumption-Gaussian} and the setting of Lemma~\ref{phase-III-high-probability}, we have
	\begin{align}
		B_{m}^{\left(T_{m}^{(3)}\right)}\leq \left(1-\eta_{m}^{(3)}\lambda_{m}^{\mathrm{III}}(1-\overline{\epsilon})\right)^{\widehat{T}_{m}^{(3)}}B_{m}^{(0)}.
	\end{align}
\end{lemma}

We next bound the variance term, which accounts for the accumulated stochastic fluctuations over Phase~III.
\begin{lemma}\label{esti-V}
	Under Assumption \ref{assumption-Gaussian} and the setting of Lemma~\ref{phase-III-high-probability}, we have
	\begin{align}
		V_{m}^{\left(T_{m}^{(3)}\right)}\leq
        \frac{8\eta_{m}^{(3)}\left\lceil\log \left(T_{m}^{(3)}\right)\right\rceil}{\left(\lambda_{m}^{\mathrm{III}}\right)^{2}(1-\overline{\epsilon})^{2}\left(T_{m}^{(3)}\right)^{4}}+\frac{268435456\left(\left(\lambda_{m}^{\mathrm{III}}\right)^{4}+\mathsf{c}_{1}^{2}d_{2m-1}^{2}d_{2m}^{2}\right)\eta_{m}^{(3)}\left\lceil\log \left(T_{m}^{(3)}\right)\right\rceil}{\left(\lambda_{m}^{\mathrm{III}}\right)^{3}(1-\overline{\epsilon})^{3}T_{m}^{(3)}}.
	\end{align}
\end{lemma}

Combining Lemma~\ref{esti-B} and Lemma~\ref{esti-V} with
Eq.~\eqref{main-error}, we obtain the following bound on the truncated alignment error.
\begin{lemma}\label{converge-phase-III}
	Under Assumption \ref{assumption-Gaussian} and the setting of Lemma~\ref{phase-III-high-probability}, we have
	\begin{align}
		\mathbb{E}\left[ \left(\widehat{\beta}_{m}^{\left(T_{m}^{(3)}\right)}\right)^{2} \right] &\leq 
        \left(1-\frac{\eta_{m}\lambda_{m}^{\mathrm{III}}(1-\overline{\epsilon})}{2}\right)^{\widehat{T}_{m}^{(3)}}\overline{\epsilon}^{2} + \frac{4\eta_{m}\left\lceil\log \left(T_{m}^{(3)}\right)\right\rceil}{\left(\lambda_{m}^{\mathrm{III}}\right)^{2}(1-\overline{\epsilon})^{2}\left(T_{m}^{(3)}\right)^{4}} \notag \\
        &\phantom{\mathrel{=}}+ \frac{134217728\left(\left(\lambda_{m}^{\mathrm{III}}\right)^{4} + \mathsf{c}_{1}^{2}d_{2m-1}^{2}d_{2m}^{2}\right)\eta_{m}\left\lceil\log \left(T_{m}^{(3)}\right)\right\rceil}{\left(\lambda_{m}^{\mathrm{III}}\right)^{3}(1-\overline{\epsilon})^{3}T_{m}^{(3)}}.\notag
	\end{align}
\end{lemma}
Lemma~\ref{converge-phase-III} provides an expectation bound on the truncated alignment error $\widehat{\beta}_{m}^{\left(T_{m}^{(3)}\right)}$. Using the fact that the truncation remains inactive with high probability, this bound can be transferred to the original alignment parameter $\alpha_{m}^{\left(T_{m}\right)}$.

We next extend this result to the sub-Gaussian setting.
\begin{lemma}\label{converge-phase-III-sub-Gaussian}
	Under the setting of Lemma~\ref{phase-III-high-probability}, we have
	\begin{align}
		\mathbb{E}\left[ \left(\widehat{\beta}_{m}^{\left(T_{m}^{(3)}\right)}\right)^{2} \right] &\leq 
		\left(1-\frac{\eta_{m}\lambda_{m}^{\mathrm{III}}(1-\overline{\epsilon})}{2}\right)^{\widehat{T}_{m}^{(3)}}\overline{\epsilon}^{2} + \frac{4\eta_{m}\left\lceil\log \left(T_{m}^{(3)}\right)\right\rceil}{\left(\lambda_{m}^{\mathrm{III}}\right)^{2}(1-\overline{\epsilon})^{2}\left(T_{m}^{(3)}\right)^{4}} \notag \\
		&\phantom{\mathrel{=}} + \frac{134217728\left(\left(\lambda_{m}^{\mathrm{III}}\right)^{4}+\mathsf{c}_{1}^{2}d_{2m-1}^{2}d_{2m}^{2}\right)\eta_{m}\left\lceil\log \left(T_{m}^{(3)}\right)\right\rceil}{\left(\lambda_{m}^{\mathrm{III}}\right)^{3}(1-\overline{\epsilon})^{3}T_{m}^{(3)}}+\frac{64\left\lceil\log \left(T_{m}^{(3)}\right)\right\rceil}{\left(\lambda_{m}^{\mathrm{III}}\right)^2(1-\frac{3}{2}\epsilon)^2T_m^{(3)}}.\notag
	\end{align}
\end{lemma}

\subsection{Proof Details}
\begin{proof}[Proof of Lemma~\ref{phase-III-high-probability}]
    Under the setting of $\eta_{m}^{(3)} \geq \eta_{m}^{(3,t)}$, we have
    \begin{gather}\label{eta-III-constraint}
        \eta_{m}^{(3,t)}\left|\mathbb{E}_{t}\left[\left\langle E_{m}^{(t+1)},v_{2m-1}^{*}v_{2m}^{*\top}\right\rangle\right]\right| + \eta_{m}^{(3,t)}\frac{\alpha_{m}^{(t)}}{\left\|W_{m}^{(t)}\right\|_{\mathrm{F}}}\cdot\left|\mathbb{E}_{t}\left[\left\langle E_{m}^{(t+1)},W_{m}^{(t)}\right\rangle\right]\right| \leq 2\sqrt{\mathsf{c}_{1}}\eta_{m} \leq \frac{\widehat{\epsilon}}{4}, \notag \\
        \frac{\left(\eta_{m}^{(3,t)}\right)^{2}}{2}\left|\Psi_{1}\left(W_{m}^{(t)}, G_{m}^{(t)},v_{2m-1}^{*},v_{2m}^{*},\overline{\eta}_{m}^{(t)}\right)\right| \leq \frac{\eta_{m}^{2}}{2}\times 128\mathsf{c}_{1}\left(\left(\lambda_{m}^{\mathrm{III}}+2\right)^{2}+(d_{2m-1}d_{2m}+4)\right) \notag \\
        \leq 1024\eta_{m}^{2}\mathsf{c}_{1} \left(\left(\lambda_{m}^{\mathrm{III}}\right)^{2}+d_{2m-1}d_{2m}\right) \leq \frac{\widehat{\epsilon}}{2},
    \end{gather}
    which implies $\alpha_{m}^{(t)}-\alpha_{m}^{(t+1)} \leq \frac{3}{4}\widehat{\epsilon}$ combining with Eq.~\eqref{all-iteration-update-alpha}.
    
    We now define the stopping time $\widehat{\tau}_{3}$ as the first time after $T_{m}^{(1)}+T_{m}^{(2)}$ that $\alpha_{m}^{(t)}$ downcrosses the level $1-\overline{\epsilon}$:
	\begin{align*}
		\widehat{\tau}_{3}=\inf_{t\geq 0}\left\{t:\alpha_{m}^{\left(T_{m}^{(1)}+T_{m}^{(2)}+t\right)}<1-\overline{\epsilon}\right\}.
	\end{align*}
    Suppose that $\widehat{\tau}_{3}=t_{2}$ for some $t_{2} \in \left[0,T_{m}-T_{m}^{(1)}-T_{m}^{(2)}\right]$. Since $\alpha_{m}^{(t)}$ can decrease by at most $\tfrac{3}{4}\widehat{\epsilon}$ in one iteration, the trajectory cannot jump across the threshold in a single step. Consequently, on the event $\bigcap_{t=T_{m}^{(1)}+T_{m}^{(2)}}^{T_{m}} \widetilde{\mathcal{E}}\left(\alpha_{m}^{(t)}\right)$, there must exist some $t_{1}<t_{2}\in \left[T_{m}-T_{m}^{(1)}-T_{m}^{(2)}\right]$ such that the event
	\begin{align*}
		\mathcal{E}_{t_{1}}^{\widehat{\tau}_{3}=t_{2}} &\coloneqq
        \Biggl\{ \alpha_{m}^{\left(T_{m}^{(1)}+T_{m}^{(2)}+t_{1}\right)}\geq 1-\widehat{\epsilon} \notag \\
        &\phantom{\mathrel{\coloneqq}}\bigcap \alpha_{m}^{\left(T_{m}^{(1)}+T_{m}^{(2)}+t_{1}:T_{m}^{(1)}+T_{m}^{(2)}+t_{2}\right)}\in\left[1-\overline{\epsilon},1-\frac{\widehat{\epsilon}}{4}\right] \notag \\
        &\phantom{\mathrel{\coloneqq}}\bigcap \alpha_{m}^{\left(T_{m}^{(1)}+T_{m}^{(2)}+t_{2}\right)}<1-\overline{\epsilon} \Biggr\}
	\end{align*}
    occurs.
    Consider the dynamics of $\alpha_{m}^{(t)}$ provided by Eq.~\eqref{all-iteration-update-alpha}, we have
	\begin{align}
		1-\alpha_{m}^{(t+1)}&\leq 1-\alpha_{m}^{(t)}-\eta_{m}^{(3,t)}\lambda_{m}^{\mathrm{III}}\cdot\left(1+\alpha_{m}^{(t)}\right)\cdot\left(1-\alpha_{m}^{(t)}\right) \notag \\
        &\phantom{\mathrel{=}} -\eta_{m}^{(3,t)}\left\langle E_{m}^{(t+1)}, v_{2m-1}^{*}v_{2m}^{*\top}-\frac{\alpha_{m}^{(t)}}{\left\|W_{m}^{(t)}\right\|_{\mathrm{F}}}W_{m}^{(t)}\right\rangle \notag \\
        &\phantom{\mathrel{=}} -\frac{\left(\eta_{m}^{(3,t)}\right)^{2}}{2}\Psi_{1}\left(W_{m}^{(t)}, G_{m}^{(t)},v_{2m-1}^{*},v_{2m}^{*},\overline{\eta}_{m}^{(t)}\right) \notag \\
        &\leq \left(1-\eta_{m}^{(3,t)}\lambda_{m}^{\mathrm{III}}(1-\overline{\epsilon})\right)\left(1-\alpha_{m}^{(t)}\right)-\eta_{m}^{(3,t)}\xi_{m}^{(t+1)} - \eta_{m}^{(3,t)}\lambda_{m}^{\mathrm{III}}\left(1-\alpha_{m}^{(t)}\right) \notag \\
        &\phantom{\mathrel{=}} -\eta_{m}^{(3,t)}\mathbb{E}_{t}\left[\left\langle E_{m}^{(t+1)}, v_{2m-1}^{*}v_{2m}^{*\top}-\frac{\alpha_{m}^{(t)}}{\left\|W_{m}^{(t)}\right\|_{\mathrm{F}}}W_{m}^{(t)}\right\rangle\right] \notag \\
        &\phantom{\mathrel{=}} -\frac{\left(\eta_{m}^{(3,t)}\right)^{2}}{2}\Psi_{1}\left(W_{m}^{(t)}, G_{m}^{(t)},v_{2m-1}^{*},v_{2m}^{*},\overline{\eta}_{m}^{(t)}\right) \notag \\
        &\leq \left(1-\eta_{m}^{(3,t)}\lambda_{m}^{\mathrm{III}}(1-\overline{\epsilon})\right)\left(1-\alpha_{m}^{(t)}\right)-\eta_{m}^{(3,t)}\xi_{m}^{(t+1)}
	\end{align}
	for any $t\geq T_{m}^{(1)}+T_{m}^{(2)}$, where $\xi_{m}^{(t+1)}$ has the following form:
    \begin{align*}
		\widehat{\xi}_{m}^{(t+1)}&\coloneqq\left\langle E_{m}^{(t+1)}, v_{2m-1}^{*}v_{2m}^{*\top}-\frac{\alpha_{m}^{(t)}}{\left\|W_{m}^{(t)}\right\|_{\mathrm{F}}}W_{m}^{(t)}\right\rangle\notag \\
		&\phantom{\mathrel{=}}-\mathbb{E}_{t}\left[\left\langle E_{m}^{(t+1)}, v_{2m-1}^{*}v_{2m}^{*\top}-\frac{\alpha_{m}^{(t)}}{\left\|W_{m}^{(t)}\right\|_{\mathrm{F}}}W_{m}^{(t)}\right\rangle\right],
	\end{align*}
	and the last inequality is derived from combining the result of Lemma~\ref{estimation} with the setting of $\eta_{m}^{(3)} \geq \eta_{m}^{(3,t)}$, which implies
    \begin{align}
        \frac{\eta_{m}^{(3,t)}\lambda_{m}^{\mathrm{III}}\left(1-\alpha_{m}^{(t)}\right)}{2} \geq \frac{\eta_{m}^{(3,t)}\lambda_{m}^{\mathrm{III}}\widehat{\epsilon}}{8} &\geq \eta_{m}^{(3,t)}\left|\mathbb{E}_{t}\left[\left\langle E_{m}^{(t+1)},v_{2m-1}^{*}v_{2m}^{*\top}\right\rangle\right]\right| \notag \\
        &\phantom{\mathrel{=}}+ \eta_{m}^{(3,t)}\frac{\alpha_{m}^{(t)}}{\left\|W_{m}^{(t)}\right\|_{\mathrm{F}}}\cdot\left|\mathbb{E}_{t}\left[\left\langle E_{m}^{(t+1)},W_{m}^{(t)}\right\rangle\right]\right|, \\
        \frac{\eta_{m}^{(3,t)}\lambda_{m}^{\mathrm{III}}\left(1-\alpha_{m}^{(t)}\right)}{2} \geq \frac{\eta_{m}^{(3,t)}\lambda_{m}^{\mathrm{III}}\widehat{\epsilon}}{8} &\geq \frac{\left(\eta_{m}^{(3,t)}\right)^{2}}{2}\Psi_{1}\left(W_{m}^{(t)}, G_{m}^{(t)},v_{2m-1}^{*},v_{2m}^{*},\overline{\eta}_{m}^{(t)}\right).
    \end{align}
    Since $\xi_{m}^{(t+1)}$ is sub-Gaussian with parameter $4\sqrt{\mathsf{c}_{1}}$, we bound the probability of the bad event with
	\begin{align}\label{PCA-lower-bound-phase-III}
		\mathbb{P}\left(\mathcal{E}_{t_{1}}^{\widehat{\tau}_{3}=t_{2}}\right) \leq \exp\left\{-\frac{\lambda_{m}^{\mathrm{III}}(1-\overline{\epsilon})\widehat{\epsilon}^{2}}{64\eta_{m}^{(3)}\mathsf{c}_{1}}\right\}
	\end{align}
	by Lemma~\ref{aux-martingale-concentration} and Lemma~\ref{aux-3}. Therefore, we have
	\begin{align}\label{E-III-constraint}
		\mathbb{P}\left(\bigcap_{t=T_{m}^{(1)}+T_{m}^{(2)}}^{T_{m}} \widetilde{\mathcal{E}}\left(\alpha_{m}^{(t)}\right)\right) &\geq 1-\sum_{0\leq t_{1}<t_{2}\leq T_{m}-T_{m}^{(1)}-T_{m}^{(2)}}\mathbb{P}\left(\mathcal{E}_{t_{1}}^{\widehat{\tau}_{3}=t_{2}}\right) \notag \\
        &\geq 1-\frac{T_{m}^{2}}{2}\exp\left\{-\frac{\lambda_{m}^{\mathrm{III}}(1-\overline{\epsilon})\widehat{\epsilon}^{2}}{64\eta_{m}^{(3)}\mathsf{c}_{1}}\right\} \geq 1-\frac{\delta''}{2},
	\end{align}
	where the last inequality follows from the setting of $\eta_{m}$.
\end{proof}

\begin{proof}[Proof of Lemma~\ref{phase-III-error-recursion}]
    According to the recursive dynamic of $\left(\widehat{\beta}_{m}^{(t)}\right)^{2}$ in Eq.~\eqref{non-expect-recursion}, we obtain
    \begin{align*}
        \mathbb{E}_{t}\left[\left(\widehat{\beta}_{m}^{(t+1)}\right)^{2}\right] &\leq
        \left(1-\eta_{m}^{(3,t)} f_{\beta_{m}^{(t)}}\right)^{2}\left(\widehat{\beta}_{m}^{(t)}\right)^{2}+2\eta_{m}^{(3,t)}\left(1-\eta_{m}^{(3,t)} f_{\beta_{m}^{(t)}}\right)\widehat{\beta}_{m}^{(t)}\mathbb{E}_{t}\left[\overline{g}_{m}^{(t)}\right] \\
		&\phantom{\mathrel{=}} +2\left(\eta_{m}^{(3,t)}\right)^{2}\left(1-\eta_{m}^{(3,t)} f_{\beta_{m}^{(t)}}\right)\widehat{\beta}_{m}^{(t)}\mathbb{E}_{t}\left[\widetilde{g}_{m}^{(t)}\right] \\ 
        &\phantom{\mathrel{=}} +2\left(\left(\eta_{m}^{(3,t)}\right)^{2}\mathbb{E}_{t}\left[\left(\overline{g}_{m}^{(t)}\right)^{2}\right]+\left(\eta_{m}^{(3,t)}\right)^{4}\mathbb{E}_{t}\left[\left(\widetilde{g}_{m}^{(t)}\right)^{2}\right]\right)
    \end{align*}
    Applying the Cauchy-Schwarz inequality $2ab\leq a^{2}+b^{2}$ and the bound
    $\bigl(\mathbb{E}[X]\bigr)^{2}\leq \mathbb{E}[X^{2}]$, we further have
    \begin{align*}
        \mathbb{E}_{t}\left[\left(\widehat{\beta}_{m}^{(t+1)}\right)^{2}\right] &\leq \left(1-\eta_{m}^{(3,t)} f_{\beta_{m}^{(t)}}\right)^{2}\left(\widehat{\beta}_{m}^{(t)}\right)^{2}+\left(\eta_{m}^{(3,t)}\right)^{2}\left(1-\eta_{m}^{(3,t)} f_{\beta_{m}^{(t)}}\right)^{2}\left(\widehat{\beta}_{m}^{(t)}\right)^{2}+\left(\mathbb{E}_{t}\left[\overline{g}_{m}^{(t)}\right]\right)^{2} \\
		&\phantom{\mathrel{=}} +\eta_{m}^{(3,t)}\lambda_{m}^{\mathrm{III}}(1-\overline{\epsilon})\left(1-\eta_{m}^{(3,t)} f_{\beta_{m}^{(t)}}\right)^{2}\left(\widehat{\beta}_{m}^{(t)}\right)^{2} \\
        &\phantom{\mathrel{=}} +4\left(\left(\eta_{m}^{(3,t)}\right)^{2}\mathbb{E}_{t}\left[\left(\overline{g}_{m}^{(t)}\right)^{2}\right] + \frac{2\left(\eta_{m}^{(3,t)}\right)^{3}}{\lambda_{m}^{\mathrm{III}}(1-\overline{\epsilon})}\mathbb{E}_{t}\left[\left(\widetilde{g}_{m}^{(t)}\right)^{2}\right]\right).
    \end{align*}
    Next, using the definition $f_{\beta_{m}^{(t)}}=\lambda_{m}^{\mathrm{III}}\left(1+\beta_{m}^{(t)}\right)$
    and the fact that $\left|\mathbb{E}_{t}\left[\overline{g}_{m}^{(t)}\right]\right|
    \leq \left(T_{m}^{(3)}\right)^{-3}/2$ (by Lemma~\ref{control-expectation}), we obtain
    \begin{align*}
        \mathbb{E}_{t}\left[\left(\widehat{\beta}_{m}^{(t+1)}\right)^{2}\right]
        &\leq \left(1-\frac{5\eta_{m}^{(3,t)}}{2}\lambda_{m}^{\mathrm{III}}(1-\overline{\epsilon})\right)\left(\widehat{\beta}_{m}^{(t)}\right)^{2}+\frac{1}{2T_{m}^6} \\
        &\phantom{\mathrel{=}} +4\left(\left(\eta_{m}^{(3,t)}\right)^{2}\mathbb{E}_{t}\left[\left(\overline{g}_{m}^{(t)}\right)^{2}\right]+\frac{2\left(\eta_{m}^{(3,t)}\right)^{3}}{\lambda_{m}^{\mathrm{III}}(1-\overline{\epsilon})}\mathbb{E}_{t}\left[\left(\widetilde{g}_{m}^{(t)}\right)^{2}\right]\right).
    \end{align*}
    To control the variance term, we invoke the following second-moment bound:
    \begin{align}\label{Gaussian-Sub-Gaussian-Gap}
        \mathbb{E}_t\left[\left(\overline{g}_{m}^{(t)}\right)^{2}\right]
        &\leq \mathbb{E}_t\left[\left\langle{E}_{m}^{(t+1)}, v_{2m-1}^{*}v_{2m}^{*\top} - \beta_{m}^{(t)}\frac{O_{m}^{(t)}}{\left\|O_{m}^{(t)}\right\|_{\mathrm{F}}}\right\rangle^{2}\right] \notag \\
		&= \mathbb{E}_t\left[\left\langle E_{m}^{(t+1)},v_{2m-1}^{*}v_{2m}^{*\top}\right\rangle^{2}\right] - \frac{2\beta_{m}^{(t)}}{\left\|{O_{m}^{(t)}}\right\|_{\mathrm{F}}}\mathbb{E}_t\left[\left\langle E_{m}^{(t+1)},v_{2m-1}^{*}v_{2m}^{*\top}\right\rangle\left\langle E_{m}^{(t+1)},{O_{m}^{(t)}}\right\rangle\right] \notag \\
		&\phantom{\mathrel{=}} + \frac{\left(\beta_{m}^{(t)}\right)^{2}}{\left\|{O_{m}^{(t)}}\right\|_{\mathrm{F}}^{2}}\mathbb{E}_t\left[\left\langle E_{m}^{(t+1)},{O_{m}^{(t)}}\right\rangle^{2}\right] \notag 
		\\
		&=\left(1-\left[\beta_{m}^{(t)}\right]^2\right),
    \end{align}
    where the random matrix $E_{m}^{(t+1)}$ is defined explicitly in Eq.~\eqref{eq:errort} and excludes the truncation function $\mathds{1}_{\mathcal{A}^{(t+1)}(\delta'')}$, and the above inequality follows from Assumption~\ref{assumption-Gaussian}.
    Substituting this bound back yields
    \begin{align}
        \mathbb{E}_{t}\left[\left(\widehat{\beta}_{m}^{(t+1)}\right)^{2}\right]
        &\leq \left(1-\frac{5\eta_{m}^{(3,t)}}{2}\lambda_{m}^{\mathrm{III}}(1-\overline{\epsilon})\right)\left(\widehat{\beta}_{m}^{(t)}\right)^{2}+\frac{1}{\left(T_{m}^{(3)}\right)^6} \notag \\
        &\phantom{\mathrel{=}} + 4\left(\frac{\left(\eta_{m}^{(3,t)}\right)^{2}}{2}\widehat{\beta}_{m}^{(t)}\left(1+\beta_{m}^{(t)}\right)+\frac{2\left(\eta_{m}^{(3,t)}\right)^{3}}{\lambda_{m}^{\mathrm{III}}(1-\overline{\epsilon})}\mathbb{E}_{t}\left[\left(\widetilde{g}_{m}^{(t)}\right)^{2}\right]\right) \notag \\
        &\leq \left(1-\eta_{m}^{(3,t)}\lambda_{m}^{\mathrm{III}}(1-\overline{\epsilon})\right)\left(\widehat{\beta}_{m}^{(t)}\right)^{2}+\frac{1}{\left(T_{m}^{(3)}\right)^6} \notag \\
        &\phantom{\mathrel{=}}+ \frac{8\left(\eta_{m}^{(3,t)}\right)^{3}}{\lambda_{m}^{\mathrm{III}}(1-\overline{\epsilon})}\cdot\left(\left(1+\beta_{m}^{(t)}\right)^{2}+\mathbb{E}_{t}\left[\left(\widetilde{g}_{m}^{(t)}\right)^{2}\right]\right). \notag
    \end{align}
    Finally, using $\beta_{m}^{(t)}\leq 1$ and applying Eq.~\eqref{esti-Psi1}
    to $\widetilde{g}_{m}^{(t)}$, which implies
    \begin{align*}
        \Psi_{1}\left(W_{m}^{(t)}, G_{m}^{(t)}, v_{2m-1}^{*}, v_{2m}^{*}, \overline{\eta}_{m}\right) \leq 2048\mathsf{c}_1\left(\left(\lambda_{m}^{\mathrm{III}}\right)^{2} + d_{2m-1}d_{2m}\right),
    \end{align*}
    we arrive at
    \begin{align}
        \mathbb{E}_{t}\left[\left(\widehat{\beta}_{m}^{(t+1)}\right)^{2}\right]
        &\leq \left(1-\eta_{m}^{(3,t)}\lambda_{m}^{\mathrm{III}}(1-\overline{\epsilon})\right)\left(\widehat{\beta}_{m}^{(t)}\right)^{2}+\frac{1}{\left(T_{m}^{(3)}\right)^6}\notag 
        \\
        &\phantom{\mathrel{=}} + \frac{33554432\left(\eta_{m}^{(3,t)}\right)^{3}}{\lambda_{m}^{\mathrm{III}} (1-\overline{\epsilon})}\cdot\left( \left(\lambda_{m}^{\mathrm{III}}\right)^{4} + \mathsf{c}_{1}^{2}d_{2m-1}^{2}d_{2m}^{2}\right).
    \end{align}
    Taking expectation on both sides directly yields Eq.~\eqref{main-recursion}, which completes the proof.
\end{proof}
\begin{proof}[Proof of Lemma~\ref{phase-III-error-recursion-sub-Gaussian}]
    Considering only Assumptions \ref{ass-pro} and \ref{ass-base}, Eq.~\eqref{Gaussian-Sub-Gaussian-Gap} reduces to
	\begin{align*}
        \mathbb{E}_t\left[\left(\overline{g}_{m}^{(t)}\right)^{2}\right] \leq \left(1+\beta_{m}^{(t)}\right)^{2}.
	\end{align*}
    Following the same proof techniques as in Lemma~\ref{phase-III-error-recursion}, we complete the proof.
\end{proof}

\begin{proof}[Proof of Lemma~\ref{esti-B}]
	According to the recursion provided by Eq.~\eqref{bias-square-hat-beta}, we can directly derive
	\begin{align}
		B_{m}^{\left(T_{m}^{(3)}\right)}=&\prod_{l=0}^{L-1}\left(1-\frac{\eta_{m}^{(3)}}{2^l}\lambda_{m}^{\mathrm{III}}(1-\overline{\epsilon})\right)^{\widehat{T}_{m}^{(3)}}B_{m}^{(0)} \leq \left(1-\eta_{m}^{(3)}\lambda_{m}^{\mathrm{III}}(1-\overline{\epsilon})\right)^{\widehat{T}_{m}^{(3)}}B_{m}^{(0)}.\notag
	\end{align}
\end{proof}
\begin{proof}[Proof of Lemma~\ref{esti-V}]
	According to the recursion provided by Eq.~\eqref{variance-square-hat-beta}, we can directly derive
	\begin{align}\label{main-variance}
		V_{m}^{\left(T_{m}^{(3)}\right)} = \sum_{t=0}^{T_{m}^{(3)}}\left(\eta_{m}^{(3,t)}\right)^{3}\prod_{i=t+1}^{T_{m}^{(3)}}\left(1-\eta_{m}^{(3,i)}\lambda_{m}^{\mathrm{III}}(1-\overline{\epsilon})\right)g_{m}^{(t)}.
	\end{align}
	Based on the update rule for $\eta_{t}$ defined in Lemma~\ref{esti-V}, we have
	\begin{align}\label{eq-variance-I}
    		\phantom{\mathrel{=}} V_{m}^{\left(T_{m}^{(3)}\right)} &\leq \sum_{t=0}^{T_{m}^{(3)}}\left(\eta_{m}^{(3,t)}\right)^{3}\prod_{i=t+1}^{T_{m}^{(3)}}\left(1-\eta_{m}^{(3,i)}\lambda_{m}^{\mathrm{III}}(1-\overline{\epsilon})\right)\cdot\underbrace{\left(\frac{1}{\left(T_{m}^{(3)}\right)^{3}}+\frac{33554432\left(\left(\lambda_{m}^{\mathrm{III}}\right)^{4}+\mathsf{c}_{1}^{2}d_{2m-1}^{2}d_{2m}^{2}\right)}{\lambda_{m}^{\mathrm{III}}(1-\overline{\epsilon})}\right)}_{g}\notag 
            \\
    		&\leq \eta_{m}^{(3)}\cdot\sum_{l=0}^{L-1}\left[\frac{\eta_{m}^{(3)}}{2^l}\right]^{2}\sum_{i=1}^{\widehat{T}_{m}^{(3)}}\left(1-\frac{\eta_{m}^{(3)}}{2^l}\lambda_{m}^{\mathrm{III}}(1-\overline{\epsilon})\right)^{\widehat{T}_{m}^{(3)}-i}\prod_{j=l+1}^{L-1}\left(1-\frac{\eta_{m}^{(3)}}{2^j}\lambda_{m}^{\mathrm{III}}(1-\overline{\epsilon})\right)^{\widehat{T}_{m}^{(3)}}g\notag 
            \\
    		&\leq \eta_{m}^{(3)}\cdot\left[\left[\eta_{m}^{(3)}\right]^{2}\sum_{i=1}^{\widehat{T}_{m}^{(3)}}\left(1-\eta_{m}^{(3)}\lambda_{m}^{\mathrm{III}}(1-\overline{\epsilon})\right)^{\widehat{T}_{m}^{(3)}-i}\prod_{j=1}^{L-1}\left(1-\frac{\eta_{m}^{(3)}}{2^j}\lambda_{m}^{\mathrm{III}}(1-\overline{\epsilon})\right)^{\widehat{T}_{m}^{(3)}}g\right.\notag 
            \\
    		&\phantom{\mathrel{=}}\left.+\sum_{l=1}^{L-1}\left[\frac{\eta_{m}^{(3)}}{2^l}\right]^{2}\sum_{i=1}^{\widehat{T}_{m}^{(3)}}\left(1-\frac{\eta_{m}^{(3)}}{2^l}\lambda_{m}^{\mathrm{III}}(1-\overline{\epsilon})\right)^{\widehat{T}_{m}^{(3)}-i}\prod_{j=l+1}^{L-1}\left(1-\frac{\eta_{m}^{(3)}}{2^j}\lambda_{m}^{\mathrm{III}}(1-\overline{\epsilon})\right)^{\widehat{T}_{m}^{(3)}}g\right]\notag 
            \\
    		&\leq \frac{\eta_{m}^{(3)} g}{\lambda_{m}^{\mathrm{III}}(1-\overline{\epsilon})}\cdot\left[\eta_{m}^{(3)}\left(1-\left(1-\eta_{m}^{(3)}\lambda_{m}^{\mathrm{III}}(1-\overline{\epsilon})\right)^{\widehat{T}_{m}^{(3)}}\right)\prod_{j=1}^{L-1}\left(1-\frac{\eta_{m}^{(3)}}{2^j}\lambda_{m}^{\mathrm{III}}(1-\overline{\epsilon})\right)^{\widehat{T}_{m}^{(3)}}\right.\notag
            \\
    		&\phantom{\mathrel{=}}\left.+\sum_{l=1}^{L-1}\frac{\eta_{m}^{(3)}}{2^l}\left(1-\left(1-\frac{\eta_{m}^{(3)}}{2^l}\lambda_{m}^{\mathrm{III}}(1-\overline{\epsilon})\right)^{\widehat{T}_{m}^{(3)}}\right)\prod_{j=l+1}^{L-1}\left(1-\frac{\eta_{m}^{(3)}}{2^j}\lambda_{m}^{\mathrm{III}}(1-\overline{\epsilon})\right)^{\widehat{T}_{m}^{(3)}}\right].
	\end{align}
	Then, we define the following scalar function
	\begin{align*}
		f(x)&\coloneqq x\left(1-(1-x)^{h+\widehat{T}_{m}^{(3)}}\right)\prod_{j=1}^{L-1}\left(1-\frac{x}{2^j}\right)^{\widehat{T}_{m}^{(3)}} \\
        &\phantom{\mathrel{=}}+ \sum_{l=1}^{L-1}\frac{x}{2^l}\left(1-\left(1-\frac{x}{2^l}\right)^{\widehat{T}_{m}^{(3)}}\right)\prod_{j=l+1}^{L-1}\left(1-\frac{x}{2^j}\right)^{\widehat{T}_{m}^{(3)}},
	\end{align*}
	as similar as that in [Theorem~C.2, \cite{wu2022last}]. Moreover, the following inequality can be directly derived
	\begin{align}\label{aux-scalar-func}
		f\left(\eta_{m}^{(3)}\lambda_{m}^{\mathrm{III}}(1-\overline{\epsilon})\right)\leq\frac{8}{\widehat{T}_{m}^{(3)}},
	\end{align}
	by [Lemma~C.3, \cite{wu2022last}]. Applying Eq.~\eqref{aux-scalar-func} to Eq.~\eqref{eq-variance-I} and combining Eq.~\eqref{main-variance}, we obtain
	\begin{align*}
		V_{m}^{\left(T_{m}^{(3)}\right)}\leq\frac{8\eta_{m}^{(3)}g}{\widehat{T}_{m}^{(3)}\left(\lambda_{m}^{\mathrm{III}}\right)^{2}(1-\overline{\epsilon})^{2}}.
	\end{align*}
\end{proof}
\begin{proof}[Proof of Lemma~\ref{converge-phase-III}]
    Combining Eq.~\eqref{main-error} with Lemmas~\ref{esti-B} and \ref{esti-V}, Lemma~\ref{converge-phase-III} follows.
\end{proof}
\begin{proof}[Proof of Lemma~\ref{converge-phase-III-sub-Gaussian}]
	Utilizing Lemma~\ref{phase-III-error-recursion-sub-Gaussian} and following the same proof techniques as in Lemmas~\ref{esti-B} and \ref{esti-V}, we obtain estimates for $B_{m}^{\left(T_{m}^{(3)}\right)}$ and $V_{m}^{\left(T_{m}^{(3)}\right)}$, respectively. This completes the proof of Lemma~\ref{converge-phase-III-sub-Gaussian}.
\end{proof}
\begin{proof}[Proof of Theorem~\ref{thm-phase-III-tensor-PCA}]
    According to the result of Lemma~\ref{converge-phase-III}, we have
    \begin{align}
    	\left(1-\beta_{m}^{\left(T_{m}^{(3)}\right)}\right)^{2} &\lesssim 
        \left(1-\eta_{m}\lambda_{m}^{\mathrm{III}}(1-\overline{\epsilon})\right)^{\widehat{T}_{m}^{(3)}}\frac{k\overline{\epsilon}^{2}}{\delta''} \notag \\ 
        &\phantom{\mathrel{=}}+  \frac{k\left(\left(\lambda_{m}^{\mathrm{III}}\right)^{4}+\mathsf{c}_{1}^{2}d_{2m-1}^{2}d_{2m}^{2}\right)\eta_{m}\left\lceil\log \left(T_{m}^{(3)}\right)\right\rceil}{\left(\lambda_{m}^{\mathrm{III}}\right)^{3}(1-\overline{\epsilon})^{3}\delta'' T_{m}^{(3)}} \label{final-error-bound}
    \end{align}
    with probability at least $1-\delta''/(2k)$ for each $m\in[k]$, utilizing Markov inequality. The proof is completed by combining the error bound Eq.~\eqref{final-error-bound} with the fact that the event $\left\{\alpha_{m}^{\left(T_{m}\right)}=\beta_{m}^{\left(T_{m}^{(3)}\right)}\right\}$ occurs with probability at least $1-\delta''/(2k)$, i.e.,
    \begin{align*}
        \mathbb{P}\left(\left\{\alpha_{m}^{\left(T_{m}\right)}=\beta_{m}^{\left(T_{m}^{(3)}\right)}\right\}\right)\geq\mathbb{P}\left(\bigcap_{t=1}^{T_{m}^{(3)}}\left\{\alpha_{m}^{\left(T_{m}^{(1)}+T_{m}^{(2)}+t\right)}=\beta_{m}^{(t)}\right\}\right) \geq 1-\frac{\delta''}{2k},
    \end{align*}
    where the last inequality is derived from Lemma~\ref{phase-III-high-probability} and the construction methodology of matrix sequence $\left\{O_{m}^{(t)}\right\}_{t=1}^{T_{m}^{(3)}}$.
\end{proof}
\begin{proof}[Proof of Theorem~\ref{thm-phase-III-tensor-PCA-sub-Gaussian}]
    Combining Lemma~\ref{converge-phase-III-sub-Gaussian} with the same proof techniques as in Theorem~\ref{thm-phase-III-tensor-PCA} completes the proof.
\end{proof}

\section{Proof for the Even Case}\label{proof-even}
\subsection{Proof of Theorem \ref{main-theorem}}\label{last-proof-main-thm}
\begin{proof}
	According to Proposition \ref{initial-prop}, under the initialization defined in Section \ref{sec-4.2}, we have 
	\begin{equation*}
        \left|\left\langle \overline{v}_{2m-1}^{(0)},v_{2m-1}^{*}\right\rangle\cdot\left\langle\overline{v}_{2m}^{(0)},v_{2m}^{*}\right\rangle\right|\geq\frac{\mathsf{c}_{0}}{\left(d_{2m-1}d_{2m}\right)^{1/2}},
	\end{equation*}
	for any $m\in[k]$ with probability at least $1-\delta$. As stipulated by the selection rules for the required sample complexity terms $\left\{T_{m,h}^{(1)}\right\}_{h\in[h^{*}]}$ and $\left\{T_{m}^{(l)}\right\}_{l\in\{2,3\}}$ in Eq.~\eqref{hyper-para-T-even}, and the corresponding step-sizes $\left\{\eta_{m,h}^{(1)}\right\}_{h\in[h^{*}]}$ and $\left\{\eta_{m}^{(l)}\right\}_{l\in\{2,3\}}$ in Eq.~\eqref{hyper-para-eta-even} for each $m\in[k]$, we can obtain Theorems \ref{thm-phase-I-tensor-PCA}, \ref{thm-phase-II-tensor-PCA} and \ref{thm-phase-III-tensor-PCA-sub-Gaussian} by setting $\delta''=\delta/8$, and $\widetilde{\epsilon}\asymp k^{-1}$, and $\widehat{\epsilon}\asymp 1$ ($\widehat{\epsilon}>2/3$). 
	Without loss of generality, we assume that $\left\langle \widehat{W}_{m},v_{2m-1}^{*}v_{2m}^{*\top}\right\rangle>0$ for any $m\in[k]$. One can notice that the last iterate of Algorithm \ref{MSANSGA} satisfies
	\begin{equation}
		\begin{aligned}
			\left(1-\left\langle \widehat{W}_{m},v_{2m-1}^{*}v_{2m}^{*\top}\right\rangle\right)^{2}\leq\widetilde{\calO}\left(\mathrm{err}_{m}^{2}\right),
		\end{aligned}
	\end{equation}
	for any $m\in[k]$ with probability at least $1-\delta$, where $\mathrm{err}_{m}^{2}$ is defined as follows,
	\begin{equation*}
		\mathrm{err}_{m}^{2} \coloneqq \left(1-\frac{\lambda\eta_{m}^{(3)}}{4e}\right)^{\left\lfloor T_{m}^{(3)}/\log \left(T_{m}^{(3)}\right)\right\rfloor}\frac{1}{k\delta} + \frac{k\eta_{m}^{(3)}}{\lambda^{2}\delta \left(T_{m}^{(3)}\right)^{4}} + \frac{k\left(\lambda^{4}+d_{2m-1}^{2}d_{2m}^{2}\right)\eta_{m}^{(3)}}{\lambda^{3}\delta T_{m}^{(3)}}+\frac{k}{\lambda^2\delta T_{m}^{(3)}},
	\end{equation*}
	by using Theorem \ref{thm-phase-III-tensor-PCA-sub-Gaussian} with $\lambda_{m}^{\mathrm{III}}\geq \lambda e^{-1}$.
	
	Consider the symmetric matrix $X_{m}=\widehat{W}_{m}^{\top}\widehat{W}_{m}$. Let $\lambda_{1} \geq \lambda_{2} \geq \cdots \geq \lambda_{d_{2m}}$ be its eigenvalues sorted in descending order, with corresponding eigenvectors $\{u_{i}\}_{i=1}^{d_{2m}}$. Then we have
	\begin{align}\label{esti-1}
		\lambda_{1}\geq\left\langle X,v_{2m}^{*}v_{2m}^{*\top}\right\rangle=\left\|\widehat{W}_{m}v_{2m}^{*}\right\|^{2}\geq\left\langle\widehat{W}_{m},v_{2m-1}^{*}v_{2m}^{*\top}\right\rangle^{2}\geq1-\widetilde{\mathcal{O}}\left(\mathrm{err}_{m}\right).
	\end{align}
	Moreover, since vector $v_{2m}^{*}$ can be written as the sum of two components: one that is parallel to $u_{1}$, and one that lies in $(u_{1})_{\perp}$, we write $v_{2m}^{*}$ as $v_{2m}^{*}=\sum_{i=1}^{d_{2m}}\alpha_{i}u_{i}$. Therefore, we can obtain
	\begin{align}\label{esti-2}
		\left\langle X, v_{2m}^{*}v_{2m}^{*\top}\right\rangle = \sum_{i=1}^{d_{2m}}\lambda_{i}\alpha_{i}^{2}.
	\end{align}
	Noticing that $\|X\|_{\mathrm{F}}\leq\left\|\widehat{W}_{m}\right\|_{\mathrm{F}}^{2}\leq 1$, we derive that $\sum_{i=2}^{d_{2m}}\lambda_{i}^{2}\leq1-\left(1-\widetilde{\mathcal{O}}\left(\mathrm{err}_{m}\right)\right)^{2}$. Eqs.~\eqref{esti-1} and \eqref{esti-2} implicate that
	\begin{align}\label{eq-1}
		\alpha_{1}^{2}\geq\frac{1-\widetilde{\mathcal{O}}\left(\mathrm{err}_{m}\right)-\sum_{i=2}^{d}\lambda_{i}\alpha_{i}^{2}}{\lambda_{1}}\overset{\text{(a)}}{\geq}\frac{1-\widetilde{\mathcal{O}}\left(\mathrm{err}_{m}\right)-\left(\sum_{i=2}^{d}\lambda_{i}^{2}\right)^{1/2}(1-\alpha_{1}^{2})^{1/2}}{\lambda_{1}},
	\end{align}
	where (a) is derived from the Cauchy-Schwarz inequality and the fact that $\sum_{i=2}^{d}\alpha_{i}^{4}\leq\sum_{i=2}^{d}\alpha_{i}^{2}=1-\alpha_{1}^{2}$. By Eq.~\eqref{eq-1}, we have
	\begin{align}\label{eq-2}
		1-\alpha_{1}^{2} \leq (1-\alpha_{1}^{2})^{1/2} \widetilde{\mathcal{O}}\left(\mathrm{err}_{m}^{1/2}\right)\Rightarrow 1-\alpha_{1}^{2}\leq\widetilde{\mathcal{O}}\left(\mathrm{err}_{m}\right).
	\end{align} 
	Considering $\|v_{2m}^{*}-u_{1}\|^{2}$ and $\|v_{2m}^{*}+u_{1}\|^{2}$, we have
	\begin{equation}\label{eq-3}
		\begin{split}
			\left\|v_{2m}^{*}-u_{1}\right\|^{2}&=(1-\alpha_{1})^{2}+\sum_{i=2}^{d}\alpha_{i}^{2}=(1-\alpha_{1})^{2}+(1-\alpha_{1}^{2}), \\
			\left\|v_{2m}^{*}+u_{1}\right\|^{2}&=(1+\alpha_{1})^{2}+\sum_{i=2}^{d}\alpha_{i}^{2}=(1+\alpha_{1})^{2}+(1-\alpha_{1}^{2}).
		\end{split}
	\end{equation}
	The combination of Eq.~\eqref{eq-2} and Eq.~\eqref{eq-3} implies that $\min\left\{\|v_{2m}^{*}-u_{1}\|^{2},\|v_{2m}^{*}+u_{1}\|^{2}\right\}\leq\widetilde{\mathcal{O}}(\mathrm{err})$. Consequently, the power method guarantees that $u_{\text{Alg}}$ converges linearly to $u_{1}$, achieving high-precision approximation within few iterations. Applying the same argument to $\widehat{W}_{m}\widehat{W}_{m}^{\top}$ and $v_{2m-1}^{*}$, we complete the proof.
\end{proof}
\begin{corollary}\label{main-result-gaussian-ass-even}
    Consider the asymmetric tensor PCA problem \ref{PCA-def} of even order $\overline{k} = 2k\geq 4$ under Assumptions \ref{ass-pro} and \ref{assumption-Gaussian}. 
	Given failure probability $\delta\in(0,1/2)$, let the total sample size be
	\linebreak $N = \text{\small$2^{\overline{k}}\left(N_{0} + \sum_{m=1}^{\overline{k}/2}\sum_{l=1}^{2}T_{m}^{(l)}\right)+2\sum_{m=1}^{\overline{k}/2}T_{m}^{(3)}$}$
	where $N_0\gtrsim1$,  $T_{m}^{(1)}=\sum_{h=1}^{h^{*}} T_{m,h}^{(1)}$ and  $h^{*}=$\linebreak$\left\lceil\log(4)+\frac{2}{\overline{k}-2}\max_{m}\{\log(\gamma_{m}^{-1})\}\right\rceil$. Then,  under the hyper-parameters setting provided in Theorem \ref{MSANSGA}, Algorithm~\ref{MSANSGA} outputs estimators $\{\widehat v_{n}\}_{n=1}^{\overline{k}}$ satisfying, with probability at least $1-2\delta$,
    $$
    \max_{n\in\left[\overline{k}\right]} \mathcal{L}\left(\widehat{v}_{n},v_{n}^{*}\right) \lesssim\log^3\left(\frac{\overline{k}T^{(3)}d}{\delta}\right)\left(1+\frac{\overline{k}d^2}{\lambda^2}\right)\frac{\sqrt{\overline{k}}}{\sqrt{\delta}T^{(3)}}.
    $$
\end{corollary}
\begin{proof}
    The proof of Corollary \ref{main-result-gaussian-ass-even} proceeds analogously to that of Theorem \ref{main-theorem}. It suffices to replace every use of Theorem \ref{thm-phase-III-tensor-PCA-sub-Gaussian} in the proof of Theorem \ref{main-theorem} with Theorem \ref{thm-phase-III-tensor-PCA}.
\end{proof}

\subsection{Simple Iterative Approach for Computing  the Top Eigenvector}\label{simple-top-eig}
Based on the proof in Section \ref{last-proof-main-thm}, we observe a remarkable gap between the largest eigenvalue $\lambda_1\left(\widehat{W}_{m}\widehat{W}_{m}^{\top}\right)$ and the second largest eigenvalue $\lambda_2\left(\widehat{W}_{m}\widehat{W}_{m}^{\top}\right)$ of $\widehat{W}_{m}\widehat{W}_{m}^{\top}$. In fact, $\lambda_1\left(\widehat{W}_{m}\widehat{W}_{m}^{\top}\right)$ is close to $1$, while $\lambda_2\left(\widehat{W}_{m}\widehat{W}_{m}^{\top}\right)$ is close to $0$. Consequently, in the final step of Algorithm \ref{MSANSGA}, we may employ the power method to compute the top eigenvector of the positive semi-definite matrix $\widehat{W}_{m}\widehat{W}_{m}^{\top}$:
\begin{equation*}
    \widetilde{v}_{2m-1}^{(t+1)}= \widehat{W}_{m}\widehat{W}_{m}^{\top}v_{2m-1}^{(t)},\qquad v_{2m-1}^{(t+1)}=\frac{1}{\left\|\widetilde{v}_{2m-1}^{(t+1)}\right\|}\widetilde{v}_{2m-1}^{(t+1)}.
\end{equation*}
The same approach can be applied to compute the top eigenvector of $\widehat{W}_{m}^{\top}\widehat{W}_{m}$. Alternatively, we may adopt a gradient ascent method for the same purpose:
\begin{align}\label{compute-top-eigenvector-GA}
	v_{2m-1}^{(t+1)}=v_{2m-1}^{(t)}+\eta_t\nabla_v\widetilde{\calR}_{2m-1}\left(v_{2m-1}^{(t)}\right),
\end{align}
where the reward function $\widetilde{\mathcal{R}}_{2m-1}(v)\coloneqq\left\langle v, \widehat{W}_{m}\widehat{W}_{m}^{\top}v\right\rangle$ is defined for any $v\in\mathbb{R}^{d_{2m-1}}$. Let $u_{2m-1,1}$ denote the leading eigenvector of $\widehat{W}_{m}\widehat{W}_{m}^{\top}$. Following the proof in Section \ref{last-proof-main-thm}, we have:
\begin{equation*}
	\begin{aligned}
		\frac{\left\langle v_{2m-1}^{(t+1)},u_{2m-1,1}\right\rangle}{\left\|v_{2m-1}^{(t+1)}\right\|}=\frac{\left\langle v_{2m-1}^{(t)},u_{2m-1,1}\right\rangle}{\left\|v_{2m-1}^{(t)}\right\|}+\eta_t&\left(\frac{\left\langle v_{2m-1}^{(t)}, \widehat{W}_{m}\widehat{W}_{m}^{\top}u_{2m-1,1}\right\rangle}{\left\|v_{2m-1}^{(t)}\right\|}\right.
		\\
		&\, \, \, \, \left.-\left\langle v_{2m-1}^{(t)},u_{2m-1,1}\right\rangle\cdot\frac{\left\langle v_{2m-1}^{(t)},\widehat{W}_{m}\widehat{W}_{m}^{\top}v_{2m-1}^{(t)}\right\rangle}{\left\|v_{2m-1}^{(t)}\right\|^3}\right)+o(\eta_t)
		\\
		=\frac{\left\langle v_{2m-1}^{(t)},u_{2m-1,1}\right\rangle}{\left\|v_{2m-1}^{(t)}\right\|}+\eta_t&\left(1-\left(\frac{\left\langle v_{2m-1}^{(t)},u_{2m-1,1}\right\rangle}{\left\|v_{2m-1}^{(t)}\right\|}\right)^2\right)\cdot\frac{\left\langle v_{2m-1}^{(t)},u_{2m-1,1}\right\rangle}{\left\|v_{2m-1}^{(t)}\right\|}+o(\eta_t).
	\end{aligned}
\end{equation*}
Thus, the convergence of the gradient ascent method in Eq.~\eqref{compute-top-eigenvector-GA} can be established by adapting the techniques used by Ding et al. \cite{ding2025scaling} for analyzing SGD in high-dimensional quadratically parameterized models, reducing the analysis to the one-dimensional case and replacing the stochastic noise with a deterministic counterpart.

\subsection{Hyper‑parameter Configuration for Practical Implementation}\label{para-setting-empirical}
\begin{algorithm}[ht] \caption{Reference Parameter Search} \label{search-c_3-even}
	\footnotesize
	\textbf{Input:} Non-negative integer $\tau$, Maximum iteration $\widetilde{T}$.
	
	\textbf{Output:} Reference parameter $\mathsf{c}_3$. 
	
	\begin{algorithmic}[1]
        \STATE Initialize $\mathsf{c}_3 \leftarrow \texttt{null}$.
		\FOR{$\tau = 0$ to $\widehat{T}$}
        \STATE Set $N_1^{(\tau)} \gtrsim \kappa^{-2k\tau}\prod_{m=1}^{k} d_{2m-1}$.
		\STATE Sample $N_1^{(\tau)}$ fresh scores $\{\mathrm{score}_{i}\}_{i\in[N_1^{(\tau)}]}$ which are defined in Eq.~\eqref{esti-cor-event}.
		\IF{$\frac{1}{N_1^{(\tau)}}\left|\sum_{i=1}^{N_1^{(\tau)}}\mathrm{score}_i\right|\geq\kappa^{k\tau}$}
		\STATE Let $\mathsf{c}_3\leftarrow\kappa^{\tau}$.
		\STATE \textbf{Break}.
		\ENDIF
		\ENDFOR
	\end{algorithmic}
\end{algorithm}
Since the exact values of $\Cor(v_{2m-1}^{*}, v_{2m}^{*})$ are unavailable at runtime, a direct approach would require estimating each term for all $m \in [k]$. To formalize this, let $\widetilde{\Cor}(v_{2m-1}^{*}, v_{2m}^{*})$ denote an estimator of $\Cor(v_{2m-1}^{*}, v_{2m}^{*})$, and define
\begin{align}\label{esti-gamma}
\widetilde{\gamma}_{m}\coloneqq\frac{\widetilde{\Cor}(v_{2m-1}^{*}, v_{2m}^{*})}{(d_{2m-1}d_{2m})^{1/4}} + \frac{\widetilde{\mathsf{c}}_{0}}{\left(d_{2m-1}d_{2m}\right)^{1/2}},
\quad \forall m \in [k],
\end{align}
where $0 \le \widetilde{\mathsf{c}}_{0} \le \mathsf{c}_{0}$ and $\mathsf{c}_{0}$ is defined in Eq.~\eqref{c0-definition}.

Importantly, setting $\widetilde{\mathsf{c}}_{0}=0$ eliminates the need to estimate each $\Cor(v_{2m-1}^{*}, v_{2m}^{*})$ individually; instead, it suffices to obtain a reliable estimate of the aggregate quantity $\prod_{m=1}^{k}\left|\Cor(v_{2m-1}^{*}, v_{2m}^{*})\right|$, thereby simplifying the hyper-parameter settings in \textbf{Phase~I}.
% Since the exact values of $\mathrm{Cor}(v_{2m-1}^{*},v_{2m}^{*})$ are unavailable during the algorithm runtime, we need to estimate them for all $m\in[k]$. We first introduce the necessary notation. Let $\widetilde{\mathrm{Cor}}(v_{2m-1}^{*},v_{2m}^{*})$ denote an estimate of $\mathrm{Cor}(v_{2m-1}^{*},v_{2m}^{*})$. Then we define $\widetilde{\gamma}_{m}$  as follows:
% \begin{align}\label{esti-gamma}
% 	\widetilde{\gamma}_{m}\coloneqq \frac{\widetilde{\mathrm{Cor}}\left(v_{2m-1}^{*},v_{2m}^{*}\right)}{\left(d_{2m-1}d_{2m}\right)^{1/4}}+\frac{\widetilde{\mathsf{c}}_{0}}{\left(d_{2m-1}d_{2m}\right)^{1/2}}, \qquad \forall m\in[k],
% \end{align}
% where 
% $0\leq \widetilde{\mathsf{c}}_{0}\leq \mathsf{c}_{0}$, with $\mathsf{c}_{0}$ defined in Eq.~\eqref{c0-definition}. Note that setting $\widetilde{\mathsf{c}}_{0}=0$ in Eq.~\eqref{esti-gamma} reduces the design of hyper-parameters for \textbf{Phase~I} to providing a suitable estimate for $\prod_{m=1}^{k}\left|\mathrm{Cor}\left(v_{2m-1}^{*},v_{2m}^{*}\right)\right|$. 
For \textbf{Phase~II}, it suffices to give a lower bound estimate $\widetilde{\lambda}_{m}^{\mathrm{II}}$ for each $\lambda_{m}^{\mathrm{II}}$. 
%Next, we introduce the specific settings of each parameter.

Consider Problem \ref{PCA-def} with SNR $\lambda\asymp 1$. The choice of the hyper-parameters depends on the reference parameter $\mathsf{c}_{3}$, which is determined by the following procedure. First, select a decay factor $\kappa\in(0,1)$. For each integer $\tau \in \{0,1,\dots,\widetilde{T}\}$, where
$$
\widetilde{T}:=\left\lceil\left.\left(\frac{1}{2}-\frac{1}{k}\right)\log\left(\prod_{n=1}^{\bar{k}}d_{n}^{1/\bar{k}}\right)\right/\log\left(\kappa^{-1}\right)\right\rceil.
$$
we proceed as follows. Given a pre-processing sample size $N_1^{(\tau)}$ satisfying $N_{1}^{(\tau)} \gtrsim \kappa^{-2k\tau}\prod_{m=1}^{k} d_{2m-1}$, we collect $N_{1}^{(\tau)}$ scores $\left\{\mathrm{score}_{i}\right\}$ in an online manner and check whether the following event holds:
\begin{align}\label{esti-cor-event}
	\frac{1}{N_{1}^{(\tau)}}\left|\sum_{i=1}^{N_{1}^{(\tau)}}\underbrace{\left\langle \bigotimes_{m=1}^{k}\begin{pmatrix} I_{d_{2m-1}} & 0_{d_{2m-1}\times(d_{2m}-d_{2m-1})}\end{pmatrix},\mathbf{T}^{(i)}\right\rangle}_{\mathrm{score}_{i}}\right|\geq\kappa^{k\tau},
\end{align}
% set $c_{3}^{(\tau)}=\kappa^{\tau}$. Given a pre-processing sample size $N_{1}$ satisfying $N_{1} \gtrsim \left(c_{3}^{(\tau)}\right)^{-2k}\prod_{m=1}^{k} d_{2m-1}$, we collect $N_{1}$ scores $\left\{\mathrm{score}_{i}\right\}$ in an online manner and check whether the following event holds:
% \begin{align}\label{esti-cor-event}
% 	\left|\frac{1}{N_{1}}\sum_{i=1}^{N_{1}}\underbrace{\left\langle \bigotimes_{m=1}^{k}\begin{pmatrix} I_{d_{2m-1}} & 0_{d_{2m-1}\times(d_{2m}-d_{2m-1})}\end{pmatrix},\mathbf{T}^{(i)}\right\rangle}_{\mathrm{score}_{i}}\right|\geq\left(c_{3}^{(\tau)}\right)^{k},
% \end{align}
where $\frac{1}{N_{1}^{(\tau)}}\left|\sum_{i=1}^{N_{1}^{(\tau)}}\mathrm{score}_{i}\right|$ denotes the estimate for $\prod_{m=1}^{k}\left|\mathrm{Cor}\left(v_{2m-1}^{*},v_{2m}^{*}\right)\right|$. A complete specification of the algorithm is presented in Algorithm~\ref{search-c_3-even}.

\paragraph{Case I:}If the first $\tau_{1}$ is found such that the event in Eq.~\eqref{esti-cor-event} holds, we let $\mathsf{c}_{3}\coloneqq\kappa^{\tau_{1}}$ and define the effective block SNR ratio estimators for \textbf{Phase~I} and \textbf{Phase~II} as
\begin{equation}\label{esti-block-eff-SNR-caseI}
	\begin{gathered}
		\widetilde{\lambda}_{m,h_1}^{\mathrm{I}} \coloneqq \frac{\lambda\mathsf{c}_3}{4^{k-1}}\prod_{i\in[k]}^{i\neq m}\widetilde{\gamma}_i=\frac{\lambda\mathsf{c}_3^k}{4^{k-1}}\prod_{i\in[k]}^{i\neq m}\left(d_{2i-1}d_{2i}\right)^{-\frac{1}{4}},
		\\ \widetilde{\lambda}_{m,h_2}^{\mathrm{I}} \coloneqq \lambda\prod_{i\in[k]}^{i\neq m}\zeta_{i}^{h_2-1}\widetilde{\gamma}_i=\lambda\mathsf{c}_3^{k-1}\prod_{i\in[k]}^{i\neq m}\zeta_{i}^{h_2-1}\left(d_{2i-1}d_{2i}\right)^{-\frac{1}{4}},
		\\
		\widetilde\lambda_{m}^{\mathrm{II}} \coloneqq \lambda\left(1-\frac{1}{k}\right)^{k-m}\prod_{i=1}^{m-1}\left(d_{2i-1}d_{2i}\right)^{-\frac{1}{4}+\frac{1}{4(k-1)}},
	\end{gathered}
\end{equation}
for any $m\in\left[k\right]$, $h_{1}\in[h_{1}^{*}]$ and $h_{2}\in[h_{2}^{*}]$, where $\zeta_{m}=\exp\left\{\left(\log(\mathsf{c}_3^{-1})+\frac{1}{4(k-1)}\log(d_{2m-1}d_{2m})\right)/h_2^{*}\right\}$, $\widetilde{\gamma}_{m}=\frac{\mathsf{c}_3}{\left(d_{2m-1}d_{2m}\right)^{1/4}}$, and $h_1^{*}\coloneqq\left\lceil\log(4)+(k-1)\log\left(\mathsf{c}_3^{-1}\right)\right\rceil$, and $h_2^{*}\coloneqq\left\lceil\log\left(\mathsf{c}_3^{-1}\right)+\frac{1}{4(k-1)}\log\left(d_{2k-1}d_{2k}\right)\right\rceil$. In this case, we divide \textbf{Phase~I} into two stages, where the required sample complexity terms are given by
\begin{equation}\label{empirical-sample-size-caseI}
    \begin{gathered}
        T_{m,h_{1}}^{(1)}\asymp\frac{\mathrm{UB}_{m,h_{1}}^{\mathrm{I}}}{\widetilde{\lambda}_{m,h_{1}}^{\mathrm{I}}\eta_{m,h_{1}}^{(1)}},\quad T_{m,h_{2}}^{(1)}\asymp\frac{\mathrm{UB}_{m,h_{2}}^{\mathrm{I}}}{\widetilde{\lambda}_{m,h_{2}}^{\mathrm{I}}\eta_{m,h_{2}}^{(1)}},\quad  T_{m}^{(2)}\asymp\frac{1}{\widetilde{\lambda}_{m}^{\mathrm{II}}\eta_{m}^{(2)}}, \\
        T_{m}^{(3)}\equiv T^{(3)}\gtrsim \log^4\left(\frac{kd_{2k}T^{(3)}}{\delta}\right)\cdot\frac{(kd_{2k-1}d_{2k})^{2}}{\lambda^{2}},
    \end{gathered}
\end{equation}
and the corresponding step-sizes are chosen as
\begin{equation}\label{empirical-set-size-caseI}
	\begin{gathered}
		\eta_{m,h_1}^{(1)}\asymp\frac{1}{\mathsf{c}_2\left(h_1^{*}\right)^2}\cdot\min\left\{\frac{\widetilde{\lambda}_{m,h_1}^{\mathrm{I}}}{\mathrm{UB}_{m,h_1}^{\mathrm{I}}d_{2m-1}d_{2m}},\mathrm{UB}_{m,h_1}^{\mathrm{I}}\widetilde{\lambda}_{m,h_1}^{\mathrm{I}}\right\},\quad \eta_{m,h_2}^{(1)}\asymp\frac{1}{\mathsf{c}_2\left(h_2^{*}\right)^2}\cdot\frac{\widetilde{\lambda}_{m,h_2}^{\mathrm{I}}}{\mathrm{UB}_{m,h_2}^{\mathrm{I}}d_{2m-1}d_{2m}},
		\\
		\eta_{m}^{(2)}\asymp\frac{\widetilde{\lambda}_{m}^{\mathrm{II}}}{\mathsf{c}_2kd_{2m-1}d_{2m}},\quad \eta_{m}^{(3)}\asymp\frac{\log^2\left(T^{(3)}\right)}{\lambda T^{(3)}},
	\end{gathered}
\end{equation}
where $\mathsf{c}_{2}\asymp k\log(kd_{2k}/\delta)$, $\mathrm{UB}_{m,h_1}^{\mathrm{I}}\coloneqq e^{h_1}\mathsf{c}_3^k\left(d_{2m-1}d_{2m}\right)^{-1/4}/4$ and $\mathrm{UB}_{m,h_2}^{\mathrm{II}}\coloneqq \mathsf{c}_3\zeta_{m}^{h_2}\left(d_{2m-1}d_{2m}\right)^{-1/4}$ for any $m\in[k]$, $h_1\in[h_1^{*}]$ and $h_2\in[h_2^{*}]$.

\paragraph{Case II:}If there is not a $\tau$ such that the event in Eq.~\eqref{esti-cor-event} holds, we define the effective block SNR ratio estimators for \textbf{Phase~I} and \textbf{Phase~II} as
\begin{equation}\label{esti-block-eff-SNR-caseII}
	\begin{gathered}
	    \widetilde{\lambda}_{m,h}^{\mathrm{I}}\coloneqq\frac{\lambda}{4^{k-1}}\prod_{i\in[k]}^{i\neq m}\zeta_{i}^{h-1}\widetilde{\gamma}_i=\frac{\lambda\mathsf{c}_0^{k-1}}{4^{k-1}}\prod_{i\in[k]}^{i\neq m}\zeta_{i}^{h-1}\left(d_{2i-1}d_{2i}\right)^{-\frac{1}{2}}, \\
        \widetilde\lambda_{m}^{\mathrm{II}} \coloneqq \lambda\left(1-\frac{1}{k}\right)^{k-m}\prod_{i=1}^{m-1}\left(d_{2i-1}d_{2i}\right)^{-\frac{1}{2}+\frac{1}{2(k-1)}},
	\end{gathered}
\end{equation}
for any $m\in[k]$ and $h\in[h^{*}]$, where $\zeta_{m}=\exp\left\{\left(\log(4/\mathsf{c}_0)+\frac{1}{2(k-1)}\log(d_{2m-1}d_{2m})\right)/h^{*}\right\}$, $\widetilde{\gamma}_{m}=\mathsf{c}_0\left(d_{2m-1}d_{2m}\right)^{-1/2}$ and $h^{*}\coloneqq \left\lceil\log\left(4/\mathsf{c}_0\right)+\frac{1}{2(k-1)}\log\left(d_{2k-1}d_{2k}\right)\right\rceil$. In this case, the required sample complexity terms are given by
\begin{align}\label{empirical-sample-size-caseII}
	T_{m,h}^{(1)}\asymp\frac{\mathrm{UB}_{m,h}^{\mathrm{I}}}{\widetilde{\lambda}_{m,h}^{\mathrm{I}}\eta_{m,h}^{(1)}},\quad T_{m}^{(2)}\asymp\frac{1}{\widetilde{\lambda}_{m}^{\mathrm{II}}\eta_{m}^{(2)}},\quad T_{m}^{(3)}\equiv T^{(3)}\gtrsim \log^4\left(\frac{kd_{2k}T^{(3)}}{\delta}\right)\cdot\frac{(kd_{2k-1}d_{2k})^{2}}{\lambda^{2}},
\end{align}
and the corresponding step-sizes are chosen as
\begin{equation}\label{empirical-set-size-caseII}
	\begin{gathered}
		\eta_{m,h}^{(1)}\asymp\frac{1}{\mathsf{c}_2\left(h^{*}\right)^2}\cdot\frac{\widetilde{\lambda}_{m,h}^{\mathrm{I}}}{\mathrm{UB}_{m,h}^{\mathrm{I}}d_{2m-1}d_{2m}},\quad \eta_{m}^{(2)}\asymp\frac{\widetilde{\lambda}_{m}^{\mathrm{II}}}{\mathsf{c}_2kd_{2m-1}d_{2m}},\quad \eta_{m}^{(3)}\asymp\frac{\log^2\left(T^{(3)}\right)}{\lambda T^{(3)}},
	\end{gathered}
\end{equation}
where $\mathsf{c}_{2}\asymp k\log(kd_{2k}/\delta)$ and $\mathrm{UB}_{m,h}^{\mathrm{I}} \coloneqq \mathsf{c}_{0}\zeta_{m}^{h}\left(d_{2m-1}d_{2m}\right)^{-1/2}/4$ for any $m\in[k]$ and $h\in[h^{*}]$.
\begin{corollary}\label{coro-gen-phase-I}
	Let $\{\overline{\gamma}_{m}\}_{m\in[k]}$ and $\{\breve{\gamma}_{m}\}_{m\in[k]}$ denote the lower bound estimate of $\{\widetilde{\gamma}_{m}\}_{m\in[k]}$ and objective reference parameters for \textbf{Phase~I}, respectively. Given the effective block SNR ratio estimators $\left\{\widetilde{\lambda}_{m,h}^{\mathrm{I}}\right\}_{m\in[k]}$ for \textbf{Phase~I} which depend on $\{\overline{\gamma}_{m}\}_{m\in[k]}$ and $\{\breve{\gamma}_{m}\}_{m\in[k]}$, suppose that Assumptions \ref{ass-pro} and \ref{ass-base} hold and consider the dynamic generated via \textbf{Phase~I} of Algorithm \ref{MSANSGA}. For any $\delta''\in(0, 1)$, if we pick 
	\begin{gather*}
		\eta_{m,h}^{(1)}\lesssim\frac{1}{\mathsf{c}_2\left(h^{*}\right)^{2}}\cdot\min\left\{\frac{\widetilde{\lambda}_{m,h}^{\mathrm{I}}}{\mathrm{UB}_{m,h}^{\mathrm{I}}d_{2m-1}d_{2m}},\frac{\mathrm{UB}_{m,h}^{\mathrm{I}}\widetilde{\lambda}_{m,h}}{\zeta_{m}^{2}},\frac{\mathrm{LB}_{m,h}^{\rmI}}{\lambda_{m,h}^{\mathrm{I}}+1}\right\}, \quad
		T_{m,h}^{(1)}\asymp\frac{\mathrm{UB}_{m,h}^{\mathrm{I}}}{\eta_{m,h}^{(1)}{\lambda}_{m,h}^{\mathrm{I}}} \text{ or } \frac{\mathrm{UB}_{m,h}^{\mathrm{I}}}{\eta_{m,h}^{(1)}\widetilde{\lambda}_{m,h}^{\mathrm{I}}},
	\end{gather*}
	for any $m\in[k]$ and $h\in[h^{*}]$, where
	\begin{align*}
		\mathsf{c}_{2}\asymp k\log(kd_{2k}/\delta''),\quad \mathrm{LB}_{m,h}^{\mathrm{I}}\coloneqq\frac{\zeta_{m}^{h-1}\overline{\gamma}_{m}}{2}, \quad
		\mathrm{UB}_{m,h}^{\mathrm{I}}\coloneqq\zeta_{m}^{h}\overline{\gamma}_{m}, \quad \zeta_{m}=\exp\left\{\frac{\log(4)+\log\left(\breve{\gamma}_{m}/\overline{\gamma}_{m}\right)}{h^{*}}\right\},
	\end{align*}
	and $h^{*}=\text{$\left\lceil\log(4)+\max_{m}\left\{\log\left(\breve{\gamma}_{m}/\overline{\gamma}_{m}\right)\right\}\right\rceil$}$. Then $\alpha_{m}^{\left(T_{m}^{(1)}\right)} \geq \breve{\gamma}_{m}$ holds for every $m\in[k]$ with probability at least $1-\delta''/2$.
\end{corollary}
\begin{proof}
	It remains to further prove that if $\alpha_{m,h}^{(0)} \geq \bigl(1+\frac{1}{\left(h^{*}\right)^{2}}\bigr) \mathrm{UB}_{m,h}^{\mathrm{I}}$, then 
    \begin{equation*}
        \alpha_{m,h}^{(T_{m,h}^{(1)})}\geq \frac{1}{1+\left(h^{*}\right)^{2}}\alpha_{m,h}^{(0)},
    \end{equation*}
	with probability at least $1-\frac{\delta''}{kh^{*}}$. The rest of the proof can be completed by applying the technique used in the proof of Theorem \ref{thm-phase-I-tensor-PCA}, with appropriate modifications of the hyper-parameters. For any $t\in \left[0:T_{m,h}^{(1)}-1\right]$, applying Eq.~\eqref{all-iteration-update-alpha} yields
	\begin{align*}
		\alpha_{m,h}^{(t+1)}&\geq \alpha_{m,h}^{(t)}+\frac{\lambda_{m,h}^{\rmI}}{2}\eta_{m,h}^{(1)}-\left(\eta_{m,h}^{(1)}\right)^{2}\left|\Phi_1\left(W_{m,h}^{(t)},G_{m,h}^{(t)},v_{2m-1}^{*},v_{2m}^{*},\overline{\eta}_{m,h}\right)\right|+\eta_{m,h}^{(1)}\cdot\widehat{\xi}_{m,h}^{(t+1)}\notag \\
		&\geq \alpha_{m,h}^{(t)}+\frac{\lambda_{m,h}^{\rmI}}{2}\eta_{m,h}^{(1)}-\left(\eta_{m,h}^{(1)}\right)^{2}\cdot\widetilde{\mathcal{O}}\left(\left(\lambda_{m,h}^{\mathrm{I}}\right)^{2}+\alpha_{m,h}^{(0)}d_{2m-1}d_{2m}\right)+\eta_{m,h}^{(1)}\cdot\widehat{\xi}_{m,h}^{(t+1)},
	\end{align*}
	where $G_{m,h}^{(t)}$ denotes the stochastic gradient of the risk function $\widehat{\mathcal{R}}_{m,\mathrm{I}}^{(t+1)}$ evaluated at the parameter matrix $W_{m,h}^{(t)}$ in stage $h$, and $\widehat{\xi}_{m,h}^{(t+1)}$ is a zero-mean random term which has been defined in Eq.~\eqref{def-scalar-random-xi}. As similar as Eq.~\eqref{prob-calEc-component}, we obtain
	\begin{align*}
		&\phantom{\mathrel{=}}\mathbb{P}\left(\alpha_{m,h}^{(T_{m,h}^{(1)})}< \frac{1}{1+\left(h^{*}\right)^{2}}\alpha_{m,h}^{(0)}\right) 
        \\
		&\leq \exp\left\{-\frac{\left(\alpha_{m,h}^{(0)}+\frac{\lambda_{m,h}^{\rmI}}{2}T_{m,h}^{(1)}\eta_{m,h}^{(1)}-\widetilde{\mathcal{O}}\left(T_{m,h}^{(1)}\left(\eta_{m,h}^{(1)}\right)^{2}\cdot\left(\left(\lambda_{m,h}^{\mathrm{I}}\right)^{2}+\alpha_{m,h}^{(0)}d_{2m-1}d_{2m}\right)\right)-\mathrm{UB}_{m,h}^{\mathrm{I}}\right)^{2}}{\widetilde{\mathcal{O}}\left(T_{m,h}^{(1)}\left(\eta_{m,h}^{(1)}\right)^{2}\right)}\right\} \\
		&\leq \exp\left\{-\frac{\left(\mathrm{UB}_{m,h}^{\mathrm{I}}\right)^{2}}{\widetilde{\mathcal{O}}\left(\left(h^{*}\right)^{2}T_{m,h}^{(1)}\left(\eta_{m,h}^{(1)}\right)^{2}\right)}\right\}\leq\frac{\delta''}{kh^{*}}.
	\end{align*}
\end{proof}
Applying the technique used in the proof of Theorem \ref{thm-phase-II-tensor-PCA} with appropriate modifications of the hyper-parameters, we can derive the following corollary.
\begin{corollary}\label{coro-gen-phase-II}
	Let $\{\breve{\gamma}_{m}\}_{m\in[k]}$ denotes the objective reference parameters for \textbf{Phase~I}, respectively. Given the effective block SNR ratio estimators $\left\{\widetilde{\lambda}_{m,h}^{\mathrm{II}}\right\}_{m\in[k]}$ for \textbf{Phase~II} which depend on $\left\{\breve{\gamma}_{m}\right\}_{m\in[k]}$, suppose that Assumptions \ref{ass-pro} and \ref{ass-base} hold and consider the dynamic generated via \textbf{Phase~II} of Algorithm \ref{MSANSGA}. For any $\delta''\in(0, 1)$, if we pick
	\begin{align*}
		\eta_{m}^{(2)}\lesssim\frac{1}{\mathsf{c}_2}\cdot\min\left\{\frac{\widetilde{\lambda}_{m}^{\mathrm{II}}}{kd_{2m-1}d_{2m}}, \frac{1}{\lambda_{m}^{\mathrm{II}}+1}\right\},\quad T^{(2)}\asymp\frac{1}{\eta_{m}^{(2)}\lambda_{m}^{\mathrm{II}}} \text{ or } \frac{1}{\eta_{m}^{(2)}\widetilde{\lambda}_{m}^{\mathrm{II}}},
	\end{align*}
	for any $m\in[k]$, where $\mathsf{c}_{2}\asymp k\log(kd_{2k}/\delta'')$. Then $\alpha_{m}^{\left(T_{m}^{(1)}+T_{m}^{(2)}\right)} \ge 1-\epsilon$ holds for every $m\in[k]$ with probability at least $1-\delta''/2$.
\end{corollary}

\subsection{Proof of Corollaries \ref{cor-equal-dim} and \ref{empirical-hyper-para-convergence}}
\begin{proof}[Proof of Corollary \ref{cor-equal-dim}]
	Suppose that the sample complexity terms and the corresponding step-sizes are given by Eq.~\eqref{hyper-para-T-even} and Eq.~\eqref{hyper-para-eta-even} in Theorem \ref{main-theorem}, respectively.
 %    \begin{equation}\label{coro-eta-setting}
	% 	\begin{aligned}
	% 		\eta_{m,h}^{(1)} \asymp \frac{\lambda_{m,h}^{\mathrm{I}}}{\mathsf{c}_{2}\zeta_{m}^{h}\gamma_{m}d_{2m-1}d_{2m}},\quad \eta_{m}^{(2)} \asymp \frac{\lambda_{m}^{\mathrm{II}}}{\mathsf{c}_{2}kd_{2m-1}d_{2m}},\quad \eta_{m}^{(3)}\asymp\frac{\log^{2}\left(T^{(3)}\right)}{\lambda T^{(3)}},
	% 	\end{aligned}
	% \end{equation}
 %    for any $m\in[k]$ and $h\in[h^*]$, where $\mathsf{c}_{2}\asymp k\log(kd_{2k}/\delta)$.
    Algorithm \ref{MSANSGA} achieves strong recovery of each $v_{n}^{*}$ provided the sample size $N$ satisfies
    \begin{equation*}
        \lambda^2N \gtrsim \log\left(\frac{kd}{\delta}\right)2^{\overline{k}}\overline{k}^3 \max_{m'\in[\overline{k}/2]}\left(\frac{\left|\left\langle v_{2m'-1}^{*},v_{2m'}^{*}\right\rangle\right|d^{1/2}+\mathsf{c}_{0}}{d^{1/2}}\right)^{4}\prod_{m=1}^{\overline{k}/2}\left(\frac{d^{1/2}}{\left|\left\langle v_{2m-1}^{*},v_{2m}^{*}\right\rangle\right|+\mathsf{c}_{0}d^{-1/2}}\right)^{2},
    \end{equation*}
	with probability at least $1-2\delta$. 
	
	Under the step-sizes setting in Eq.~\eqref{hyper-para-eta-even}, if the resulting sample size satisfies $\sum_{m=1}^k T_{m,1}^{(1)}\gtrsim\log\left(\frac{kd}{\delta}\right)4^{\overline{k}}\overline{k}^3$\linebreak $\mathsf{c}_0^{4-\overline{k}}d^{\overline{k}-2}$, we instead adopt the effective block SNR ratio estimators $\left\{\widetilde{\lambda}_{m,h}^{\mathrm{I}}\right\}_{m\in[k],h\in[h^{*}]}$ and $\left\{\widetilde{\lambda}_{m}^{\mathrm{II}}\right\}_{m\in[k]}$ for \textbf{Phase I} and \textbf{Phase II}, which are specified in Eq.~\eqref{esti-block-eff-SNR-caseII}, and choose the required sample complexity terms as
	\begin{align*}
		T_{m,h}^{(1)}\asymp\frac{\mathrm{UB}_{m,h}^{\mathrm{I}}}{{\lambda}_{m,h}^{\mathrm{I}}\eta_{m,h}^{(1)}},\quad T_{m}^{(2)}\asymp\frac{1}{{\lambda}_{m}^{\mathrm{II}}\eta_{m}^{(2)}},\quad T_{m}^{(3)}\equiv T^{(3)}\gtrsim \log^4\left(\frac{kdT^{(3)}}{\delta}\right)\cdot\frac{k^2d^4}{\lambda^2},
	\end{align*}
	the corresponding step-sizes as
	\begin{align*}
		\eta_{m,h}^{(1)}\asymp\frac{1}{\mathsf{c}_{2}\left(h^{*}\right)^2}\cdot\min\left\{\frac{\widetilde{\lambda}_{m,h}^{\mathrm{I}}}{\mathrm{UB}_{m,h}^{\mathrm{I}}d^2}, \frac{\mathrm{LB}_{m,h}^{\rmI}}{\lambda_{m,h}^{\mathrm{I}}+1}\right\},\quad \eta_{m}^{(2)}\asymp\frac{1}{\mathsf{c}_2}\cdot\min\left\{\frac{\widetilde{\lambda}_{m}^{\mathrm{II}}}{kd^2}, \frac{1}{{\lambda}_{m}^{\mathrm{II}}+1}\right\},\quad \eta_{m}^{(3)}\asymp\frac{\log^2\left(T^{(3)}\right)}{\lambda T^{(3)}},
	\end{align*}
	where
	\begin{align*}
		\mathsf{c}_{2}\asymp k\log(kd/\delta''),\quad \mathrm{LB}_{m,h}^{\mathrm{I}}\coloneqq\frac{\zeta_{m}^{h-1}\overline{\gamma}_{m}}{2}, \quad
		\mathrm{UB}_{m,h}^{\mathrm{I}}\coloneqq\zeta_{m}^{h}\overline{\gamma}_{m}, \quad \zeta_{m}=\exp\left\{\frac{\log(4)+\log\left(\breve{\gamma}_{m}/\overline{\gamma}_{m}\right)}{h^{*}}\right\},
	\end{align*} 
    and set $h^{*}=\text{$\left\lceil\log(4)+\max_{m}\left\{\log\left(\breve{\gamma}_{m}/\overline{\gamma}_{m}\right)\right\}\right\rceil$}$,  $\overline{\gamma}_{m}=\mathsf{c}_0d^{-1}/4$, and $\breve{\gamma}_{m}=d^{-1+1/(k-1)}$ for any $m\in[k]$. Then, combining Corollaries \ref{coro-gen-phase-I} and \ref{coro-gen-phase-II} with Theorem \ref{thm-phase-III-tensor-PCA-sub-Gaussian} and letting $\delta''\leq\delta/8$, for any sample size $N$ satisfying
	\begin{equation*}
    	\lambda^2N\gtrsim \log^3\left(\frac{kd}{\delta}\right)2^{\overline{k}}\overline{k}^3\mathsf{c}_0^{4-\overline{k}}d^{\overline{k}-2},
	\end{equation*}
	Algorithm~\ref{MSANSGA} achieves strong recovery of each $v_n^{*}$ with probability at least $1-2\delta$.
\end{proof}
\begin{proof}[Proof of Corollary \ref{empirical-hyper-para-convergence}]
	Under the hyper-parameter setting of \textbf{Case I} in Section \ref{para-setting-empirical}, for any $m\in[k]$, we let $\overline{\gamma}_{m}=\mathsf{c}_3^kd^{-1/2}/4$ and $\breve{\gamma}_{m}=\mathsf{c}_3d^{-1/2}$ for the first stage, and let $\overline{\gamma}_{m}=\mathsf{c}_3d^{-1/2}$ and $\breve{\gamma}_{m}=d^{-1/2+1/(2(k-1))}$ for the second stage. Under the hyper-parameter setting of \textbf{Case II} in Section \ref{para-setting-empirical}, for any $m\in[k]$, we let $\overline{\gamma}_{m}=\mathsf{c}_0d^{-1}/4$ and $\breve{\gamma}_{m}=d^{-1+1/(k-1)}$. Then, combining Corollaries \ref{coro-gen-phase-I} and \ref{coro-gen-phase-II} with Theorem \ref{thm-phase-III-tensor-PCA-sub-Gaussian}, we complete the proof.
\end{proof}

\section{Proof for the Odd Case}\label{proof-odd}
\begin{algorithm}[ht] \caption{Sequential Normalized SGA with Constant Step-Size (Odd Case)} \label{SGA-con-ss-odd}
	\footnotesize
	\textbf{Input:} Initial weights {\scriptsize $w_{1}^{(0)}\bigcup\left\{W_{m}^{(0)}\right\}$}, step sizes $\left\{\eta_{m}\right\}$, iteration budgets $\left\{T_{m}\right\}$, reward oracle $\widehat{\mathcal{R}}$ with query access $\widehat{\mathcal{R}}_{m}^{(t)}(\cdot)$.
	
	\textbf{Output:} Normalized final weights {\scriptsize $\overline{w}_{1}^{(T_{1})}\bigcup\left\{\overline W_{m}^{(T_{m})}\right\}$}. 
	
	\textbf{Variables:} Memory states $M = \left\{M_{m}: m\in[k+1]\right\}$,  where $M_{1}$ stores the iterates {\scriptsize $\left\{w_{1}^{(t)}:t \in \{0\} \cup [T_{1}]\right\}$} and $M_{m}$ stores the iterates {\scriptsize $\left\{W_{m}^{(t)} : t \in \{0\} \cup [T_{m}]\right\}$} for block $m\in[2:k+1]$.
	
	\begin{algorithmic}[1]
		\FOR{$m = k+1, k, \dots, 2$}
		\FOR{$t = 0, \dots, T_{m}-1$}
		\STATE Sample fresh data $\mathbf T^{(t+1)}$.
		\STATE Update iterate in memory $M_{m}$ via normalized stochastic gradient ascent:
        {\scriptsize
		\begin{equation}\label{update-SGA-odd}
			W_{m}^{(t+1)}\leftarrow \left(1+\frac{\eta_{m} \widehat{\mathcal{R}}_{m}^{(t+1)}\left(W_{m}^{(t)}\right)}{\left\|W_{m}^{(t)}\right\|_{\mathrm{F}}}\right)W_{m}^{(t)}+\eta_{m}\left\|W_{m}^{(t)}\right\|_{\mathrm{F}}\nabla_{W_{m}}\widehat{\mathcal{R}}_{m}^{(t+1)}\left(W_{m}^{(t)}\right).
		\end{equation}}%
		\ENDFOR
		\STATE Set final weight: $\overline W_{m}^{(T_{m})} \gets W_{m}^{(T_{m})}/\left\|W_{m}^{(T_{m})}\right\|_{\mathrm{F}}$.
		\ENDFOR
		\FOR{$t=0,\cdots, T_{1}-1$}
		\STATE Sample a fresh data $\upT^{(t+1)}$.
		\STATE Update iterate in memory $M_{1}$ via normalized stochastic gradient ascent:
        {\scriptsize
		\begin{equation}\label{update-SGA-odd-vec}
			w_{1}^{(t+1)}\leftarrow \left(1+\frac{\eta_{m} \widehat{\mathcal{R}}_{1}^{(t+1)}\left(w_{1}^{(t)}\right)}{\left\|w_{1}^{(t)}\right\|}\right)w_{1}^{(t)}+\eta_{1}\left\|w_{1}^{(t)}\right\|\nabla_{w_{1}}\widehat{\mathcal{R}}_{1}^{(t+1)}\left(w_{1}^{(t)}\right).
		\end{equation}}%
		\ENDFOR
		\STATE Set final weight: $\overline w_{1}^{(T_{1})} \gets w_{1}^{(T_{1})}/\left\|w_{1}^{(T_{1})}\right\|_{2}$.
		\RETURN {\scriptsize $\overline{w}_{1}^{(T_{1})}\bigcup\left\{\overline{W}_{m}^{(T_{m})}\right\}$}.
	\end{algorithmic}
\end{algorithm}

\subsection{Algorithm}\label{alg-odd-case}
For odd case, we use a vector parameter $w_{1}$ and $k$ matrix $W_{m}$ to represent $v_{1}^{*}$ and the rank-$1$ components $v_{2m-2}^{*}v_{2m-1}^{*\top}$ for $m\in[2:k+1]$, respectively. The inner loop update has been shown in Algorithm \ref{SGA-con-ss-odd} and Algorithm \ref{SGA-decay-ss-odd}. In the $l$-phase, for each $m$, given a fresh data point $\upT^{(t)}$, we consider the block-specific objective
\begin{equation}\label{obj-odd-case}
	\begin{gathered}
		\widehat{\mathcal{R}}_{1}^{(t)}(w_{1})=\left\langle w_{1}\otimes\bigotimes_{i=2}^{k+1} \overline{W}_{i}^{\left(\sum_{\tau=1}^{l}T_{i}^{(\tau)}\right)},\mathbf{T}^{(t)}\right\rangle, \\
		\widehat{\mathcal{R}}_{m}^{(t)}(W_{m})=\left\langle\overline{w}_{1}^{\left(\sum_{\tau=1}^{l-1}T_{1}^{(\tau)}\right)}\otimes\bigotimes_{i=2}^{m-1}\overline{W}_{i}^{\left(\sum_{\tau=1}^{l-1}T_{i}^{(\tau)}\right)}\otimes W_{m}\otimes \bigotimes_{j=m+1}^{k+1}\overline{W}_{j}^{\left(\sum_{\tau=1}^{l}T_{j}^{(\tau)}\right)}, \mathbf{T}^{(t)}\right\rangle, \quad \forall m\geq2,
	\end{gathered}
\end{equation}
and updates $W_{m}$ (or $w_{1}$) via the gradient of $\widehat{\mathcal{R}}_{m}^{(t)}$ with respect to $W_{m}$ (or $w_{1}$). The weight parameter in Algorithm \ref{MSANSGA-odd} is initialized with 
\begin{equation*}  
    w_{m}^{(0)}:=\pm \overline{v}_{1}^{(0)},\qquad v_{1}^{(0)}\sim\mathcal{N}(0,I_{d_{1}}).
\end{equation*}
And the weight matrices in Algorithm~\ref{MSANSGA-odd} are initialized as follows:
\begin{align}\label{init-setting-odd}
	W_{m}^{(0)}:=\pm\underbrace{\frac{d_{2m-2}^{-1/2}}{2}\left(I_{d_{2m-2}},0_{d_{2m-2}\times(d_{2m-1}-d_{2m-2})}\right)}_{\widehat{W}_{m}^{(0)}}\pm\underbrace{\frac{1}{2}\overline{v}_{2m-2}^{(0)}\left(\overline{v}_{2m-1}^{(0)}\right)^{\top}}_{\widetilde{W}_{m}^{(0)}},\qquad \forall m\in[2:k+1],
\end{align}
where the first term $\widehat{W}_{m}^{(0)}$ is deterministic, and the second term $\widetilde{W}_{m}^{(0)}$ is random with $v_{2m-2}^{(0)}\sim \mathcal{N}(0,I_{d_{2m-2}}), \allowbreak v_{2m-1}^{(0)}\sim \mathcal{N}(0,I_{d_{2m-1}})$ independently. We run algorithm over all $2^{2k+1}$ sign combinations for the $\pm$ choices in $v_{1}^{(0)}$, and $\widehat{W}_{m}^{(0)}$ and $\widetilde{W}_{m}^{(0)}$ across $m \in [2:k+1]$. For the $l$-th phase, the effective block SNR is defined by
\begin{equation}\label{lambda-phase_odd}
	\begin{aligned}
		\lambda_{m}^{\mathrm{Phase}}:=\lambda\cdot\prod_{i=1}^{m-1}\left|\left\langle\overline{W}_{i}^{\left(\sum_{\tau=1}^{l-1}T_{i}^{(\tau)}\right)},v_{2i-2}^{*}v_{2i-1}^{*\top}\right\rangle\right|\cdot\prod_{j=m+1}^{k+1}\left|\left\langle\overline{W}_{j}^{\left(\sum_{\tau=1}^{l}T_{j}^{(\tau)}\right)},v_{2j-2}^{*}v_{2j-1}^{*\top}\right\rangle\right|,
	\end{aligned}
\end{equation}
for any $m\in[k+1]$, where $\overline{W}_{1}=v_0^{*}\overline{w}_{1}^{\top}$ and $v_0^{*}:=(1,0,\cdots,0)^{\top}\in\bbR^{d_{1}}$.
Moreover, in \textbf{Phase~I}, the effective block SNR for each cycle $h\in[h^{*}]$ takes a similar form as in Eq.~\eqref{block-eff-SNR-phase-I}:
\begin{equation}\label{lambda-phase-I_odd}
	\begin{aligned}
		\lambda_{m,h}^{\mathrm{I}}&\coloneqq \lambda\cdot\prod_{i=1}^{m-1}\left|\left\langle\overline{W}_{m,h-1}^{\left(T_{m,h-1}^{(1)}\right)},v_{2m-1}^{*}v_{2m}^{*\top}\right\rangle\right|\cdot\prod_{j=m+1}^{k+1}\left|\left\langle\overline{W}_{m,h}^{\left(T_{m,h}^{(1)}\right)},v_{2m-1}^{*}v_{2m}^{*\top}\right\rangle\right|.
	\end{aligned}
\end{equation}

\begin{algorithm}[t] \caption{Multi-Phase Sequential Normalized SGA (Odd Case)} \label{MSANSGA-odd}
	\footnotesize
	\textbf{Input:} cycles $h^{*}$, initializations {\scriptsize $\left\{w_{1}^{(0,\tau)}\right\}_{\tau\in \left[2^{2k+1}\right]}\bigcup\left\{W_{m}^{(0,\tau)}\right\}_{\tau\in\left[2^{2k+1}\right]}$}, phase schedules {\scriptsize $\left\{\left(\eta_{m,h}^{(1)},T_{m,h}^{(1)}\right)\right\}_{h\in[h^{*}]}$}
	and {\scriptsize $\left\{\left(\eta_{m}^{(i)},T_{m}^{(i)}\right)\right\}_{i\in\{2,3\}}$}, step size decay lengths {\scriptsize $\left\{\widehat T_{m}^{(3)}\right\}$}.
	
	\textbf{Output:} Estimator $V = \{v_n\}_{n=1}^{\overline{k}}$.
	
	\textbf{Variables:} Memory states $M = \{ M_{m,\tau}, \widehat{M}_{m,p} : m \in [k+1], \tau \in \left[2^{2k+1}\right], p \in \{0,1\}\}$, where each $M_{m,\tau}$ stores iterates {\small $\left\{W_{m}^{(\tau,t)}: t \in \{0\} \cup \left[T_{m}^{(1)}+T_{m}^{(2)}\right]\right\}$}(or {\scriptsize $\left\{w_{1}^{(t)}:t\in \left[T_{m}^{(1)}+T_{m}^{(2)}\right]\right\}$}), $\widehat{M}_{m,p}$ stores iterates {\small $\left\{\widehat W_{m,p}^{(t)}: t \in \{0\} \cup \left[T_{m}^{(3)}\right]\right\}$}, and $T_{m}^{(1)}\coloneqq \sum_{h\in[h^{*}]}T_{m,h}^{(1)}$.
	
	\begin{algorithmic}[1]
		\STATE \textbf{Phase~I.}
		For each $\tau\in\left[2^{2k+1}\right]$, run Algorithm~\ref{SGA-con-ss-odd} for $h^{*}$ cycles with schedule $\left\{\left(\eta_{m,h}^{(1)},T_{m,h}^{(1)}\right)\right\}_{h\in[h^{*}]}$ and reward $\widehat{\mathcal{R}}$, updating iterates $W_{m}^{(\tau,t)}$ (or $w_{1}^{(\tau,t)}$) in memory $M_{m,\tau}$ for each block $m\in[k+1]$.
		
		\STATE \textbf{Phase~II.}
		Continue each run with schedule $\left(\eta_{m}^{(2)},T_{m}^{(2)}\right)$ and reward $\widehat{\mathcal{R}}$ to update iterates  $W_{m}^{(\tau,t)}$ (or $w_{1}^{(\tau,t)}$) in memories $M_{m,\tau}$ for each block $m\in[k+1]$, producing $\breve{w}_{1}^{(\tau)}\bigcup\left\{\breve W_{m}^{(\tau)}\right\}$ at the final step.
		
		\STATE \textbf{Select Initialization.}
		For each $\tau\in\left[2^{2k+1}\right]$, draw $N_{0}$ fresh samples $\mathbf{T}^{(\tau,i)}$ and compute the scores {\scriptsize $S_{\tau}=\frac{1}{N_{0}}\sum_{i=1}^{N_{0}}\left\langle \breve{w}_{1}^{(\tau)}\otimes\bigotimes_{m=2}^{k+1} \breve W_{m}^{(\tau)},\mathbf T^{(\tau,i)}\right\rangle$}, where the inner product can be evaluated using the MLE gradient oracle by treating $\breve{W}_{k+1}$ as the parameter (see Section~\ref{sec-3.1}). Let $\tau^{*}=\mathop{\arg\max}_{\tau} S_{\tau}$, and set $\widehat{w}_{1}^{(0)}=\breve w_{1}^{(\tau^{*})}$ and $\widehat W_{m}^{(0)}=\breve W_{m}^{(\tau^{*})}$ for $m\in[2:k+1]$.
		
		\STATE \textbf{Phase~III.}
		For each block $m\in[k+1]$, run Algorithm~\ref{SGA-decay-ss-odd} twice from {\scriptsize $\widehat W_{m}^{(0)}$} (or {\scriptsize $\widehat{w}_{1}^{(0)}$}) with schedule {\scriptsize $\left(\eta_{m}^{(3)},T_{m}^{(3)}\right)$} and step size decay lengths {\scriptsize $\widehat T_{m}^{(3)}$} to produce two sequences {\scriptsize $\left\{\widehat W_{m,1}^{(t)}\right\}$} and {\scriptsize $\left\{\widehat W_{m,2}^{(t)}\right\}$} (or {\scriptsize $\left\{\widehat w_{1,1}^{(t)}\right\}$} and {\scriptsize $\left\{\widehat w_{1,2}^{(t)}\right\}$}) in memory $M_{m,1}$ and $M_{m,2}$, corresponding to rewards $+\widehat{\mathcal{R}}$ and $-\widehat{\mathcal{R}}$, respectively, and set the final weight {\scriptsize $\widehat{w}_{1} \gets \mathop{\arg\min}_{w\in \left\{\overline{\widehat{w}}_{1,1}^{\left(T_{1}^{(3)}\right)},\overline{\widehat{w}}_{1,2}^{\left(T_{1}^{(3)}\right)}\right\}} \left\|w-\widehat w_{1}^{(0)}\right\|_{2}, \widehat{W}_{m} \gets \mathop{\arg\min}_{W\in \left\{\overline{\widehat{W}}_{m,1}^{\left(T_{m}^{(3)}\right)},\overline{\widehat{W}}_{m,2}^{\left(T_{m}^{(3)}\right)}\right\}} \left\|W-\widehat W_{m}^{(0)}\right\|_{\mathrm{F}}$}.
		
		\STATE \textbf{Estimator.}
		Let $\widehat{v}_{1}=\widehat{w}_{1}$ and for each $m\in[2:k+1]$, compute {\scriptsize $\widehat v_{2m-1}=\mathrm{top\text{-}eig}\left(\widehat W_{m}^{}\widehat W_{m}^\top\right)$ and $\widehat v_{2m}=\mathrm{top\text{-}eig}\left(\widehat W_{m}^\top\widehat W_{m}^{}\right)$}.
		\RETURN Estimator $V = \{v_n\}_{n=1}^{\overline{k}}$.
	\end{algorithmic}
\end{algorithm}

\subsection{Hyper‑parameter Configuration for Practical Implementation}\label{para-setting-empirical-odd}
\begin{algorithm}[ht] \caption{Reference Parameter Search (Odd Case)} \label{search-c_3-odd}
	\footnotesize
	\textbf{Input:} Non-negative integer $\tau$, Maximum iteration $\widetilde{T}$.
	
	\textbf{Output:} Reference parameter $\mathsf{c}_3$. 
	
	\begin{algorithmic}[1]
        \STATE Initialize $\mathsf{c}_3 \leftarrow \texttt{null}$.
		\FOR{$\tau = 0$ to $\widehat{T}$}
        \STATE Set $N_1^{(\tau)} \gtrsim \kappa^{-(2k+2)\tau}d_1\prod_{m=2}^{k+1} d_{2m-2}$.
		\STATE Sample $N_1^{(\tau)}$ fresh scores $\{\mathrm{score}_{i}\}_{i\in[N_1^{(\tau)}]}$ which are defined in Eq.~\eqref{esti-cor-event}.
		\IF{$\frac{1}{N_1^{(\tau)}}\left|\sum_{i=1}^{N_1^{(\tau)}}\mathrm{score}_i\right|\geq\kappa^{(k+1)\tau}$}
		\STATE Let $\mathsf{c}_3\leftarrow\kappa^{\tau}$.
		\STATE \textbf{Break}.
		\ENDIF
		\ENDFOR
	\end{algorithmic}
\end{algorithm}
Consider Problem \ref{PCA-def} with SNR $\lambda\asymp 1$ and select a decay factor $\kappa\in(0,1)$. For each integer $\tau \in \{0,1,\dots,\widetilde{T}\}$, where
$$
\widetilde{T}=\left\lceil\left.\left(\frac{1}{2}-\frac{3}{2k+2}\right)\log\left(\prod_{n=1}^{\bar{k}}d_n^{1/\bar{k}}\right)\right/\log\left(\kappa^{-1}\right)\right\rceil,
$$
we proceed as follows. Given a pre-processing sample size $N_1^{(\tau)}$ satisfying $N_{1}^{(\tau)}\gtrsim \kappa^{-(2k+2)\tau}d_{1}\prod_{m=2}^{k+1}d_{2m-2}$, we collect $N_{1}^{(\tau)}$ scores $\left\{\mathrm{score}_i\right\}$ in an online manner and check whether the following event holds:
\begin{align}\label{esti-cor-event-odd}
	\left|\frac{1}{N_{1}^{(\tau)}}\sum_{i=1}^{N_{1}^{(\tau)}}\underbrace{\left\langle w_{1}^{(0)}\otimes\bigotimes_{m=2}^{k+1}\begin{pmatrix} I_{d_{2m-2}} & 0_{d_{2m-2}\times(d_{2m-1}-d_{2m-2})}\end{pmatrix},\upT^{(i)}\right\rangle}_{\mathrm{score}_i}\right|\geq\kappa^{(k+1)\tau},
\end{align}
where $\frac{1}{N_{1}^{(\tau)}}\left|\sum_{i=1}^{N_{1}^{(\tau)}}\mathrm{score}_i\right|$ denotes the estimate for $\left|\left\langle w_{1}^{(0)},v_{1}^{*}\right\rangle\right|\cdot\prod_{m=2}^{k+1}\left|\Cor\left(v_{2m-2}^{*},v_{2m-1}^{*}\right)\right|$. A complete specification of the algorithm is presented in Algorithm~\ref{search-c_3-odd}.

\paragraph{Case I:}If the first $\tau_{1}$ is found such that the event in Eq.~\eqref{esti-cor-event-odd} holds, we let $\mathsf{c}_3:=\kappa^{\tau_{1}}, d_0=d_{1}$ and define the effective block SNR ratio estimators for \textbf{Phase~I} and \textbf{Phase~II} as
\begin{equation}\label{esti-block-eff-SNR-caseI-odd}
	\begin{gathered}
		\widetilde{\lambda}_{m,h_{1}}^{\mathrm{I}} \coloneqq \frac{\lambda\mathsf{c}_3}{4^{k}}\prod_{i\in[k+1]}^{i\neq m}\widetilde{\gamma}_i=\frac{\lambda\mathsf{c}_3^{k+1}}{4^{k}}\prod_{i\in[k+1]}^{i\neq m}\left(d_{2i-2}d_{2i-1}\right)^{-\frac{1}{4}},
		\\ \widetilde{\lambda}_{m,h_{2}}^{\mathrm{I}} \coloneqq \lambda\prod_{i\in[k+1]}^{i\neq m}\zeta_{i}^{h_{2}-1}\widetilde{\gamma}_i=\lambda\mathsf{c}_3^{k}\prod_{i\in[k+1]}^{i\neq m}\zeta_{i}^{h_{2}-1}\left(d_{2i-2}d_{2i-1}\right)^{-\frac{1}{4}},
		\\
		\widetilde\lambda_{m}^{\mathrm{II}} \coloneqq \lambda\left(1-\frac{1}{k+1}\right)^{k+1-m}\prod_{i=1}^{m-1}\left(d_{2i-2}d_{2i-1}\right)^{-\frac{1}{4}+\frac{1}{4k}},
	\end{gathered}
\end{equation}
for any $m\in\left[k+1\right]$, $h_{1}\in[h_{1}^{*}]$ and $h_{2}\in[h_{2}^{*}]$, where  $\zeta_{m}=\exp\left\{\left(\log(\mathsf{c}_3^{-1})+\frac{1}{4k}\log(d_{2m-2}d_{2m-1})\right)/h_{2}^{*}\right\}$, $\widetilde{\gamma}_{m}=\frac{\mathsf{c}_3}{\left(d_{2m-2}d_{2m-1}\right)^{1/4}}$, and $h_{1}^{*} \coloneqq \left\lceil\log(4)+k\log\left(\mathsf{c}_3^{-1}\right)\right\rceil$, and $h_{2}^{*} \coloneqq \left\lceil\log\left(\mathsf{c}_3^{-1}\right)+\frac{1}{4k}\log\left(d_{2k}d_{2k+1}\right)\right\rceil$. In this case, we divide \textbf{Phase~I} into two stages, where the required sample complexity terms are given by
\begin{gather}\label{empirical-sample-size-caseI-odd}
	T_{m,h_{1}}^{(1)}\asymp\frac{\mathrm{UB}_{m,h_{1}}^{\mathrm{I}}}{\widetilde{\lambda}_{m,h_{1}}^{\mathrm{I}}\eta_{m,h_{1}}^{(1)}},\quad T_{m,h_{2}}^{(1)}\asymp\frac{\mathrm{UB}_{m,h_{2}}^{\mathrm{I}}}{\widetilde{\lambda}_{m,h_{2}}^{\mathrm{I}}\eta_{m,h_{2}}^{(1)}},\quad T_{m}^{(2)}\asymp\frac{1}{\widetilde{\lambda}_{m}^{\mathrm{II}}\eta_{m}^{(2)}},\notag
    \\
    T_{m}^{(3)}\equiv T^{(3)}\gtrsim\log^4\left(\frac{kd_{2k+1}T^{(3)}}{\delta}\right)\cdot\frac{(kd_{2k}d_{2k+1})^{2}}{\lambda^{2}},
\end{gather}
and the corresponding step-sizes are chosen as
\begin{equation}\label{empirical-set-size-caseI-odd}
	\begin{gathered}
		\eta_{1,h_{1}}^{(1)}\asymp\frac{1}{\mathsf{c}_2k\left(h_{1}^{*}\right)^{2}}\cdot\min\left\{\frac{\widetilde{\lambda}_{1,h_{1}}^{\mathrm{I}}}{\mathrm{UB}_{1,h_{1}}^{\mathrm{I}}d_{1}},\mathrm{UB}_{1,h_{1}}^{\mathrm{I}}\widetilde{\lambda}_{1,h_{1}}^{\mathrm{I}}\right\},\quad \eta_{1,h_{2}}^{(1)}\asymp\frac{1}{\mathsf{c}_2k\left(h_{2}^{*}\right)^{2}}\cdot\frac{\widetilde{\lambda}_{1,h_{2}}^{\mathrm{I}}}{\mathrm{UB}_{1,h_{2}}^{\mathrm{I}}d_{1}},\quad
        \\
        \eta_{1}^{(2)}\asymp\frac{\widetilde{\lambda}_{1}^{\mathrm{II}}}{\mathsf{c}_2kd_{1}},\quad \eta_{1}^{(3)}\asymp\frac{\log^2\left(T^{(3)}\right)}{\lambda T^{(3)}},
		\\
		\eta_{m,h_{1}}^{(1)}\asymp\frac{1}{\mathsf{c}_2k\left(h_{1}^{*}\right)^{2}}\cdot\min\left\{\frac{\widetilde{\lambda}_{m,h_{1}}^{\mathrm{I}}}{\mathrm{UB}_{m,h_{1}}^{\mathrm{I}}d_{2m-2}d_{2m-1}},\mathrm{UB}_{m,h_{1}}^{\mathrm{I}}\widetilde{\lambda}_{m,h_{1}}^{\mathrm{I}}\right\},\quad \eta_{m,h_{2}}^{(1)}\asymp\frac{1}{\mathsf{c}_2k\left(h_{2}^{*}\right)^{2}}\cdot\frac{\widetilde{\lambda}_{m,h_{2}}^{\mathrm{I}}}{\mathrm{UB}_{m,h_{2}}^{\mathrm{I}}d_{2m-2}d_{2m-1}},
		\\
		\eta_{m}^{(2)}\asymp\frac{\widetilde{\lambda}_{m}^{\mathrm{II}}}{\mathsf{c}_2kd_{2m-2}d_{2m-1}},\quad \eta_{m}^{(3)}\asymp\frac{\log^2\left(T^{(3)}\right)}{\lambda T^{(3)}},
	\end{gathered}
\end{equation}
where $\mathsf{c}_{2}\asymp k\log(kd_{2k+1}/\delta)$,  $\mathrm{UB}_{m,h_{1}}^{\mathrm{I}}\coloneqq e^{h_{1}}\mathsf{c}_3^{k+1}\left(d_{2m-2}d_{2m-1}\right)^{-1/4}/4$ and $\mathrm{UB}_{m,h_{2}}^{\mathrm{II}} \coloneqq \mathsf{c}_3\zeta_{m}^{h_{2}}\left(d_{2m-2}d_{2m-1}\right)^{-1/4}$ for any $m\in[k+1]$, $h_{1}\in[h_{1}^{*}]$ and $h_{2}\in[h_{2}^{*}]$.

\paragraph{Case II:}If there is not a $\tau$ such that the event in Eq.~\eqref{esti-cor-event-odd} holds, we let $d_0=1$ and define the effective block SNR ratio estimators for \textbf{Phase~I} and \textbf{Phase~II} as
\begin{equation}\label{esti-block-eff-SNR-caseII-odd}
	\begin{gathered}
		\widetilde{\lambda}_{m,h}^{\mathrm{I}}:=\frac{\lambda}{4^{k}}\prod_{i\in[k+1]}^{i\neq m}\zeta_{i}^{h-1}\widetilde{\gamma}_i,\quad
%		=\frac{\lambda\mathsf{c}_0^{k-1}}{4^{k}}\prod_{i\in[k]}^{i\neq m}\zeta_{i}^{h-1}\left(d_{2i-1}d_{2i}\right)^{-\frac{1}{2}},
	 	\widetilde\lambda_{m}^{\mathrm{II}} \coloneqq \lambda\left(1-\frac{1}{k+1}\right)^{k+1-m}\prod_{i=1}^{m-1}\left(d_{2i-2}d_{2i-1}\right)^{-\frac{1}{2}+\frac{1}{2k}},
	\end{gathered}
\end{equation}
for any $m\in[k+1]$ and $h\in[h^{*}]$, where $\zeta_{m}=\exp\left\{\left(\log(4/\mathsf{c}_0)+\frac{1}{2k}\log(d_{2m-1}d_{2m})\right)/h^{*}\right\}$, $\widetilde{\gamma}_{m}=$\linebreak $\mathsf{c}_0\left(d_{2m-1}d_{2m}\right)^{-1/2}$ for $m\in[2:k+1]$ and $h^{*}\coloneqq  \left\lceil\log\left(4/\mathsf{c}_0\right)+\frac{1}{2k}\log\left(d_{2k}d_{2k+1}\right)\right\rceil$ (with $\zeta_{1}=$\linebreak $\exp\left\{\left(\log(2/\mathsf{c}_0^{1/2})+\frac{1}{2k}\log(d_{1})\right)/h^{*}\right\}$, $\widetilde{\gamma}_{1}=\mathsf{c}_0^{1/2}d_{1}^{-1/2}$). In this case, the required sample complexity terms are given by
\begin{align}\label{empirical-sample-size-caseII-odd}
	T_{m,h}^{(1)}\asymp\frac{\mathrm{UB}_{m,h}^{\mathrm{I}}}{\widetilde{\lambda}_{m,h}^{\mathrm{I}}\eta_{m,h}^{(1)}},\quad T_{m}^{(2)}\asymp\frac{1}{\widetilde{\lambda}_{m}^{\mathrm{II}}\eta_{m}^{(2)}},\quad T_{m}^{(3)}\equiv T^{(3)}\gtrsim \log^4\left(\frac{kd_{2k+1}T^{(3)}}{\delta}\right)\cdot\frac{(kd_{2k}d_{2k+1})^{2}}{\lambda^{2}},
\end{align}
and the corresponding step-sizes are chosen as
\begin{equation}\label{empirical-set-size-caseII-odd}
	\begin{gathered}
		\eta_{m,h}^{(1)}\asymp\frac{1}{\mathsf{c}_2\left(h^{*}\right)^{2}}\cdot\frac{\widetilde{\lambda}_{m,h}^{\mathrm{I}}}{\mathrm{UB}_{m,h}^{\mathrm{I}}d_{2m-2}d_{2m-1}},\quad \eta_{m}^{(2)}\asymp\frac{\widetilde{\lambda}_{m}^{\mathrm{II}}}{\mathsf{c}_2kd_{2m-2}d_{2m-1}},\quad \eta_{m}^{(3)}\asymp\frac{\log^2\left(T^{(3)}\right)}{\lambda T^{(3)}},
	\end{gathered}
\end{equation}
where $\mathsf{c}_{2}\asymp k\log(kd_{2k+1}/\delta)$ and $\mathrm{UB}_{m,h}^{\mathrm{I}} \coloneqq \mathsf{c}_0\zeta_{m}^h\left(d_{2m-1}d_{2m}\right)^{-1/2}/4$ for any $m\in[k+1]$ and $h\in[h^{*}]$.

\subsection{Proof for the Main Results in Section \ref{extension}}

\begin{proof}[Proof of Theorem \ref{main-theorem-odd}]
    The proof of Theorem \ref{main-theorem-odd} is analogous to that of Theorem \ref{main-theorem}, with the additional consideration of the vector parameter $w_1$. By treating $w_1$ as a matrix of dimension $d_1\times 1$, the proof follows directly from the same technique used for the matrix parameters in even case.
\end{proof}

\begin{proof}[Proof of Corollary \ref{cor-equal-dim-odd}]
    Suppose the hyper-parameter configurations follow those in Theorem \ref{main-theorem-odd}. Algorithm \ref{MSANSGA-odd} achieves strong recovery of each $v_n^*$ provided the memory–runtime–sample complexity satisfies
    \begin{align*}
        \lambda^2NKS \gtrsim 4^{\overline{k}}\overline{k}^{3}\log\left(\frac{kd}{\delta}\right)\cdot&\max\left\{\max_{m'\in[2:\lceil\overline{k}/2\rceil]}\left(\left|\llangle v_{2m'-2}^*,v_{2m'-1}^*\rrangle\right|d^{1/2}+\mathsf{c}_0\right)^4,\,  \mathsf{c_0^2}\right\}\notag
        \\
        \cdot&\prod_{m=1}^{\lceil\overline{k}/2\rceil}\left(\frac{d^{1/2}}{\left|\llangle v_{2m-2}^*,v_{2m-1}^*\rrangle\right|+\mathsf{c}_0d^{-1/2}}\right)^2,
    \end{align*}
    with probability at least $1-2\delta$.

    Under the step-sizes setting in Theorem \ref{main-theorem-odd}, if the memory–runtime–sample complexity satisfies\linebreak $\lambda^2K\sum_{m=1}^{k+1}T_{m,1}^{(1)}S_m\gtrsim 4^{\overline{k}}\overline{k}^3\log^3\left(\frac{kd}{\delta}\right)\cdot\mathsf{c}_0^{2-\overline{k}}\cdot d^{\overline{k}}$ ($S_m$ denotes the memory cost for block $m$), we instead adopt the effective block SNR ratio estimators $\left\{\widetilde{\lambda}_{m,h}^{\mathrm{I}}\right\}_{m\in[k+1],h\in[h^{*}]}$ and $\left\{\widetilde{\lambda}_{m}^{\mathrm{II}}\right\}_{m\in[k+1]}$ for \textbf{Phase I} and \textbf{Phase II}, which are specified in Eq.~\eqref{esti-block-eff-SNR-caseII-odd}, and choose the required sample complexity terms as
	\begin{align*}
		T_{m,h}^{(1)}\asymp\frac{\mathrm{UB}_{m,h}^{\mathrm{I}}}{{\lambda}_{m,h}^{\mathrm{I}}\eta_{m,h}^{(1)}},\quad T_{m}^{(2)}\asymp\frac{1}{{\lambda}_{m}^{\mathrm{II}}\eta_{m}^{(2)}},\quad T_{m}^{(3)}\equiv T^{(3)}\gtrsim \log^4\left(\frac{kdT^{(3)}}{\delta}\right)\cdot\frac{k^2d^4}{\lambda^2},
	\end{align*}
	for $m\in[k+1]$, the corresponding step-sizes as
	\begin{gather*}
        \eta_{1,h}^{(1)}\asymp\frac{1}   {\mathsf{c}_{2}\left(h^{*}\right)^2}\cdot\min\left\{\frac{\widetilde{\lambda}_{1,h}^{\mathrm{I}}}{\mathrm{UB}_{1,h}^{\mathrm{I}}d}, \frac{\mathrm{LB}_{1,h}^{\rmI}}{\lambda_{1,h}^{\mathrm{I}}+1}\right\},\quad \eta_{1}^{(2)}\asymp\frac{1}{\mathsf{c}_2}\cdot\min\left\{\frac{\widetilde{\lambda}_{1}^{\mathrm{II}}}{kd}, \frac{1}{{\lambda}_{1}^{\mathrm{II}}+1}\right\},\quad \eta_{1}^{(3)}\asymp\frac{\log^2\left(T^{(3)}\right)}{\lambda T^{(3)}},
        \\
		\eta_{m,h}^{(1)}\asymp\frac{1}{\mathsf{c}_{2}\left(h^{*}\right)^2}\cdot\min\left\{\frac{\widetilde{\lambda}_{m,h}^{\mathrm{I}}}{\mathrm{UB}_{m,h}^{\mathrm{I}}d^2}, \frac{\mathrm{LB}_{m,h}^{\rmI}}{\lambda_{m,h}^{\mathrm{I}}+1}\right\},\quad \eta_{m}^{(2)}\asymp\frac{1}{\mathsf{c}_2}\cdot\min\left\{\frac{\widetilde{\lambda}_{m}^{\mathrm{II}}}{kd^2}, \frac{1}{{\lambda}_{m}^{\mathrm{II}}+1}\right\},\quad \eta_{m}^{(3)}\asymp\frac{\log^2\left(T^{(3)}\right)}{\lambda T^{(3)}},
	\end{gather*}
	for $m\in[2:k+1]$, where
	\begin{align*}
		\mathsf{c}_{2}\asymp k\log(kd/\delta''),\quad \mathrm{LB}_{m,h}^{\mathrm{I}}\coloneqq\frac{\zeta_{m}^{h-1}\overline{\gamma}_{m}}{2}, \quad
		\mathrm{UB}_{m,h}^{\mathrm{I}}\coloneqq\zeta_{m}^{h}\overline{\gamma}_{m}, \quad \zeta_{m}=\exp\left\{\frac{\log(4)+\log\left(\breve{\gamma}_{m}/\overline{\gamma}_{m}\right)}{h^{*}}\right\},
	\end{align*} 
    and set $h^{*}=\text{$\left\lceil\log(4)+\max_{m}\left\{\log\left(\breve{\gamma}_{m}/\overline{\gamma}_{m}\right)\right\}\right\rceil$}$,  $\overline{\gamma}_{m}=\mathsf{c}_0d^{-1}/4$ ($\overline{\gamma}_{1}=\mathsf{c}_0^{1/2}d^{-1/2}/2$), and $\breve{\gamma}_{m}=d^{-1+1/k}$ ($\breve{\gamma}_{1}=d^{-1/2+1/2k}$) for any $m\in[2:k+1]$. Then, combining the odd-case extensions of Corollaries \ref{coro-gen-phase-I} and \ref{coro-gen-phase-II} with that of Theorem \ref{thm-phase-III-tensor-PCA-sub-Gaussian} and letting $\delta''\leq\delta/8$, for any memory–runtime–sample complexity satisfying
    $$
    \lambda^2NKS\gtrsim 4^{\overline{k}}\overline{k}^3\log^3\left(\frac{kd}{\delta}\right)\cdot\mathsf{c}_0^{2-\overline{k}}\cdot d^{\overline{k}},
    $$
    Algorithm \ref{MSANSGA-odd} achieves strong recovery of each $v_n^*$ with probability at least $1-2\delta$.
\end{proof}

\begin{proof}[Proof of Corollary \ref{empirical-hyper-para-convergence-odd}]
    The proof is similar to that of Corollary \ref{empirical-hyper-para-convergence}.
\end{proof}

\section{Proof of Technical Lemmas}\label{tech-lem}
\begin{proof}[Proof of Lemma~\ref{lemma:high-prob}]
    Recall that $E_{m}^{(t+1)}$ is linear in the noise tensor $\mathbf E^{(t+1)}$.
    By the definition of $E_{m}^{(t+1)}$, for any matrix $Q\in\mathbb R^{d_{2m-1}\times d_{2m}}$, $\left\langle E_{m}^{(t+1)},Q\right\rangle$ is a sum of $\prod_{i=1}^{k} d_{i}$ independent mean-zero random variables.
    
    Conditioned on $\mathcal{F}_{t-1}$, by Assumption~\ref{assumption-Gaussian}, each entry of $\mathbf{E}_{m}^{(t+1)}$ is $\sigma$-sub-Gaussian ($\sigma=1$), and hence $\left\langle E_{m}^{(t+1)},Q\right\rangle$ is sub-Gaussian with variance proxy
    \begin{equation}
        \sigma^{2} \left\| \overline{W}_{1}^{(0)}\otimes \cdots \otimes \overline{W}_{m-1}^{(0)} \otimes Q \otimes \overline{W}_{m+1}^{(T)}\otimes \cdots \otimes \overline{W}_{k}^{(T)}\right\|_{\mathrm{F}}^{2}.
    \end{equation}
    Using the identity $\|A\otimes B\|_{\mathrm{F}}=\|A\|_{\mathrm{F}}\|B\|_{\mathrm{F}}$, we obtain
    \begin{equation}
        \left\| \overline{W}_{1}^{(0)}\otimes \cdots \otimes \overline{W}_{m-1}^{(0)} \otimes Q \otimes \overline{W}_{m+1}^{(T)}\otimes \cdots \otimes \overline{W}_{k}^{(T)}\right\|_{\mathrm{F}} =\|Q\|_{\mathrm{F}} .
    \end{equation}
    Therefore, for any $u>0$,
    \begin{equation}
        \mathbb{P}\left(\left|\left\langle E_{m}^{(t+1)},Q\right\rangle\right| > 2\sigma\|Q\|_{\mathrm{F}}\sqrt u\,\middle|\, \mathcal F_{t-1}
        \right)\leq e^{-u}.
    \end{equation}
    Taking expectation over $\mathcal F_{t-1}$ yields
    \begin{equation}
        \mathbb{P}\left(
        \left|\left\langle E_{m}^{(t+1)},Q\right\rangle\right|
        >2\sigma\|Q\|_{\mathrm{F}}\sqrt u
        \right)
        \leq e^{-u}.
    \end{equation}
    
    Applying the above bound with $Q=V^{(i,j)}$, where
    $V^{(i,j)}_{i',j'}=\mathds{1}\{i'=i,j'=j\}$,
    and taking a union bound over all
    $t\in[T]$, $m\in[k]$, $i\in[d_{2m-1}]$, and $j\in[d_{2m}]$,
    we obtain that with probability at least
    \begin{equation}
        1-T\sum_{i=1}^{k} d_{2i-1}d_{2i}\,e^{-u},
    \end{equation}
    the following holds simultaneously for all $t,m,i,j$,
    \begin{equation}
        \left|\left(E_{m}^{(t+1)}\right)_{i,j}\right|
        \leq 2\sigma\sqrt u.
    \end{equation}
    
    Moreover, applying the same bound with $Q=v_{2m-1}^{*}v_{2m}^{*\top}$ and $Q=W_m^{(t)}$, and taking an additional union bound over all $t\in[T]$ and $m\in[k]$, we obtain that with probability at least
    \begin{equation}
        1-\left(T\sum_{i=1}^{k} d_{2i-1}d_{2i}+2Tk\right)\,e^{-u},
    \end{equation}
    the following bounds hold simultaneously for all $t\in[T]$ and $m\in[k]$,
    \begin{align}
        &\left|\left\langle E_{m}^{(t+1)},v_{2m-1}^{*}v_{2m}^{*\top}\right\rangle\right| \leq 2\sigma\sqrt{u}, \\
        &\left|\left\langle E_{m}^{(t+1)},W_m^{(t)}\right\rangle\right| \leq 2\sigma\|W_m^{(t)}\|_{\mathrm{F}}\sqrt{u}.
    \end{align}
    
    Choosing
    \begin{equation}
        u=\log\left(\frac{T\sum_{i=1}^{2k} d_{2i-1}d_{2i}+2Tk}{\delta'}\right)
    \end{equation}
    completes the proof.
\end{proof}

\begin{proof}[Proof of Lemma \ref{control-expectation}]
	Given $m\in[k]$ and $\mathrm{Phase}\in\{\mathrm{I},\mathrm{II},\mathrm{III}\}$, for any $t\in[0:T-1]$, we define that event $\mathcal{A}_{1}^{(t+1)}(\delta')$ has the form of
	\begin{align*}
		\mathcal{A}_{1}^{(t+1)}(\delta'):=\left\{ \forall (i,j)\in [d_{2m-1}]\times[d_{2m}], ~~\left| \left(E_{m,\mathrm{Phase}}^{(t+1)}\right)_{i,j} \right| \leq \sqrt{\mathsf{c}_{1}}  \right\},
	\end{align*}
	where $\mathsf{c}_1=4\sigma^2\log\left(\frac{T\sum_{i=1}^{k}d_{2i-1}d_{2i}+2Tk}{\delta'}\right)$. Moreover, event $\mathcal{A}_{2}^{(t+1)}(\delta')$ can be decomposed into $\mathcal{A}_{2}^{(t+1)}(\delta')=\mathcal{A}_{2,1}^{(t+1)}(\delta')\cap\mathcal{A}_{2,2}^{(t+1)}(\delta')$, where 
	\begin{align}
		&\mathcal{A}_{2,1}^{(t+1)}(\delta'):=\left\{\left|\left\langle E_{m,\mathrm{Phase}}^{(t+1)},v_{2m-1}^{*}v_{2m}^{*\top}\right\rangle\right|\leq\sqrt{\mathsf{c}_{1}}\right\},\notag \\
		&\mathcal{A}_{2,2}^{(t+1)}(\delta'):=\left\{\left|\left\langle E_{m,\mathrm{Phase}}^{(t+1)},W_{m}^{(t)}\right\rangle\right|\leq\sqrt{\mathsf{c}_{1}}\left\|W_{m}^{(t)}\right\|_{\mathrm{F}}\right\}.\notag
	\end{align}
	According to Assumption \ref{ass-base}, one can notice that $\frac{1}{\left\|Q\right\|_{\mathrm{F}}}\left\langle E^{(t)},Q\right\rangle$ (defined in Eq.~\eqref{eq:errort}) is sub-Gaussian with parameter $\sigma$ for any fixed matrix $Q\in\mathbb{R}^{d\times d}$. Therefore, we have 
	\begin{align}
		&\phantom{\mathrel{=}}\left|\mathbb{E}_t\left[\left\langle E_{m,\mathrm{Phase}}^{(t+1)}\cdot\mathds{1}_{\mathcal{A}_{1}^{(t+1)}(\delta')\cap\mathcal{A}_{2}^{(t+1)}(\delta')},v_{2m-1}^{*}v_{2m}^{*\top}\right\rangle\right]\right|\notag \\
		&=\left|\mathbb{E}_t\left[\left\langle E_{m,\mathrm{Phase}}^{(t+1)}\cdot\mathds{1}_{\mathcal{A}_{1}^{(t+1)}(\delta')\cap\mathcal{A}_{2}^{(t+1)}(\delta')},v_{2m-1}^{*}v_{2m}^{*\top}\right\rangle\right]-\mathbb{E}_t\left[\left\langle E_{m,\mathrm{Phase}}^{(t+1)},v_{2m-1}^{*}v_{2m}^{*\top}\right\rangle\right]\right|\notag \\
		&\leq\underbrace{\left|\mathbb{E}_t\left[\left\langle E_{m,\mathrm{Phase}}^{(t+1)}\cdot\mathds{1}_{\mathcal{A}_{2,1}^{(t+1)}(\delta')},v_{2m-1}^{*}v_{2m}^{*\top}\right\rangle\right]-\mathbb{E}_t\left[\left\langle E_{m,\mathrm{Phase}}^{(t+1)},v_{2m-1}^{*}v_{2m}^{*\top}\right\rangle\right]\right|}_{\mathcal{I}_{2,t+1}}\notag \\
		&\phantom{\mathrel{=}}+\underbrace{\left|\mathbb{E}_t\left[\left\langle E_{m,\mathrm{Phase}}^{(t+1)}\cdot\mathds{1}_{\left(\mathcal{A}_{1}^{(t+1)}(\delta')\right)^{c}\cap\mathcal{A}_{2,1}^{(t+1)}(\delta')},v_{2m-1}^{*}v_{2m}^{*\top}\right\rangle\right]\right|}_{\mathcal{II}_{2,t+1}}\notag \\
		&\phantom{\mathrel{=}}+\underbrace{\left|\mathbb{E}_t\left[\left\langle E_{m,\mathrm{Phase}}^{(t+1)}\cdot\mathds{1}_{\left(\mathcal{A}_{2,2}^{(t+1)}(\delta')\right)^c\cap\mathcal{A}_{1}^{(t+1)}(\delta')\cap\mathcal{A}_{2,1}^{(t+1)}(\delta')},v_{2m-1}^{*}v_{2m}^{*\top}\right\rangle\right]\right|}_{\mathcal{III}_{2,t+1}}.\label{tech-1}
	\end{align}
	Based on Lemma \ref{aux-truncation-subGaussian-covariance} and the setting that $T\geq256\max\{\sigma^{2},\sigma^{-2}\}(\sigma^{2}+K)$, the setting of $\mathsf{c}_1$ yields $\mathcal{I}_{2,t+1}\leq\sqrt{\delta'}$. Utilizing Cauchy-Schwartz inequality and Proposition \ref{prop-A5}, we obtain
	\begin{align}\label{tech-2}
		\mathcal{II}_{2,t+1}\leq&\sup_{E_{m,\mathrm{Phase}}^{(t+1)}}\left|\left\langle E_{m,\mathrm{Phase}}^{(t+1)}\mathds{1}_{\mathcal{A}_{2,1}^{(t+1)}(\delta')},v_{2m-1}^{*}v_{2m}^{*\top}\right\rangle\right| \cdot \left[1-\mathbb{P}\left(\mathcal{A}_{1}^{(t+1)}(\delta')\right)\right]\notag
		\\
		{\leq}&\sqrt{\mathsf{c}_1}\cdot\left[1-\mathbb{P}\left(\mathcal{A}_{1}^{(t+1)}(\delta')\right)\right]{\leq}\sqrt{\delta'},
%		\mathcal{II}_{2,t}\leq\left(\mathbb{E}_t\left[\left\langle v_{2m-1}^{*},E_{m,\mathrm{Phase}}^{(t+1)}v_{2m}^{*}\right\rangle^{2}\cdot\mathds{1}_{\mathcal{A}_{1}^{(t+1)}(\delta')}\right]\right)^{\frac{1}{2}}\leq\sigma\delta'^{\frac{1}{4}}.
	\end{align}
	Finally, $\mathcal{III}_{2,t+1}$ satisfies
	\begin{align}\label{tech-3}
		\mathcal{III}_{2,t+1} &\leq \sup_{E_{m,\mathrm{Phase}}^{(t+1)}}\left|\left\langle E_{m,\mathrm{Phase}}^{(t+1)}\mathds{1}_{\mathcal{A}_{1}^{(t+1)}(\delta')\cap\mathcal{A}_{2,1}^{(t+1)}(\delta')},v_{2m-1}^{*}v_{2m}^{*\top}\right\rangle\right| \cdot \left[1-\mathbb{P}\left(\mathcal{A}_{2,2}^{(t+1)}(\delta')\right)\right] \notag \\
		&\overset{\text{(a)}}{\leq} \sqrt{\mathsf{c}_1}\cdot\left[1-\mathbb{P}\left(\mathcal{A}_{2,2}^{(t+1)}(\delta')\right)\right]\overset{\text{(b)}}{\leq}\sqrt{\delta'},
	\end{align}
	where (a) is derived from Cauchy-Schwartz inequality and (b) follows from Proposition \ref{prop-A5}. Therefore, if $\delta'\lesssim{\tau^{2}}/{\mathsf{c}_1}$, combining Eqs.~\eqref{tech-1}-\eqref{tech-3} can directly derive Eq.~\eqref{tech-lemma-2-1}. By employing a similar proof method, Eq.~\eqref{tech-lemma-2-2} can be further derived.
\end{proof}

\begin{proof}[Proof of Lemma \ref{estimation}]
	Recall the definition of $G_{m}^{(t)}$ and $E_{m}^{(t+1)}$, Utilizing the construction of $E_{m}^{(t+1)}$ provided in Eq.~\eqref{eq:errort} directly, we can obtain
	\begin{align}
		\mathcal{I}^{(t)}&:= \frac{\left|\left\langle G_{m}^{(t)}, v_{2m-1}^{*} v_{2m}^{*\top}\right\rangle\right|}{\left\|W_m^{(t)}\right\|_{\mathrm{F}}} \notag \\
        &\leq \lambda_{m}^{\mathrm{Phase}}\left(\alpha_{m}^{(t)}\right)^{2}+\alpha_{m}^{(t)}\left|\frac{\left\langle W_m^{(t)},E_m^{(t+1)}\right\rangle}{\left\|W_m^{(t)}\right\|_{\mathrm{F}}}\right|+\lambda_{m}^{\mathrm{Phase}}+\left|\left\langle v_{2m-1}^{*}v_{2m}^{*\top},E_m^{(t+1)}\right\rangle\right| \notag \\
        &\leq \lambda_{m}^{\mathrm{Phase}}\left(\alpha_{m}^{(t)}\right)^{2}+\lambda_{m}^{\mathrm{Phase}}+\sqrt{\mathsf{c}_1}\left(1+\alpha_{m}^{(t)}\right),\label{component-1-PCA-dynamic} \\
		\mathcal{II}^{(t)}&:=\frac{\left|\left\langle G_{m}^{(t)},W_m^{(t)}\right\rangle\right|}{\left\|W_m^{(t)}\right\|_{\mathrm{F}}^{2}}\leq2\lambda_{m}^{\mathrm{Phase}}\alpha_{m}^{(t)}+2\left|\frac{\left\langle{E}_{m}^{(t+1)},W_{m}^{(t)}\right\rangle}{\left\|W_m^{(t)}\right\|_{\mathrm{F}}^{2}}\right|\leq2\lambda_{m}^{\mathrm{Phase}}\alpha_{m}^{(t)}+2\sqrt{\mathsf{c}_1},\label{component-2-PCA-dynamic} \\
		\mathcal{III}^{(t)}&:=\frac{\left\|G_{m}^{(t)}\right\|_{\mathrm{F}}^{2}}{\left\|W_m^{(t)}\right\|_{\mathrm{F}}^{2}}\leq4\left(\lambda_{m}^{\mathrm{Phase}}\right)^{2}\left[\left(\alpha_{m}^{(t)}\right)^{2}+1\right]+4\left|\frac{\left\langle E_{m}^{(t+1)},W_m^{(t)}\right\rangle}{\left\|W_m^{(t)}\right\|_{\mathrm{F}}}\right|^{2}+4\left\|{E}_{m}^{(t+1)}\right\|_{\mathrm{F}}^{2}\notag \\
		&\leq 4\left(\lambda_{m}^{\mathrm{Phase}}\right)^{2}\left[\left(\alpha_{m}^{(t)}\right)^{2}+1\right]+4\mathsf{c}_1(d_{2m-1}d_{2m}+1).\label{component-3-PCA-dynamic}
	\end{align}
	According to Eq.~\eqref{component-2-PCA-dynamic}, it can be derived that
	\begin{align}\label{fmu}
		\frac{\left\|W_{m}^{(t)}+\eta_m G_{m}^{(t)}\right\|_{\mathrm{F}}^{2}}{\left\|W_{m}^{(t)}\right\|_{\mathrm{F}}^{2}}\geq1-2\eta_m\mathcal{II}^{(t)}\geq1-4\eta_m\left(\lambda_{m}^{\mathrm{Phase}}\alpha_{m}^{(t)}+\sqrt{\mathsf{c}_1}\right).
	\end{align}
	Combining the expression of $\Psi_1$ in Lemma \ref{lemma-expansion} with Eq.~\eqref{fmu}, we have
	\begin{align}
		&\phantom{\mathrel{=}} \left|\Psi_1(W_m^{(t)}, G_{m}^{(t)}, v_{2m-1}^{*}, v_{2m}^{*}, \overline{\eta}_m)\right| \notag \\
		&\leq \frac{2\mathcal{I}^{(t)}\left(\mathcal{II}^{(t)}+\overline{\eta}_m\mathcal{III}^{(t)}\right)+\left(\alpha_{m}^{(t)}+\overline{\eta}_m\mathcal{I}^{(t)}\right)\mathcal{III}^{(t)}}{\left[1-4\overline{\eta}_m\left(\lambda_{m}^{\mathrm{Phase}}\alpha_{m}^{(t)}+\sqrt{\mathsf{c}_1}\right)\right]^{3}} \notag \\
		&\phantom{\mathrel{=}}+ \frac{3\left(\alpha_{m}^{(t)}+\overline{\eta}_m\mathcal{I}^{(t)}\right)\left(\mathcal{II}^{(t)}+\overline{\eta}_m\mathcal{III}^{(t)}\right)^{2}}{\left[1-4\overline{\eta}_m\left(\lambda_{m}^{\mathrm{Phase}}\alpha_{m}^{(t)}+\sqrt{\mathsf{c}_1}\right)\right]^{3}} \notag \\
		&\leq \frac{8\left[\lambda_{m}^{\mathrm{Phase}}\left(1+\left(\alpha_{m}^{(t)}\right)^{2}\right)+\sqrt{\mathsf{c}_1}\left(1+\alpha_{m}^{(t)}\right)\right]\cdot\left(\lambda_{m}^{\mathrm{Phase}}\left(\alpha_{m}^{(t)}+1\right)+2\sqrt{\mathsf{c}_1}\right)}{\left[1-4\overline{\eta}_m\left(\lambda_{m}^{\mathrm{Phase}}\alpha_{m}^{(t)}+\sqrt{\mathsf{c}_1}\right)\right]^{3}} \notag \\
		&\phantom{\mathrel{=}}+ \frac{4\mathsf{c}_1\alpha_{m}^{(t)}\left(d_{2m-1}d_{2m}+1+2\left(\lambda_{m}^{\mathrm{Phase}}\right)^{2}\right)+12\alpha_{m}^{(t)}\cdot\left(\lambda_{m}^{\mathrm{Phase}}\alpha_{m}^{(t)}+2\sqrt{\mathsf{c}_1}\right)}{\left[1-4\overline{\eta}_m\left(\lambda_{m}^{\mathrm{Phase}}\alpha_{m}^{(t)}+\sqrt{\mathsf{c}_1}\right)\right]^{3}} \notag \\
		&\leq\frac{64\mathsf{c}_1\left[\left(\lambda_{m}^{\mathrm{Phase}}+2\right)^{2}+\alpha_{m}^{(t)}(d_{2m-1}d_{2m}+4)\right]}{\left[1-4\overline{\eta}_m\left(\lambda_{m}^{\mathrm{Phase}}\alpha_{m}^{(t)}+\sqrt{\mathsf{c}_1}\right)\right]^{3}}.
	\end{align}
\end{proof}

\section{Auxiliary Lemma}\label{aux-lem}
\begin{proposition}\label{initial-prop}[Corollary 4.4, \cite{ding2025near}]
	Given $v^*\in\bbR^d$ and  $\delta\in(0,1/2)$, suppose the preprocessed vector $v^{(0)}\sim\calN(0,I_d)$. By choosing 
	$$
	\mathsf{c}_0'\in\left(0, \sqrt{d} t_{d-1,(1+\delta)/2}/\left(\sqrt{d}+t_{d-1,(1+\delta)/2}\right)\right],
	$$ 
	where $t_{d-1,(1+\delta)/2}$ denotes the $(1+\delta)/2$-quantile of a t-distribution with $d-1$ degrees of freedom, we have $\left|\langle v^*,\overline{v}^{(0)}\rangle\right|\geq \mathsf{c}_0'd^{-\frac{1}{2}}$ with probability at least $1-\delta$.
\end{proposition}
\begin{lemma}\label{lemma-aux-frac}
	Consider vector $v\in\mathbb{R}^{d_{1}}, u\in\mathbb{R}^{d_{2}}$ and matrix $W,Q\in\mathbb{R}^{d_{1}\times d_{2}}$. Define the function $f:\mathbb{R}_{+}\rightarrow\mathbb{R}$ as $f(\eta):=\frac{\left\langle v,Wu\right\rangle+\eta\left\langle v,Qu\right\rangle}{\|W+\eta Q\|_{\mathrm{F}}}$. The following equality holds:
	\begin{align}\label{Taylor-eq1}
		f(\eta)=\frac{\left\langle v,Wu\right\rangle}{\|W\|_{\mathrm{F}}}+\eta\left(\frac{\left\langle v,Qu\right\rangle}{\|W\|_{\mathrm{F}}}-\frac{\left\langle v,Wu\right\rangle\left\langle W,Q\right\rangle}{\|W\|_{\mathrm{F}}^{3}}\right)+\frac{\eta^{2}}{2}\left(-2\mathcal{I}(\tau\eta)-\mathcal{II}(\tau\eta)+3\mathcal{III}(\tau\eta)\right),
	\end{align}
	where $\tau\in[0,1]$ depends on $\eta$, $v$, $u$, $Q$ and $W$. $\mathcal{I}(x)$, $\mathcal{II}(x)$ and $\mathcal{III}(x)$ have the following definitions for any $x\in\mathbb{R}_{+}$:
	\begin{gather*}
		\mathcal{I}(x):=\frac{\left\langle v,Qu\right\rangle\cdot\left(\left\langle W,Q\right\rangle+x\|Q\|_{\mathrm{F}}^{2}\right)}{\|W+x Q\|_{\mathrm{F}}^{3}}, \\
		\mathcal{II}(x):=\frac{\left(\left\langle v,Wu\right\rangle+x\left\langle v,Qu\right\rangle\right)\cdot\|Q\|_{\mathrm{F}}^{2}}{\|W+xQ\|_{\mathrm{F}}^{3}}, \\
		\mathcal{III}(x):=\frac{\left(\left\langle v,Wu\right\rangle+x\left\langle v,Qu\right\rangle\right)\cdot\left(\left\langle W,Q\right\rangle+x\|Q\|_{\mathrm{F}}^{2}\right)^{2}}{\|W+xQ\|_{\mathrm{F}}^{5}}.
	\end{gather*}
\end{lemma}
\begin{proof}
	Observe that the first derivative of $f(\eta)$ is:
	\begin{align}\label{fd}
		f'(\eta)=\frac{\left\langle v,Qu\right\rangle}{\|W+\eta Q\|_{\mathrm{F}}}-\frac{\left(\left\langle v,Wu\right\rangle+\eta\left\langle v,Qu\right\rangle\right)\cdot\left(\left\langle W,Q\right\rangle+\eta\|Q\|_{\mathrm{F}}^{2}\right)}{\|W+\eta Q\|_{\mathrm{F}}^{3}},
	\end{align}
	In addition, one can notice the second derivative of $f(\eta)$ has the following expression:
	\begin{equation}\label{sd}
		\begin{aligned}
			f''(\eta)&=-2\underbrace{\frac{\left\langle v,Qu\right\rangle\cdot\left(\left\langle W,Q\right\rangle+\eta\|Q\|_{\mathrm{F}}^{2}\right)}{\|W+\eta Q\|_{\mathrm{F}}^{3}}}_{\mathcal{I}(\eta)}-\underbrace{\frac{\left(\left\langle v,Wu\right\rangle+\eta\left\langle v,Qu\right\rangle\right)\cdot\|Q\|_{\mathrm{F}}^{2}}{\|W+\eta Q\|_{\mathrm{F}}^{3}}}_{\mathcal{II}(\eta)} \\
			&\mathrel{\phantom{=}}+3\underbrace{\frac{\left(\left\langle v,Wu\right\rangle+\eta\left\langle v,Qu\right\rangle\right)\cdot\left(\left\langle W,Q\right\rangle+\eta\|Q\|_{\mathrm{F}}^{2}\right)^{2}}{\|W+\eta Q\|_{\mathrm{F}}^{5}}}_{\mathcal{III}(\eta)},
		\end{aligned}
	\end{equation}
	Therefore, combining Eq.~\eqref{fd} and Eq.~\eqref{sd} with the Taylor expansion of $f(\eta)$, we complete the proof of Eq.~\eqref{Taylor-eq1}.
	\begin{align}
		f(\eta)=\frac{\left\langle v,Wu\right\rangle}{\|W\|_{\mathrm{F}}}+\eta\left[\left.f'(x)\right|_{x=0}\right]+\frac{\eta^{2}}{2}\left[\left.f''(x)\right|_{x=\tau\eta}\right],\notag
	\end{align}
	where $\tau\in[0,1]$ is a scaling parameter dependent on $\eta$, $v$, $Q$ and $W$. 
\end{proof}

\begin{definition}[Sub-Gaussian Random Variable]\label{sub-Gaussian}
	A random variable $X$ with mean $\mathbb{E} X$ is sub-Gaussian if there is $\sigma\in\mathbb{R}_{+}$ such that
	\begin{align}
		\mathbb{E}\left[e^{\lambda(X-\mathbb{E} X)}\right]\leq e^{\frac{\lambda^{2}\sigma^{2}}{2}},\quad \forall\lambda\in\mathbb{R}.\notag
	\end{align}
\end{definition}

\begin{proposition}\label{prop-A5}[\citep{wainwright2019high}]
	For a random variable $X$ which satisfies the sub-Gaussian condition \ref{sub-Gaussian} with parameter $\sigma$, we have
	\begin{align}\label{eq-subGaussian-prob}
		\mathbb{P}\left(|X-\mathbb{E} X|>c\right)\leq 2e^{-\frac{c^{2}}{2\sigma^{2}}},\quad \forall c>0.
	\end{align}
\end{proposition}
The following two lemmas are from Ding et al. \cite{ding2025near}. We include their complete proofs for the reader’s convenience, allowing easy reference to the precise constants.
\begin{lemma}\label{aux-truncation-subGaussian-second-moment}[Lemma A.11, \cite{ding2025near}]
	Consider a random variable $X$ which is zero-mean and sub-Gaussian with parameter $\sigma$ for some $\sigma>0$. Then, for any $\tau\in(0,1)$, there exists $R>0$ which depends on $\sigma$ and $\tau$ such that
	\begin{align}
		\mathbb{E}\left[X^{2}\mathds{1}_{|X|\leq R}\right]\geq(1-\tau)\mathbb{E}\left[X^{2}\right],\quad \mathbb{E}\left[X^{4}\mathds{1}_{|X|>R}\right]\leq\tau.\notag
	\end{align}
\end{lemma}
\begin{proof}
	According to Eq.~\eqref{eq-subGaussian-prob}, we have $\mathbb{P}(|X|\geq r)\leq 2e^{-\frac{r^{2}}{2\sigma^{2}}}$ for any $r>0$. Therefore, we obtain
	\begin{align}\label{estimate-second-moment}
		\mathbb{E}\left[X^{2}\mathds{1}_{|X|>R}\right]\overset{\text{(a)}}{=}&2\int_{0}^{\infty}r\mathbb{P}(|X|\mathds{1}_{|X|>R}>r)\text{d}r\notag
		\\
		=&2\int_R^{\infty}r\mathbb{P}(|X|>r)\text{d}r+R^{2}\mathbb{P}(|X|>R)\notag
		\\
		\leq&4\int_R^{\infty}re^{-\frac{r^{2}}{2\sigma^{2}}}\text{d}r+2R^{2}e^{-\frac{R^{2}}{2\sigma^{2}}}=(4\sigma^{2}+2R^{2})e^{-\frac{R^{2}}{2\sigma^{2}}},
	\end{align}
	where (a) is derived from [Lemma 2.2.13, \cite{wainwright2019high}]. Moreover, we have
	\begin{align}\label{original-estimate-forth-moment}
		\mathbb{E}\left[X^{4}\mathds{1}_{|X|>R}\right] &= 4\int_{0}^{\infty}r^{3}\mathbb{P}\left(|X|\mathds{1}_{|X|>R}>r\right)\text{d}r\notag \\
        &= 4\int_{R}^{\infty}r^{3}\mathbb{P}\left(|X|>r\right)\text{d}r+R^{4}\mathbb{P}\left(|X|>R\right) \notag \\
		&\leq 8\int_{R}^{\infty}r^{3}e^{-\frac{r^{2}}{2\sigma ^{2}}}\text{d}r+2R^{4}e^{-\frac{R^{2}}{2\sigma ^{2}}}\leq(4\sigma^{2}+2R^{2})^{2}e^{-\frac{R^{2}}{2\sigma ^{2}}},
	\end{align}
	Therefore, we only need to set $R$ as:
	\begin{align}
		R=2\sqrt{2}\sigma\log^{1/2}\left(\frac{4\sigma^{2}+2K}{\tau\min\{\mathbb{E}\left[X^{2}\right],1\}}\right),\notag
	\end{align}
	where $K$ satisfies $K\geq 8\sigma^{2}\log\left((4\sigma^{2}+2K)/(\tau\min\{\mathbb{E}[X^{2}],1\})\right)$.
\end{proof}

\begin{lemma}\label{aux-truncation-subGaussian-covariance}[Lemma A.12, \cite{ding2025near}]
	Suppose zero-mean random variables $X$ and $Y$ are sub-Gaussian with parameters $\sigma_{1}$ and $\sigma_{2}$, respectively. Then, for any $\tau\in(0,1)$, there exists $R_{1},R_{2}\geq0$ such that
		$$
		\left|\mathbb{E}\left[XY\mathds{1}_{\{|X|\leq R_{1}\}\bigcap\{|Y|\leq R_{2}\}}\right]-\mathbb{E}[XY]\right|\leq\tau.
		$$  
\end{lemma}
\begin{proof}
	According to $\mathbb{E}\left[XY\mathds{1}_{\{|X|\leq R_{1}\}\bigcap\{|Y|\leq R_{2}\}}\right]=\mathbb{E}\left[X(1-\mathds{1}_{|X|>R_{1}})Y(1-\mathds{1}_{|Y|>R_{2}})\right]$, we have
	\begin{align}
		\mathbb{E}\left[XY\mathds{1}_{\{|X|\leq R_{1}\}\bigcap\{|Y|\leq R_{2}\}}\right]-\mathbb{E}[XY]=&-\mathbb{E}\left[XY\mathds{1}_{|X|>R_{1}}\right]-\mathbb{E}\left[XY\mathds{1}_{|Y|>R_{2}}\right]\notag
		\\
		&+\mathbb{E}\left[XY\mathds{1}_{|X|>R_{1}}\mathds{1}_{|Y|>R_{2}}\right].\notag
	\end{align}
	For $\mathbb{E}\left[XY\mathds{1}_{|X|>R_{1}}\right]$, we can obtain
	\begin{align}\label{eq:EXY-X}
		\left|\mathbb{E}\left[XY\mathds{1}_{|X|>R_{1}}\right]\right| &\overset{\text{(a)}}{\leq} \left(\mathbb{E}\left[X^{2}\mathds{1}_{|X|>R_{1}}\right]\right)^{1/2}\left(\mathbb{E}\left[Y^{2}\right]\right)^{1/2}\notag \\
		&\overset{\text{(b)}}{\leq} 2(\sigma_{1}^{2}+R_{1}^{2})^{1/2}e^{-\frac{R_{1}^{2}}{4\sigma_{1}^{2}}} \left(\mathbb{E}\left[Y^{2}\right]\right)^{1/2} \\
		&\overset{\text{(c)}}{\leq}4\sigma_{2}(\sigma_{1}^{2}+R_{1}^{2})^{1/2}e^{-\frac{R_{1}^{2}}{4\sigma_{1}^{2}}},
	\end{align}
	where (a) follows from the Cauchy-Schwarz inequality; (b) is derived from the inequality Eq.~\eqref{estimate-second-moment} in the proof of Lemma \ref{aux-truncation-subGaussian-second-moment}; (c) is established through the repeated application of the proof of Lemma \ref{aux-truncation-subGaussian-second-moment}. Similarly, it can be derived that
	\begin{align}
		\left|\mathbb{E}\left[XY\mathds{1}_{|Y|>R_{2}}\right]\right|\leq4\sigma_{1}\left(\sigma_{2}^{2}+R_{2}^{2}\right)^{1/2}e^{-\frac{R_{2}^{2}}{4\sigma_{2}^{2}}}.\notag
	\end{align}
	Finally, for $\mathbb{E}\left[XY\mathds{1}_{|X|>R_{1}}\mathds{1}_{|Y|>R_{2}}\right]$, we have
	\begin{align}
		\left|\mathbb{E}\left[XY\mathds{1}_{|X|>R_{1}}\mathds{1}_{|Y|>R_{2}}\right]\right| &\leq \left(\mathbb{E}\left[X^{2}\mathds{1}_{|X|>R_{1}}Y\right]\right)^{1/2}\left(\mathbb{E}\left[Y^{2}\mathds{1}_{|Y|>R_{2}}Y\right]\right)^{1/2} \notag \\
        &\leq 4(\sigma_{1}^{2}+R_{1}^{2})^{1/2}\left(\sigma_{2}^{2}+R_{2}^{2}\right)^{1/2}e^{-\left(\frac{R_{1}^{2}}{4\sigma_{1}^{2}}+\frac{R_{2}^{2}}{4\sigma_{2}^{2}}\right)}.\notag
	\end{align}
	Therefore, we only need to set $R_{1}$ and $R_{2}$ as:
	\begin{align}\label{frac-estimator}
		R_{1}=\sqrt{2}\sigma_{1}\log^{1/2}\left(\frac{256\max\{\sigma_{2}^{2},1\}(\sigma_{1}^{2}+K)}{\tau^{2}}\right),\quad R_{2}=\sqrt{2}\sigma_{2}\log^{1/2}\left(\frac{256\max\{\sigma_{1}^{2},1\}(\sigma_{2}^{2}+K)}{\tau^{2}}\right),
	\end{align}
	where $K$ satisfies $K\geq2(\sigma_{1}^{2}+\sigma_{2}^{2})\log\left(256\max\{\sigma_{1}^{2},\sigma_{2}^{2},1\}(\sigma_{1}^{2}+\sigma_{2}^{2}+K)\tau^{-2}\right)$.
\end{proof}
\begin{lemma}\label{aux-martingale-concentration}[Lemma A.14, \cite{ding2025near}]
	Let $c>0$, $\gamma<1$ and $a_t>0$ for any $t\in[0:T-1]$. Consider a sequence of random variables $\{v^t\}_{t=0}^{T-1}\subset[0,2c]$, which satisfies $\bbE[e^{\lambda(v^{t+1}-(1-\eta_t)v^t)}\mid\calF_t]\leq e^{\frac{\lambda^2a_t^2}{2}}$ almost surely for any $\lambda\in\bbR_+$ with stepsize $\eta_t\geq0$. Then, there is 
    \begin{align}
        \bbP\left(v^T>c\bigwedge v^0\leq\gamma c\right)\leq\exp\left\{-\frac{(1-\gamma)^2c^2}{2\sum_{j=0}^{T-1}a_j^2\prod_{i=j+1}^{t-1}(1-\eta_i)^{2}}\right\}.\notag
    \end{align}
\end{lemma}
\begin{lemma}\label{aux-martingale-concentration-subtraction}[Lemma A.15, \cite{ding2025near}]
	Let $c>\gamma>0$ and $a_{t}>0$ for any $t\in[0:T-1]$. Consider a sequence of random variables $\{v^{t}\}_{t=0}^{T-1}\subset[0,c]$, which satisfies 
	$\mathbb{E}\left[e^{\lambda(v^{t+1}+\eta_{t}-v^{t})}\middle|\mathcal{F}_{t}\right]\leq e^{\frac{\lambda^{2}a_{t}^{2}}{2}}$ almost surely for any $\lambda\in\mathbb{R}_{+}$ with stepsize $\eta_{t}\geq 0$. Suppose $\sum_{t=0}^{T-1}\eta_{t}>v^{0}$, there is
	\begin{align}
		\mathbb{P}\left(v^T>\gamma\right)\leq\exp\left\{-\frac{\left(\gamma+\sum_{t=0}^{T-1}\eta_{t}-v^{0}\right)^{2}}{2\sum_{j=0}^{T-1}a_{j}^{2}}\right\}.\notag
	\end{align}
\end{lemma}
\begin{lemma}\label{aux-martingale-concentration-subtraction-coro}[Extension of Lemma \ref{aux-martingale-concentration-subtraction}]
	Let $c>0>\gamma$ and $a_{t}>0$ for any $t\in[0:T-1]$. Consider a sequence of random variables $\{v^{t}\}_{t=0}^{T-1}\subset[-c,0]$, which satisfies 
	$\mathbb{E}\left[e^{\lambda(v^{t+1}+\eta_{t}-v^{t})}\middle|\mathcal{F}_{t}\right]\leq e^{\frac{\lambda^{2}a_{t}^{2}}{2}}$ almost surely for any $\lambda\in\mathbb{R}_{+}$ with stepsize $\eta_{t}\geq 0$. There is
	\begin{align}
		\mathbb{P}\left(v^T>\gamma\right)\leq\exp\left\{-\frac{\left(\gamma+\sum_{t=0}^{T-1}\eta_{t}-v^{0}\right)^{2}}{2\sum_{j=0}^{T-1}a_{j}^{2}}\right\}.\notag
	\end{align}
\end{lemma}
\begin{lemma}\label{aux-3}[Lemma A.16, \cite{ding2025near}]
    For $L,K\in\bbN_+$, consider $T\in\bbN^+$ such that $LK\leq T<(L+1)K$. Then we have 
    \begin{align}
        \sum_{t=0}^{T}\left(\prod_{i=t}^{T}(1-c\eta_t)\right)\eta_t^2\leq\frac{2\eta_0}{c},
    \end{align}
    where $\eta_t=\frac{\eta_0}{2^l}$ if $lK\leq t\leq \min\{(l+1)K-1,T\}$ for any $l\in[0:L]$ and $c>0$ is a constant.
\end{lemma}
%\begin{proof}
%	We define $D_{t+1}:=v^{t+1}+\sum_{i=0}^{t}\eta_{i}-\left(v^{t}+\sum_{i=0}^{t-1}\eta_{i}\right)$ for any $t\in[0:T-1]$. Therefore, applying iterated expectation yields
%	\begin{align}\label{eq-sub-gaussian-substraction}
%		\mathbb{E}\left[e^{\lambda\left(\sum_{i=1}^{t}D_{i}\right)}\right] &= \mathbb{E}\left[e^{\lambda\left(\sum_{i=1}^{t-1}D_{i}\right)}\mathbb{E}\left[\left.e^{\lambda D_{t}}\right|\mathcal{F}_{t-1}\right]\right]\notag \\
%		&= \mathbb{E}\left[e^{\lambda\left(\sum_{i=1}^{t-1}D_{i}\right)}\mathbb{E}\left[\left.e^{\lambda(v^{t}+\eta_{t-1}-v^{t-1})}\right|\mathcal{F}_{t-1}\right]\right] \notag \\
%		&\overset{\text{(a)}}{\leq} e^{\frac{\lambda^{2}a_{t-1}^{2}}{2}}\mathbb{E}\left[e^{\lambda\left(\sum_{i=1}^{t-1}D_{i}\right)}\right] \notag \\
%		&\leq e^{\frac{\lambda^{2}\sum_{j=0}^{t-1}a_{j}^{2}}{2}},
%	\end{align}
%	for any $\lambda\in\mathbb{R}^+$ and $t\in[1:T]$, where 
%	(a) follows from the condition that $\mathbb{E}[e^{\lambda(v^{t}+\eta_{t-1}-v^{t-1})}\mid\mathcal{F}_{t-1}]\leq e^{\frac{\lambda^{2}a_{t-1}^{2}}{2}}$ almost surely for any $\lambda\in\mathbb{R}_{+}$. Then we obtain 
%	\begin{align}
%		\mathbb{P}\left(v^T>\gamma\right) &\leq \min_{\lambda>0}\frac{\mathbb{E}\left[e^{\lambda(\sum_{i=1}^TD_{i})}\right]}{e^{\lambda\left(\gamma+\sum_{t=0}^{T-1}\eta_{t}-v^{0}\right)}} \notag \\
%		&\overset{\text{(b)}}{\leq} \exp\left\{-\frac{\left(\gamma+\sum_{t=0}^{T-1}\eta_{t}-v^{0}\right)^{2}}{2\sum_{j=0}^{T-1}a_{j}^{2}}\right\},
%	\end{align}
%	where (b) is derived from Eq.~\eqref{eq-sub-gaussian-substraction}.
%\end{proof}

\section{Supplementary Algorithm}\label{sup-alg}
In this section, we present some algorithm variants not shown in Section~\ref{alg}.

\begin{algorithm}[H] \caption{Sequential Normalized SGA with Decaying Step-Size (SNSGA2)} \label{SGA-decay-ss}
    \footnotesize
    \textbf{Input:} Initial weights {\scriptsize $\left\{W_{m}^{(0)}\right\}$}, initial step sizes $\{\eta_{m}\}$, iteration budgets $\{T_{m}\}$, step size decay lengths {\scriptsize $\left\{\widehat T_{m}\right\}$}, reward oracle $\widehat{\mathcal{R}}$ with query access $\widehat{\mathcal{R}}_{m}^{(t)}(\cdot)$.
    
    \textbf{Output:} Normalized final weights {\scriptsize $\left\{\overline W_{m}^{(T_{m})}\right\}$}. 
    
    \textbf{Variables:} Memory states $M = \left\{M_{m}: m \in [k]\right\}$, where $M_{m}$ stores the iterates {\scriptsize $\left\{W_{m}^{(t)} : t \in \{0\} \cup [T_{m}]\right\}$} for block $m$.
    
    \begin{algorithmic}[1]
        \FOR{$m = k, k-1, \dots, 1$}
            \FOR{$t = 0, \dots, T_{m}-1$}
                \IF{$t>0$ and $t \bmod \widehat T_{m} = 0$}
                    \STATE Decay step size: $\eta_{m} \gets \eta_{m} / 2$.
                \ENDIF
                \STATE Sample data $\mathbf T^{(t+1)}$ and update iterates $W_{m}^{(t+1)}$ via Eq.~\eqref{update-SGA} in memory $M_{m}$.
            \ENDFOR
            \STATE Set final weight: $\overline W_{m}^{(T_{m})} \gets W_{m}^{(T_{m})}/\left\|W_{m}^{(T_{m})}\right\|_{\mathrm{F}}$.
        \ENDFOR
        \RETURN {\scriptsize $\left\{\overline{W}_{m}^{(T_{m})}\right\}$}.
    \end{algorithmic}
\end{algorithm}

\begin{algorithm}[H] \caption{Sequential Normalized SGA with Decaying Step-Size (Odd Case)} \label{SGA-decay-ss-odd}
    \footnotesize
    \textbf{Input:} Initial weights {\scriptsize $w_1^{(0)}\bigcup\left\{W_{m}^{(0)}\right\}$}, initial step sizes $\{\eta_{m}\}$, iteration budgets $\{T_{m}\}$, step size decay lengths {\scriptsize $\left\{\widehat T_{m}\right\}$}, reward oracle $\widehat{\mathcal{R}}$ with query access $\widehat{\mathcal{R}}_{m}^{(t)}(\cdot)$.
    
    \textbf{Output:} Normalized final weights {\scriptsize $\overline{w}_1^{(T_1)}\bigcup\left\{\overline W_{m}^{(T_{m})}\right\}$}. 
    
    \textbf{Variables:} Memory states $M = \left\{M_{m}: m \in [k+1]\right\}$, where $M_{1}$ stores the iterates {\scriptsize $\left\{w_{1}^{(t)}:t\in \{0\}\cup \left[T_{1}\right]\right\}$} and $M_{m}$ stores the iterates {\scriptsize $\left\{W_{m}^{(t)} : t \in \{0\} \cup [T_{m}]\right\}$} for block $m\in[2:k+1]$.
    
    \begin{algorithmic}[1]
        \FOR{$m = k+1, k, \dots, 2$}
            \FOR{$t = 0, \dots, T_{m}-1$}
                \IF{$t>0$ and $t \bmod \widehat T_{m} = 0$}
                    \STATE Decay step size: $\eta_{m} \gets \eta_{m} / 2$.
                \ENDIF
                \STATE Sample data $\mathbf T^{(t+1)}$ and update iterates $W_{m}^{(t+1)}$ via Eq.~\eqref{update-SGA-odd} in memory $M_{m}$.
            \ENDFOR
            \STATE Set final weight: $\overline W_{m}^{(T_{m})} \gets W_{m}^{(T_{m})}/\left\|W_{m}^{(T_{m})}\right\|_{\mathrm{F}}$.
        \ENDFOR
        \FOR{$t = 0, \dots, T_{1}-1$}
            \IF{$t>0$ and $t \bmod \widehat T_{1} = 0$}
            \STATE Decay step size: $\eta_{1} \gets \eta_{1} / 2$.
            \ENDIF
            \STATE Sample data $\mathbf T^{(t+1)}$ and update iterates $w_{1}^{(t+1)}$ via Eq.~\eqref{update-SGA-odd-vec} in memory $M_{1}$.
        \ENDFOR
        \STATE Set final weight: $\overline w_{1}^{(T_{1})} \gets w_{1}^{(T_{1})}/\left\|w_{1}^{(T_{1})}\right\|_{2}$.
        \RETURN {\scriptsize $\overline{w}_1^{(T_1)}\bigcup\left\{\overline{W}_{m}^{(T_{m})}\right\}$}.
    \end{algorithmic}
\end{algorithm}

\addcontentsline{toc}{section}{Acknowledgments}
\section*{Acknowledgments}
We thank a bunch of people and funding agency.
\addcontentsline{toc}{section}{References}

\bibliographystyle{alpha}
\bibliography{ref}

\end{document}